\documentclass{article}
\PassOptionsToPackage{numbers,compress}{natbib}
\usepackage[main,final]{neurips_2026}

\usepackage[utf8]{inputenc}
\usepackage[T1]{fontenc}
\usepackage{hyperref}
\usepackage{url}
\usepackage{booktabs}
\usepackage{amsfonts}
\usepackage{nicefrac}
\usepackage{microtype}
\usepackage{xcolor}
\usepackage{amsmath}
\usepackage{amssymb}
\usepackage{mathtools}
\usepackage{amsthm}
\usepackage{makecell}
\usepackage{multirow}
\usepackage{graphicx}
\usepackage{seqsplit}

\makeatletter
\newcommand\midfootnotesize{%
  \@setfontsize\midfootnotesize{7.7}{9}%
}
\DeclareMathSizes{7.7}{7.7}{6}{5}
\makeatother

\usepackage[capitalize,noabbrev]{cleveref}

\theoremstyle{plain}
\newtheorem{theorem}{Theorem}[section]
\newtheorem{proposition}[theorem]{Proposition}

\theoremstyle{definition}

\theoremstyle{remark}

\title{pCoMole: Pareto-Constrained\\Molecule Editing with Discrete Flows}

\author{%
Tong Chen,\textsuperscript{1,*}%
\quad Maximilian Holsman,\textsuperscript{2,*}%
\quad Lin Zhao,\textsuperscript{3,*}%
\quad Pranam Chatterjee\textsuperscript{1,3,\ensuremath{\dagger}}\\[0.8em]
\normalfont
\textsuperscript{1}Department of Computer and Information Science, University of Pennsylvania\\
\textsuperscript{2}Department of Computer Science, Duke University\\
\textsuperscript{3}Department of Bioengineering, University of Pennsylvania\\[0.4em]
\textsuperscript{*}Equal contribution \\ 
\textsuperscript{\ensuremath{\dagger}}\textbf{Corresponding author:} \href{mailto:pranam@upenn.edu}{\texttt{pranam@upenn.edu}}\\
\textbf{Code:} \href{https://huggingface.co/ChatterjeeLab/pCoMole}{\texttt{https://huggingface.co/ChatterjeeLab/pCoMole}}
}

\begin{document}
\maketitle

\begin{abstract}
Biomolecular therapeutics often start from known sequences and require targeted editing to improve multiple properties while satisfying hard biochemical and manufacturability constraints. However, existing generative methods do not jointly support multi-objective optimization, hard feasibility, and sequence editing in discrete, variable-length biological spaces. In this work, we introduce \textbf{P}areto-\textbf{Co}nstrained \textbf{Mol}ecule \textbf{e}diting (\textbf{pCoMole}), a framework built on discrete flow matching that steers a pre-trained Edit Flow toward user-specified preferences while enforcing terminal feasibility. pCoMole defines a feasibility-gated terminal distribution using an augmented Tchebycheff utility and realizes the resulting preference tilt through a Doob-$h$ transform of the underlying edit process. To make this construction practical, we approximate the required harmonic function using short Monte Carlo rollouts over candidate edits, yielding an efficient guided editor with provable preference consistency. We validate pCoMole by shrinking GFP while retaining fluorescence-related properties, shortening diverse Cas9 orthologs while preserving PAM specificity, and compressing peptide binders into short peptidomimetics that optimize seven drug-related properties under hard constraints. In wet lab testing, two 229-residue pCoMole-designed eGFP variants retained clear green fluorescence in BL21 cells after 10 deletions, with either one or two substitutions. Together, pCoMole enables constraint-aware, Pareto-aligned editing of biomolecular sequences in discrete, variable-length spaces.
\end{abstract}

\section{Introduction}
\label{sec:introduction}

Biomolecular therapeutics are rarely designed from scratch \citep{packer2015methods, mclure2022high}. Engineering typically proceeds by editing existing sequences to improve efficacy and developability while obeying hard biochemical and manufacturability constraints. In many settings, \emph{sequence shrinkage} is a primary objective: shorter molecules are cheaper to synthesize, easier to manufacture, and more amenable to delivery. This pressure is especially acute for genome editors such as CRISPR-Cas9, whose large size poses a fundamental barrier to viral delivery and \textit{in vivo} translation \citep{wang2020crispr, kabadi2024aav}, motivating aggressive shrinkage while preserving PAM specificity and catalytic function. Similar edit-based optimization arises across modalities. For insulin analogs, targeted edits tune absorption kinetics and duration while preserving receptor engagement and formulation stability \citep{howey1994lys, kramer2021insulin}. For monoclonal antibodies such as adalimumab, sequence modifications reduce immunogenicity and aggregation risk without sacrificing affinity or specificity \citep{weinblatt2003adalimumab}. For peptide therapeutics such as the revolutionary GLP-1R agonists, optimization explicitly includes shortening or removing dispensable segments to reduce synthesis cost while retaining biological activity \citep{swedberg2016truncated, han2024transformation, yoshida2022peptide, cunningham2017peptides}. These examples illustrate that biomolecular design is fundamentally an \emph{editing problem}, often driven by shrinkage and requiring multi-objective optimization under hard terminal constraints.

Several recent methods address protein shrinkage directly. RayGun \citep{devkota2024miniaturizing} performs template-guided miniaturization by encoding variable-length protein sequences into a fixed-dimensional embedding space and decoding them at shorter lengths, enabling large indels while broadly preserving structure. SCISOR \citep{baron2025diffusion} instead formulates shrinkage as a discrete diffusion process trained to plan deletions by reversing an insertion-only noising process, yielding state-of-the-art predictions of deletion effects and improved motif preservation. While effective for shrinkage, these methods are not designed for general multi-objective editing: they do not jointly optimize multiple developability or functional objectives under explicit feasibility constraints, nor do they provide a mechanism for preference-driven trade-off control during editing.

More broadly, constrained multi-objective molecular optimization has been studied using evolutionary search and reinforcement learning \citep{xia2025cmomo, lin2026multi, sun2022molsearch, yang2024enabling, olivecrona2017molecular, you2018graph, zhou2019optimization, loeffler2024reinvent}. These approaches can approximate Pareto fronts but typically rely on repeated online oracle evaluations and are not naturally suited to offline editing of incumbent molecules in large, variable-length discrete spaces. Discrete diffusion and discrete flow matching models have recently emerged as powerful generators for sequence spaces \citep{austin2021structured, campbell2022continuous, gat2024discrete}, with guidance mechanisms enabling control over objectives and constraints \citep{li2022diffusion, schiff2024simple, tang2025gumbelfm, zhang2025metalorian}. However, most controllable discrete generators focus on \emph{de novo} generation and address either constraint satisfaction or preference control in isolation, without supporting iterative sequence editing with insertions and deletions \citep{chen2025mogdfm, chen2025areuredi, tang2025peptune, tang2025tr2d2, vincoff2025soapia}. Edit Flows extend discrete flow matching to stochastic sequence editing via local insertions, deletions, and substitutions, but have not been developed as a constraint-aware, multi-objective editing framework for biomolecular design \citep{havasi2025edit}.

To address this gap, we introduce \textbf{P}areto-\textbf{Co}nstrained \textbf{Mol}ecule \textbf{e}diting (\textbf{pCoMole}), a framework for offline, constraint-aware multi-objective sequence editing. pCoMole steers a pre-trained Edit Flow by defining a feasibility-gated terminal distribution ranked by an augmented Tchebycheff utility, and realizes this preference tilt through a Doob-$h$ transform. Practical guidance is achieved via short Monte Carlo rollouts that approximate the required harmonic function, enabling efficient editing in large, variable-length discrete biomolecular spaces.

\noindent\textbf{Our key contributions are as follows:}
\begin{enumerate}
    \item We provide an offline method for Pareto-aligned biomolecular sequence editing under hard terminal constraints.
    \item We derive a Doob-$h$ guided editing process with an augmented Tchebycheff utility and make it practical using short rollout-based guidance with theoretical guarantees.
    \item We demonstrate pCoMole across multiple biomolecular editing tasks, including GFP and Cas9 shrinkage and \textit{in silico} peptidomimetic binder compression, and experimentally confirm retained green fluorescence in two shortened eGFP designs.
    \item We show improved feasibility and trade-off control relative to state-of-the-art shrinkage and editing methods, including RayGun and SCISOR.
\end{enumerate}

A detailed discussion of related work is provided in Appendix~\ref{app:related_work}.

\section{Background}
\label{sec:editflows}

\paragraph{Continuous-time Markov chains}
Let $\mathcal{X}$ be a discrete state space and $(X_t)_{t\in[0,1]}$ a continuous-time Markov chain (CTMC) on $\mathcal{X}$ with time-inhomogeneous rate $u_t(\cdot\mid x_t)$. The CTMC is characterized by the first-order expansion
\begin{equation}
% \small
\Pr(X_{t+h}=x \mid X_t=x_t) \;=\; \delta_{x_t}(x) + h\,u_t(x\mid x_t) + o(h),
\label{eq:ctmc_update}
\end{equation}
where $\delta_{x_t}(x)$ is the Kronecker delta and $\lim_{h\to 0} o(h)/h = 0$. Valid rates satisfy the rate conditions:
\begin{equation}
u_t(x\mid x_t)\ge 0 \;\; \forall x\neq x_t,
\qquad
\sum_{x\in\mathcal{X}} u_t(x\mid x_t)=0,
\label{eq:rate_conditions}
\end{equation}
which implies $u_t(x_t\mid x_t)=-\sum_{x\neq x_t} u_t(x\mid x_t)$ so that the time marginals $p_t(x)=\Pr(X_t=x)$ obey the Kolmogorov forward equation.

\paragraph{Edit Flows.}
Edit Flows define a CTMC directly on the space of variable-length sequences \citep{havasi2025edit}. Let $\mathcal{T}$ denote a vocabulary of size $M$ and let $\mathcal{X}=\bigcup_{n=0}^{N}\mathcal{T}^n$ be the set of all sequences up to length $N$. The model parameterizes a rate field $u_t^\theta(\cdot\mid x)$ whose support is restricted to sequences that differ from $x$ by a single edit. For a sequence $x=(x_1,\ldots,x_{n(x)})$ and position $i\in\{1,\ldots,n(x)\}$, the edit operations are insertion $\mathrm{ins}(x,i,a)$, deletion $\mathrm{del}(x,i)$, and substitution $\mathrm{sub}(x,i,a)$, where $a\in\mathcal{T}$. Since these operations yield mutually exclusive outcomes, the rate for each possible edit is parameterized by a type-specific total rate and, when applicable, a token distribution:
\begin{align}
u_t^\theta(\mathrm{ins}(x,i,a)\mid x) &= \lambda^{\mathrm{ins}}_{t,i}(x)\,Q^{\mathrm{ins}}_{t,i}(a\mid x), \\
u_t^\theta(\mathrm{del}(x,i)\mid x) &= \lambda^{\mathrm{del}}_{t,i}(x), \\
u_t^\theta(\mathrm{sub}(x,i,a)\mid x) &= \lambda^{\mathrm{sub}}_{t,i}(x)\,Q^{\mathrm{sub}}_{t,i}(a\mid x).
\end{align}
Here $\lambda^{\mathrm{ins}}_{t,i}(x)$, $\lambda^{\mathrm{del}}_{t,i}(x)$, and $\lambda^{\mathrm{sub}}_{t,i}(x)$ are nonnegative total rates that control how frequently each edit type occurs at position $i$, and $Q^{\mathrm{ins}}_{t,i}(\cdot\mid x)$ and $Q^{\mathrm{sub}}_{t,i}(\cdot\mid x)$ are normalized distributions over token values. The diagonal term $u_t^\theta(x\mid x)$ is determined by the rate conditions and equals the negative sum of all insertion, deletion, and substitution total rates across positions, ensuring that the CTMC is well-defined.

Edit Flows are trained using an auxiliary alignment variable to resolve the ambiguity that multiple edit paths can map $x_0$ to $x_1$. The method augments the token space with a blank symbol $\varepsilon$ and considers aligned sequences $z\in\mathcal{Z}=(\mathcal{T}\cup\{\varepsilon\})^N$, together with a projection $f_{\mathrm{frm\text{-}blank}}:\mathcal{Z}\to\mathcal{X}$ that removes blanks. Training constructs pairs $(z_0,z_1)$ that project to $(x_0,x_1)$ and samples intermediate alignments $z_t$ along a token-wise mixture path with schedule $\kappa_t$. Each mismatch coordinate between $z_t$ and $z_1$ induces a unique single-edit update of the projected sequence $x_t=f_{\mathrm{frm\text{-}blank}}(z_t)$, which provides a tractable supervision signal for the Edit Flow rates. The resulting objective penalizes large total outgoing rate while increasing the log-rate assigned to the alignment-implied edit:
\begin{equation}
% \small
\mathcal{L}(\theta)
=
\mathbb{E}\Bigg[
\sum_{x\neq x_t} u_t^\theta(x\mid x_t)
-
w_t \sum_{i\in\Delta_t}\log u_t^\theta\!\big(x_t^{(i)}\mid x_t\big)
\Bigg],
\label{eq:editflows_loss}
\end{equation}
where $w_t=\dot{\kappa}_t/(1-\kappa_t)$, $\Delta_t=\{\,i: z_t^{\,i}\neq z_1^{\,i}\,\}$, and $x_t^{(i)}$ denotes the sequence obtained by replacing the $i$-th coordinate of $z_t$ by $z_1^{\,i}$ and projecting back via $f_{\mathrm{frm\text{-}blank}}$.

In the rest of the paper, we use $x$ to denote a generic sequence in $\mathcal{X}$ and write $x_t$ for the state of the CTMC at time $t$; all rate functions such as $u_t(\cdot\mid x)$, $\lambda_{t,p}(x)$, and $\pi^{\mathrm{type}}_{t,p}(\cdot\mid x)$ are defined for generic $x$ and are evaluated at $x=x_t$ inside the training objective.

\begin{figure}[t]
  \vskip 0.2in
  \begin{center}
    \centerline{\includegraphics[width=\textwidth]{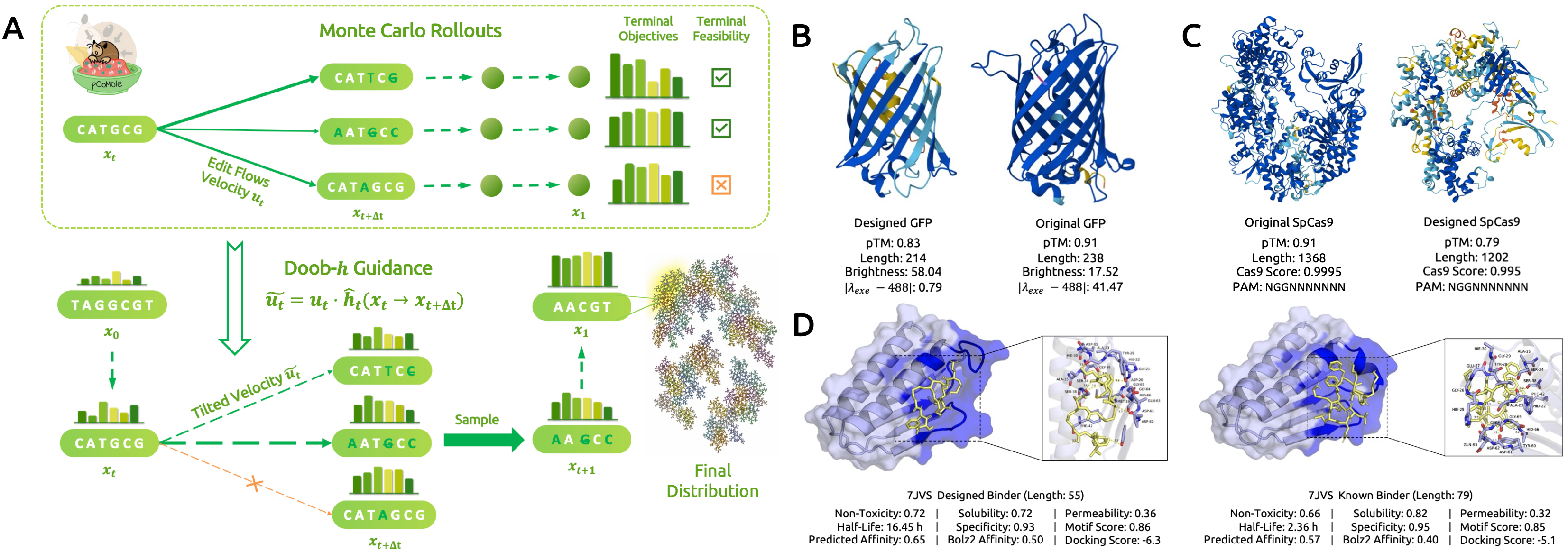}}
    \caption{
\textbf{Overview of pCoMole and representative designs.}
\textbf{(A)} Starting from an intermediate sequence $x_t$, pCoMole samples candidate edits from a pretrained Edit Flows velocity $u_t$ and uses short Monte Carlo rollouts to estimate terminal utility and feasibility. The resulting approximate Doob-$h$ guidance reweights transitions into a tilted velocity $\tilde{u}_t$, steering generation toward feasible, high-utility terminal sequences.
\textbf{(B)} In GFP miniaturization, pCoMole reduces sequence length while preserving the canonical $\beta$-barrel and improving predicted optical objectives.
\textbf{(C)} In Cas9 shrinkage, pCoMole produces a shortened SpCas9 variant that satisfies the PAM constraint while maintaining high predicted Cas9-likeness and plausible structure.
\textbf{(D)} In short peptidomimetic binder design, pCoMole shortens a known 7JVS binder while improving predicted properties. The binder is shown in yellow, the target in light blue, and dark blue highlights target motif regions. Insets zoom into the binding interface (see Figure \ref{fig:dock4} for enlarged views).
}
    \label{fig:pcomole}
  \end{center}
\end{figure}

\section{Methods}
\label{sec:method}
\subsection{Problem Setup}
\label{sec:problem_setup}

Let $\mathcal{T}$ be a finite token vocabulary and let $\mathcal{X}=\bigcup_{n=0}^{N}\mathcal{T}^n$ denote the set of variable-length token sequences up to the model-specific maximum length $N$. We consider $K$ objective functions $f_k:\mathcal{X}\to\mathbb{R}$ for $k\in\{1,\dots,K\}$ to be maximized, together with $M$ inequality constraints $g_m:\mathcal{X}\to\mathbb{R}$ for $m\in\{1,\dots,M\}$ and $L$ equality constraints $h_\ell:\mathcal{X}\to\mathbb{R}$ for $\ell\in\{1,\dots,L\}$. The feasible set is
\begin{equation}
\mathcal{F}\coloneqq\{x\in\mathcal{X}: g_m(x)\le 0\ \forall m,\ \ h_\ell(x)=0\ \forall \ell\}.
\label{eq:feasible_set}
\end{equation}
Given an input sequence $x_0\in\mathcal{X}$ and a pre-trained Edit Flow that induces a conditional terminal distribution $q(\cdot\mid x_0)$ over edited sequences, our goal is to sample terminal sequences $x_1\sim q(\cdot\mid x_0)$ that solve a multi-objective optimization problem by improving the objective vector $(f_1(x_1),\dots,f_K(x_1))$ while satisfying $x_1\in\mathcal{F}$. Intermediate states visited by the editing process are not required to satisfy the constraints because they are transient and are not returned as designs. Enforcing feasibility after every edit can exclude coordinated trajectories that temporarily leave $\mathcal{F}$ before reaching a feasible endpoint. Terminal rollouts retain feasibility information by assigning zero preference to infeasible terminal states.

\subsection{pCoMole with \texorpdfstring{Doob-$h$}{Doob-h} Guided Edit Flows}
\label{sec:doob_h_guidance}

With a pre-trained Edit Flow in hand, we next introduce \textbf{pCoMole}, a \textbf{P}areto-\textbf{Co}nstrained \textbf{Mol}ecule \textbf{e}diting algorithm, to edit molecule sequences for optimizing multiple objectives while maintaining feasibility (Figure \ref{fig:pcomole}). Let $u_t(\cdot\mid x)$ denote the fixed rate field of the trained Edit Flow and let $(X_t)_{t\in[0,1]}$ be the induced time-inhomogeneous CTMC with $X_0=x_0$, which defines a terminal distribution $q(\cdot\mid x_0)$ at $t=1$. pCoMole biases this terminal distribution toward feasible molecules with improved multi-objective performance by specifying a terminal preference function $G:\mathcal{X}\to\mathbb{R}_{\ge 0}$.

We construct $G$ from an augmented Tchebycheff utility. Let $f(x)=(f_1(x),\dots,f_K(x))$ be objective scores to be maximized, let $\omega\in\mathbb{R}_{>0}^K$ be a weight vector, and let $r\in\mathbb{R}^K$ be a reference point. We define
\begin{equation}
U(x)
\coloneqq
\min_{k\in\{1,\dots,K\}} \omega_k\,(f_k(x)-r_k)
+
\rho\sum_{k=1}^K \omega_k\,(f_k(x)-r_k),
\label{eq:atc_utility}
\end{equation}
where $\rho>0$ is the augmentation coefficient. This utility is strictly monotone with respect to Pareto dominance for $\omega>0$, and therefore encourages improvements across objectives rather than collapsing to a single criterion. We incorporate feasibility by defining
\begin{equation}
G(x)
\coloneqq
\exp\!\big(\beta\,U(x)\big)\,\mathbb{I}[x\in\mathcal{F}],
\label{eq:terminal_preference}
\end{equation}
with inverse-temperature $\beta>0$. The corresponding target terminal distribution is the tilted law
\begin{equation}
q_G(x_1\mid x_0) = \frac{q(x_1\mid x_0)\,G(x_1)}{\sum_{x'} q(x'\mid x_0)\,G(x')}\propto q(x_1\mid x_0)\,G(x_1),
\label{eq:tilted_terminal}
\end{equation}
which reweights base Edit Flow samples by preference while assigning zero mass to infeasible terminals.

The Doob-$h$ transform provides a principled way to realize \eqref{eq:tilted_terminal} through a modified CTMC that preserves the local edit structure of the base process. Define the harmonic function
\begin{equation}
h_t(x)\ \coloneqq\ \mathbb{E}\!\left[G(X_1)\,\middle|\,X_t=x\right],
\label{eq:h_function}
\end{equation}
where the conditional expectation is taken under the base rates $u_t$. The Doob-$h$ guided process is the CTMC with off-diagonal rates
\begin{equation}
u_t^{G}(y\mid x)\ \coloneqq\ u_t(y\mid x)\,\frac{h_t(y)}{h_t(x)},\qquad y\neq x,
\label{eq:doob_rates}
\end{equation}
and diagonal term determined by the rate conditions.

\begin{proposition}
\label{prop:doob_transform}
Assume $G(X_1)$ is integrable under the base Edit Flow and $h_t(x)>0$ for all $(t,x)$ reachable from $x_0$. The CTMC with rates \eqref{eq:doob_rates} has terminal distribution $q_G(\cdot\mid x_0)$ defined in \eqref{eq:tilted_terminal}.
\end{proposition}

The proof is provided in Appendix \ref{prop:doob_transform_proof}. Proposition~\ref{prop:doob_transform} justifies our design by showing that, in the ideal setting where $h_t$ is available, the guided edit process samples exactly from the preference-tilted terminal distribution that encodes both feasibility and multi-objective optimization through $G$. Varying the weight vector $\omega$ yields different tilts of $q(\cdot\mid x_0)$, which allows pCoMole to explore distinct Pareto trade-offs while remaining anchored to the pre-trained Edit Flow dynamics. The proofs for Pareto optimality and coverage (Proposition \ref{prop:atc_pareto}) are provided in Appendix \ref{sec:proofs}.

\subsection{Practical Estimation with Short Rollouts}
\label{sec:short_rollouts}

The Doob-$h$ guided rates in \eqref{eq:doob_rates} depend on the harmonic function
$h_t(x)=\mathbb{E}[G(X_1)\mid X_t=x]$, which is intractable to compute exactly in large edit spaces. pCoMole therefore approximates $h_t$ using short Monte Carlo rollouts under the base Edit Flow. Given $(t,x)$, we simulate $R$ independent continuations from time $t$ to $1$ using the base rate field $\{u_s\}_{s\in[t,1]}$, producing terminal states $\{X_1^{(r)}(t,x)\}_{r=1}^R$, and define
\begin{equation}
\widehat{h}_t(x)\ \coloneqq\ \frac{1}{R}\sum_{r=1}^{R} G\!\left(X_1^{(r)}(t,x)\right).
\label{eq:rollout_estimator}
\end{equation}
When the continuations are exact simulations of the base CTMC, $\widehat{h}_t(x)$ is an unbiased estimator of $h_t(x)$, with feasibility incorporated through $G$.

Evaluating \eqref{eq:rollout_estimator} for all admissible successors is prohibitive, so we adopt a two-stage approximation. At a decision point $(t,x_t)$, we first sample $C$ candidate next states $\{y_c\}_{c=1}^{C}$ by drawing $C$ edit events from the base rates $u_t(\cdot\mid x_t)$ and applying the corresponding edits. For each candidate $y_c$, we run $R$ rollouts from $(t,y_c)$ to time $1$ and compute
\begin{equation}
\widehat{h}_t(y_c)\ \coloneqq\ \frac{1}{R}\sum_{r=1}^{R} G\!\left(X_1^{(c,r)}\right),\qquad c\in\{1,\dots,C\}.
\label{eq:rollout_estimator_candidates}
\end{equation}
Substituting $\widehat{h}_t$ into the Doob-$h$ reweighting yields an approximate guided selection rule over the candidate set,
\begin{equation}
\widetilde{u}_t(y_c\mid x_t)\ \propto\ u_t(y_c\mid x_t)\,\widehat{h}_t(y_c),
\label{eq:approx_guided_rates}
\end{equation}
with normalization over $c\in\{1,\dots,C\}$. Here, $u_t(y_c\mid x_t)$ proposes locally plausible edits, while $\widehat{h}_t(y_c)$ favors candidates whose downstream continuations achieve higher terminal preference under $G$, naturally downweighting candidates unlikely to admit feasible completions.

Because short rollouts may miss feasible terminals under limited budgets, pCoMole maintains an incumbent feasible terminal sequence with the highest observed preference value across all rollouts. For efficiency, we first prune candidates using a cheap screening criterion and allocate rollouts only to candidates whose immediate scalarized score exceeds that of the current state. We then discard candidates whose best observed feasible terminal preference fails to improve upon the current sequence, focusing computation on candidates with evidence of downstream improvement. After the iterative editing phase, we run one additional terminal rollout from the best intermediate state and return the best feasible terminal sequence observed across the entire sampling process. Candidates whose rollouts are all infeasible receive zero estimated continuation value. If no candidate at a step has positive value, the current state is retained. If a complete run observes no feasible terminal, pCoMole returns the input sequence and counts the run as unsuccessful. The rollout approximation and pruning preserve the main guarantees of pCoMole (Appendix~\ref{sec:proofs}).

\section{Experiments}
Our experiments focus on molecule editing settings, where the goal is to improve multiple properties while preserving feasibility under a constrained edit process. In biomolecular design, reducing sequence length is a central practical objective to improve synthesizability, manufacturability, and delivery, whereas systematically increasing sequence length is typically misaligned with the design paradigm and offers limited utility. To demonstrate the efficacy of pCoMole, we evaluate on three shrinkage benchmarks: (i) GFP shrinkage, which designs shorter fluorescent proteins without sacrificing fluorescence-related objectives; (ii) Cas9 shrinkage, which aims to reduce protein length while preserving PAM specificity; and (iii) peptidomimetic binder design, where long peptide binders are compressed into shorter peptidomimetics while improving multiple drug-related properties. The primary benchmarks use task-specific computational predictors, hard constraints, and structure-based checks where applicable. For GFP, we additionally tested pCoMole-shortened eGFP designs in BL21 cells using 488-nm fluorescence imaging.

\subsection{Edit Flows Enable High-Quality Biological Sequence Editing}
To support sequence editing, we first trained five unconditional Edit Flow models: (i) a UniRef short-sequence Edit Flow, UniRef$_S$, trained on around 30k proteins no longer than 350 residues; (ii) a GFP Edit Flow obtained by fine-tuning UniRef$_S$ on around 800 known GFP sequences; (iii) a UniRef long-sequence Edit Flow, UniRef$_L$, trained on around 30k proteins between 800 residues and 1600 residues in length; (iv) a Cas9 Edit Flow trained from scratch on around 20k Cas9-family sequences; (v) a peptidomimetic Edit Flow trained from scratch on 28k peptidomimetics represented as SELFIES strings. These editors define variable-length CTMC dynamics through insertion, deletion, and substitution operations, serving as the fixed proposal processes for pCoMole to steer toward feasibility and user preferences.

\begin{table}[h!]
% \centering
\small
\renewcommand{\arraystretch}{1.2}
\caption{Edit Flow training and evaluation metrics across sequence domains.}
\label{tab:editflow_losses}
\begin{center}
        \begin{tabular}{ccccc}
        \toprule
        \textbf{\makecell{Edit Flow}}&\textbf{\makecell{Train\\Loss}}& \textbf{\makecell{Test\\Loss}}& \textbf{Diversity}& \textbf{\makecell{Validity\\Rate}} \\
        \midrule
        UniRef$_{S}$&594.39&602.47&0.68& \textemdash\\
        UniRef$_L$&3621.89&3573.07&0.45& \textemdash\\
        GFP &147.61&436.65&0.53&0.91\\
        Cas9 &2430.42&2498.37&0.25&0.93\\
        Peptidomimetic &36.26& 38.40&0.80&0.70\\
        \bottomrule 
        \end{tabular}
        \label{tab:editflow}
\end{center}
% \vspace{-2em}
\end{table}

Our Edit Flow models achieved low train-test gaps at their respective length scales, obtaining a lower loss when trained on domain-specific data (Table \ref{tab:editflow}). To assess unconditional generation quality, we sampled 1,000 sequences from each Edit Flow model and report validity and diversity (Appendix \ref{sec:diversity}). We further compared structural confidence by plotting the distribution of mean pLDDT scores for 1,000 sequences generated by the UniRef$_S$ Edit Flow against the corresponding 1,000 UniRef$_S$ test sequences used as inputs (Figure~\ref{fig:plddts}). The generated sequences closely match the pLDDT distribution of the originals, indicating that sampling preserves overall predicted fold quality.

\subsection{pCoMole Shrinks GFP While Optimizing Optical Functions Under Experimentally Motivated Constraints}
\label{sec:gfp}

We next evaluated pCoMole on shrinking GFP while preserving or improving optical function, a task in which fluorescence depends on preserving the characteristic $\beta$-barrel fold. We guide pCoMole to (i) increase predicted brightness, (ii) align the predicted excitation peak with the 488~nm laser line commonly used in fluorescence measurements, and (iii) encourage shrinkage via an explicit length-reduction objective. Three constraints are enforced: a GFP classifier constraint, an emission-range constraint requiring the predicted emission peak to fall within the green band, and a terminal length constraint implemented as a hard equality or a soft bound depending on the setting. Brightness, excitation, and emission are predicted with FPredX \citep{tam2022fpredx}. All experiments start from the same GFP sequence (UniProt ID: P42212; 238 residues) to enable consistent comparison and facilitate downstream experimental testing.

\begin{table}[t]
% \centering
\small
\renewcommand{\arraystretch}{1.3}
\caption{\textbf{GFP design comparison between pCoMole, RayGun, and SCISOR under a fixed length constraint.} For each method, we generate 100 candidates and report average metrics. All methods enforce a fixed terminal length of 213 residues and achieve a validity rate near 100\%.}
\label{tab:gfp_raygun}
\begin{center}
        \begin{tabular}{ccc}
        \toprule
        \textbf{Method}&\textbf{$\lvert \lambda_{\mathrm{exc}} - 488 \rvert$ ($\downarrow$)}& \textbf{Brightness}\\
        \midrule
        RayGun&22.3594& 9.4207\\
 SCISOR& 20.1279& 19.0333\\ 
        \makecell{pCoMole (UniRef$_S$ Edit Flow)}&\textbf{5.8213}& 43.5116\\
        \makecell{pCoMole (GFP Edit Flow)}&8.4372& \textbf{49.0121}\\
        \bottomrule
        \end{tabular}
\end{center}
\vspace{-1em}
\end{table}
\begin{table}[h!]
\centering
\caption{
\textbf{Compute-normalized time to first predicted-usable sequence.}
We report the time required to obtain a sequence satisfying each task-specific quality tier under a 3600-second trial budget. Lower is better. ``N.A.'' indicates that no valid sequence was found within the budget.
}
\label{tab:runtime_tiers}
\small
\begin{tabular}{llccc}
\toprule
Task & Method & Low Tier (s) & Medium Tier (s) & High Tier (s) \\
\midrule
\multirow{3}{*}{GFP}
     & RayGun  & N.A. & N.A. & N.A. \\
     & SCISOR  & 14   & 613  & N.A. \\
     & pCoMole & 266  & 333  & 665  \\
\midrule
\multirow{3}{*}{Cas9}
     & RayGun  & N.A.   & N.A.    & N.A. \\
     & SCISOR  & 6.86  & 76.61 & 240.28 \\
     & pCoMole & 131.69 & 129.15  & 138.96 \\
\bottomrule
\end{tabular}
\end{table}

We first compared pCoMole against RayGun \citep{devkota2024miniaturizing} and SCISOR \citep{baron2025diffusion}, the state-of-the-art methods for protein shrinkage (Table~\ref{tab:gfp_raygun}). Because both baselines require a fixed output length, we applied the same hard length constraint to pCoMole and set all methods to generate sequences of exactly 213 residues. pCoMole achieves substantially improved optical metrics, producing candidates with excitation peaks closer to 488~nm and markedly higher predicted brightness than either baseline. pCoMole remains competitive even with a UniRef-pre-trained Edit Flow without GFP-specific fine-tuning, indicating that gains arise from multi-objective, constraint-aware sampling rather than from a specialized prior. These results also highlight a practical advantage of pCoMole: it can target arbitrary terminal lengths while jointly optimizing functional objectives.

\begin{table}[h!]
\centering
\small
\setlength{\tabcolsep}{4pt}
\renewcommand{\arraystretch}{1.5}
\caption{\textbf{Objective guidance ablation for pCoMole on GFP design.} For each objective guidance setting, we report the average metrics over 100 pCoMole-shrunk GFP sequences. Length reduction guidance is applied in all settings.}
\label{tab:gfp_ablation}

\begin{tabular}{c|cccc}
\toprule
 \textbf{\makecell{Objective Guidance\\$\lvert \lambda_{\mathrm{exc}} - 488 \rvert$\hspace{1em}Brightness}}& \textbf{$\lvert \lambda_{\mathrm{exc}} - 488 \rvert$ ($\downarrow$)}&  \textbf{Brightness}& \textbf{Length}& \textbf{Validity Rate}\\
\midrule
 $\checkmark$\hspace{5em}$\checkmark$& 3.7557& 54.0159& 214& 1.0\\
  $\checkmark$\hspace{5em}$\times$& 2.9810& 31.2600& 213& 1.0\\
$\times$\hspace{5em}$\checkmark$& 12.9500& 62.1200& 214& 1.0\\
$\times$\hspace{5em}$\times$& 32.8200& 25.0200& 213& 1.0\\
\bottomrule
\end{tabular}
\end{table}

We then measured wall-clock time to produce the first predicted-usable design at three increasingly stringent quality tiers, with a 1-hour budget across all three methods (Table~\ref{tab:runtime_tiers}). For GFP, a design was predicted-usable only if it satisfied the terminal length, GFP-likeness, and green-emission constraints. Let $B_0$ denote the predicted brightness of the original GFP. The low, medium, and high tiers require predicted brightness at least $B_0$, $2B_0$, and $3B_0$, respectively, with excitation error $|\lambda_{\mathrm{exc}}-488| \leq 20$, $10$, and $5$ nm. SCISOR reached the low tier faster than pCoMole, consistent with its efficiency as a shrinkage-only generator. However, as the quality threshold tightens, pCoMole became more compute-efficient: it reached the medium tier faster than SCISOR and is the only method to reach the high tier within the budget. This result demonstrates that because pCoMole enforces quality objectives during generation, it can continue producing usable designs even as the threshold tightens, while other methods either cannot or do so much less efficiently.

Finally, we ablated objective guidance using a soft length constraint that limits shrinkage to at most 25 residues, since excessive truncation can destabilize the $\beta$-barrel fold (Table~\ref{tab:gfp_ablation}). Optimizing brightness and excitation alignment jointly yields the best overall trade-off, with omitting brightness guidance reducing predicted brightness and omitting excitation guidance shifting excitation away from 488~nm. AlphaFold3-predicted structures for representative designs preserve the canonical $\beta$-barrel architecture with substantial gains in targeted optical properties (Figure~\ref{fig:gfp}) \citep{abramson2024accurate}.

% To test retained fluorescence beyond FPredX predictions, we separately applied pCoMole to a 239-residue in-house eGFP construct (PC1847, PCMV-eGFP), distinct from the 238-residue UniProt sequence used for the computational benchmark. We tested two shortened designs in BL21 cells. pCoMole\_eGFP\_1 and pCoMole\_eGFP\_2 were each 229 residues in length after 10 deletions and contained one and two substitutions, respectively (Appendix~\ref{app:wetlab_sequences}). Both retained clear green fluorescence under 488-nm imaging, whereas a manually selected 228-residue negative control containing 11 deletions and two substitutions showed no comparable signal (Figure~\ref{fig:gfp_wetlab}).

To test retained fluorescence beyond FPredX predictions, we separately applied pCoMole to a 239-residue in-house eGFP construct (PC1847, PCMV-eGFP), distinct from the 238-residue UniProt sequence used for the computational benchmark. Shortened designs were tested in BL21 cells. Two designs, pCoMole\_eGFP\_1 and pCoMole\_eGFP\_2, were shortened to 229 residues through 10 deletions and contained one and two substitutions, respectively (Appendix~\ref{app:wetlab_sequences}). Both retained clear green fluorescence under 488-nm imaging, whereas a manually selected 228-residue negative control containing 11 deletions and two substitutions showed no comparable signal (Figure~\ref{fig:gfp_wetlab}).

\begin{figure}[h!]
    \centering
    \includegraphics[width=\linewidth]{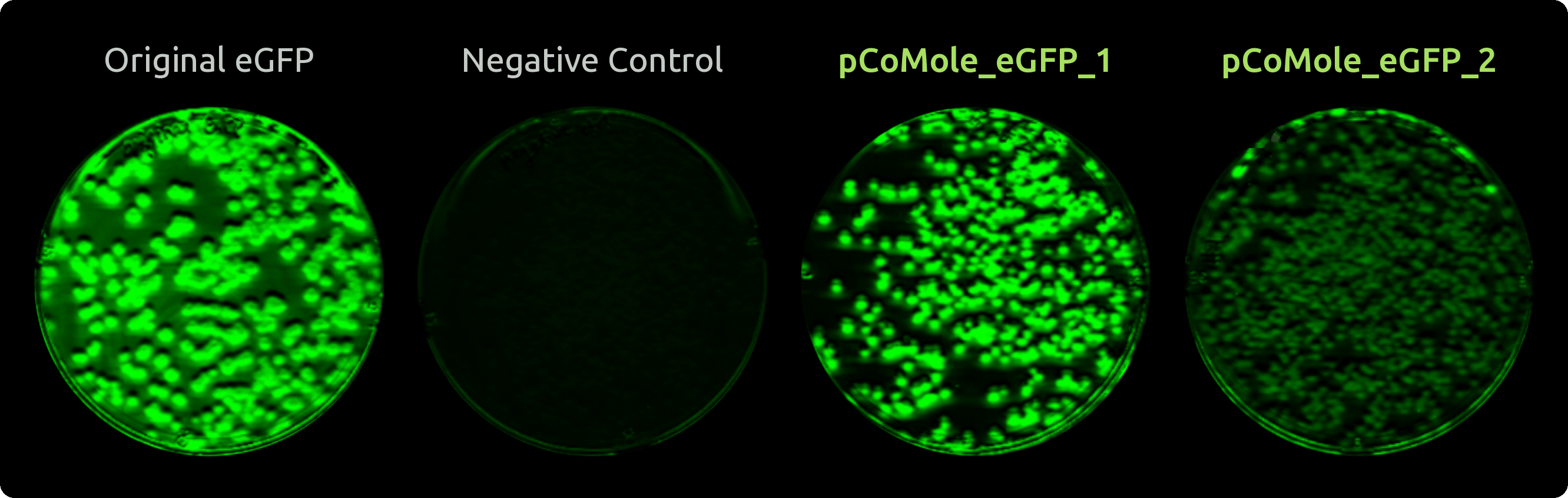}
    \caption{\textbf{Wet-lab fluorescence validation of pCoMole-shortened eGFP designs.} 488-nm fluorescence images of the 239-residue in-house eGFP construct, a manually selected 228-residue negative control, and two 229-residue pCoMole designs. pCoMole\_eGFP\_1 contains 10 deletions and one substitution; pCoMole\_eGFP\_2 contains 10 deletions and two substitutions.}
    \label{fig:gfp_wetlab}
\end{figure}

\subsection{pCoMole Shrinks Cas9s While Maintaining PAM Specificity}
\label{sec:cas9}

CRISPR-Cas9 is a transformative genome-editing platform, but the large size of Cas9 endonucleases poses a fundamental challenge for delivery and packaging, particularly for viral vectors \citep{edraki2019compact, ran2015vivo, behr2021vivo, kabadi2024aav, wang2020crispr}. Systematic shrinkage is therefore highly desirable, but must preserve critical functional properties, most notably PAM specificity, which determines the genomic targets a Cas9 can recognize \citep{collias2021crispr, chatterjee2018minimal, chatterjee2020engineered, zhao2023pam}. We task pCoMole with shrinking Cas9 sequences while maintaining (i) a high likelihood that the edited sequence remains a valid Cas9 and (ii) the original Cas9's PAM specificity.

Here, pCoMole was guided by three objectives: maximizing Cas9-likelihood, minimizing the cross-entropy between the predicted PAM distribution of the edited sequence and that of the input, and maximizing length reduction. We further imposed hard terminal constraints requiring exact PAM matching and a shorter terminal length. Cas9-likelihood was evaluated with a pre-trained classifier distinguishing valid Cas9s from UniRef, other Cas proteins, and corrupted variants, and PAM prediction uses Protein2PAM \citep{nayfach2025engineering}.

We first evaluated pCoMole on four well-characterized Cas9s: SpCas9, St1Cas9, St3Cas9, and GeoCas9, \citep{jinek2012programmable, muller2016streptococcus, harrington2017thermostable} generating 50 shrunk variants per input while strictly enforcing PAM preservation (Table~\ref{tab:cas9_main_table}). Across all tested Cas9s, pCoMole achieves substantial length reductions while maintaining high Cas9-likelihood as well as the predicted PAMs of the original input sequences (Figure~\ref{fig:designed_cas9}).

We then repeated the compute-quality analysis from Section~\ref{sec:gfp} on Cas9 shrinkage, where the tiers now correspond to progressively greater length reductions under exact PAM preservation and maintained Cas9-likelihood (Table~\ref{tab:runtime_tiers}). Specifically, the low, medium, and high tiers require terminal lengths $L\in[930, 970]$, $[890, 930]$, and $[850, 890]$ amino acids, respectively. While both pCoMole and SCISOR were able to reach all three tiers within budget, pCoMole did so at nearly constant runtimes, as its Doob-$h$ guidance enforces the quality criteria during generation. SCISOR's sampling, by contrast, is unguided, and thus its samples increasingly fail as the constraint becomes more demanding. The result is that SCISOR's time-to-usable rises sharply across tiers while pCoMole's stays nearly flat (132s, 129s, 139s), with pCoMole reaching the high tier substantially faster than SCISOR. A complementary fixed-length benchmark (984 $\rightarrow$ 934 residues) confirms this gap: pCoMole achieves 100\% PAM match against SCISOR's 35\% and RayGun's 0\% (Appendix Table~\ref{tab:cas9_benchmark}).

\begin{table*}[h!]
\footnotesize
\caption{\textbf{pCoMole shrinks common Cas9 sequences while maintaining PAM specificity.}
``Cas9 Likelihood'' is the score of a pre-trained Cas9 classifier. 
``PAM Distribution CE'' is a cross-entropy between the predicted PAM distribution of the input and that of the generated sequence.
``PAM Match Rate'' is the fraction of generated sequences whose predicted PAM exactly matches the input.
``PAM matched' is the predicted PAM which pCoMole maintained during shrinkage.}
\label{tab:cas9_main_table}
\begin{center}
        \begin{tabular}{c|cccc|c}
        \toprule
        \textbf{Input Cas9} &
        \textbf{Length Change} &
        \textbf{Cas9 Likelihood} &
        \textbf{PAM Distr. CE} &
        \textbf{PAM Match Rate} &
        \textbf{PAM matched}\\
        \midrule
        SpCas9    &
        $1368 \rightarrow 1211\;(-157)$ &
        $1.00 \rightarrow 0.97$ &
        $0.97 \rightarrow 0.90$ &
        1.00 &
        NGG \\
        St1Cas9 &
        $1121 \rightarrow 1023\;(-98)$ &
        $1.00 \rightarrow 0.98$ &
        $0.91 \rightarrow 0.89$ &
        1.00 &
        NNANAA \\
        St3Cas9 &
        $1388 \rightarrow 1301\;(-87)$ &
        $1.00 \rightarrow 0.95$ &
        $0.94 \rightarrow 0.90$ &
        1.00 &
        NGGNG \\
        GeoCas9   &
        $1087 \rightarrow 1035\;(-52)$ &
        $0.96 \rightarrow 0.95$ &
        $0.93 \rightarrow 0.89$ &
        1.00 &
        NNNNCNAA \\
        \bottomrule
        \end{tabular}
\end{center}
\end{table*}

To evaluate PAM preservation independently of Protein2PAM, we used CICERO only for post hoc evaluation of St3Cas9 and GeoCas9 designs \citep{fang2026cicero}. We evaluated 25 designs for each ortholog and condition. The guidance-disabled control removed the PAM objective and hard PAM constraint while keeping the remaining pCoMole settings fixed. PAM guidance substantially improved PAM preservation under the independent CICERO evaluator, increasing exact PAM match rates from 4\% to 88\% for St3Cas9 and from 20\% to 72\% for GeoCas9 (Appendix Table~\ref{tab:cicero}). PWM similarity, which measures similarity in predicted PAM distribution between generated and original sequences, changed in the same direction. These held-out evaluations support PAM preservation beyond the Protein2PAM guidance model.

\subsection{pCoMole Develops Short Peptidomimetic Binders Optimizing Seven Drug-Related Properties}

Finally, via a peptidomimetic Edit Flow, we used pCoMole to shrink known peptide binders into short peptidomimetic candidates that retain key interaction motifs while mitigating common peptide liabilities such as poor stability, limited permeability, and unfavorable pharmacokinetics \cite{cunningham2017peptides, wang2022therapeutic}. Starting from each known binder, pCoMole jointly optimized seven drug-related properties: non-toxicity, solubility, permeability, half-life, binding affinity, motif score, and motif specificity, where the latter two reward pocket retention and penalize off-target binding. We enforced a soft shrinkage constraint to prevent degenerate collapse and a hard peptidomimetic validity constraint computed via RDKit. Property predictions use a chemotype-adaptive ensemble: each candidate SMILES received a peptide-likeness score $\alpha(x)\in[0,1]$, and property scores were computed as a convex mixture of peptide-specialized PeptiVerse predictions \citep{zhang2026peptiverse} and small-molecule Admetica predictions \citep{admetica}, weighting peptide-specific predictors more heavily for peptide-like candidates.

\begin{table*}[t]
\centering
\scriptsize
\setlength{\tabcolsep}{4pt}
\renewcommand{\arraystretch}{1.2}
\caption{
\textbf{pCoMole designs shorter peptidomimetics with improved predicted properties.}
For each target, we compare the original peptide drug binder with the average over 100 pCoMole-designed peptidomimetics. Length denotes the number of tokenized SMILES tokens. All pCoMole designs satisfy the specified task constraints. For GLP-1R, motif score and specificity are computed relative to Semaglutide and are therefore reported only for Semaglutide and its pCoMole designs.
}
\vspace{0.5em}
\label{tab:mimetics_drug}

\begin{tabular}{cccccccccc}
\toprule
\textbf{Target}&\textbf{Ligand}& \textbf{Length}& \textbf{Non-Toxicity}&  \textbf{Solubility}& \textbf{Permeability}& \textbf{Half-life (h)}& \textbf{Affinity}& \textbf{Motif Score}&  \textbf{Specificity}\\
\midrule
\multirow{2}{*}{PTH1R}&Teriparatide& 259& 0.8890& 0.6905& 0.2793& 2.8301& 0.9148& 0.2696& 0.9834\\
&pCoMole& 217& 0.8729& 0.7268& 0.2959& 2.9488& 0.9358& 0.3593& 0.9706\\
\midrule
\multirow{2}{*}{p53}&p28& 170& 0.8533& 0.6737& 0.2627& 2.2015& 0.6369& 0.5418& 0.9873\\
&pCoMole& 133& 0.8076& 0.6309& 0.3071& 3.8517& 0.6885& 0.5014& 0.9649\\
\midrule
\multirow{6}{*}{GLP-1R}&Semaglutide& 248& 0.8872& 0.7656& 0.2583& 1.8571& 0.9467& 0.4691& 0.9827\\
&pCoMole& 186& 0.8122& 0.7240& 0.2985& 4.0277& 0.9367& 0.5212& 0.9778\\
&Liraglutide& 237& 0.7854& 0.8307& 0.2609& 2.0808& 0.8440& \textemdash & \textemdash \\
&Exenatide& 256& 0.8319& 0.7547& 0.2319& 2.0663& 0.8949& \textemdash & \textemdash \\
&Lixisenatide& 286& 0.7204& 0.8392& 0.2215& 2.7535& 0.8280& \textemdash & \textemdash \\
&Tirzepatide& 273& 0.9491& 0.8303& 0.2573& 1.7149& 0.9798& \textemdash & \textemdash\\
\bottomrule
\end{tabular}
\end{table*}

On 15 PDB targets with known peptide binders (Appendix Table~\ref{tab:mimetics_pdb}), pCoMole generated 100 candidates per target that satisfy both terminal constraints while achieving substantial length reduction. These candidates exhibited higher permeability, longer predicted half-life, and stronger binding affinity than the original peptides, with motif and specificity scores largely preserved or improved. Modest decreases in non-toxicity or solubility on some targets reflect realistic multi-objective trade-offs. Affinity guidance used PeptiVerse, and Boltz-2 and AutoDock VINA were reserved for post hoc evaluation after generation. We scored generated designs with Boltz-2 and docked representative high-scoring designs into each target using AutoDock VINA (Figure~\ref{fig:pcomole}D, Figures~\ref{fig:dock1}-\ref{fig:dock4}) \citep{passaro2025boltz, trott2010autodock}. Despite shorter tokenized SMILES lengths, docked poses occupied the same binding pockets as the known peptides, sit adjacent to annotated motif regions, and engaged highly overlapping residue sets, supporting pocket-level recognition beyond the in-loop affinity predictor.

We further applied pCoMole to clinically relevant peptide drugs on three therapeutic targets (Table~\ref{tab:mimetics_drug}): teriparatide for PTH1R \citep{neer2001effect}, p28 for p53 \citep{warso2013first, yamada2013p28}, and semaglutide for GLP-1R, with comparison against other GLP-1R peptide drugs \citep{marso2016semaglutide, wilding2021once}. The known therapeutic binders served as incumbent reference molecules for this editing case study. pCoMole produced substantially shorter, constraint-satisfying peptidomimetics with improved predicted permeability and half-life while maintaining strong predicted affinity. Most notably, for GLP-1R, semaglutide-derived peptidomimetics remained substantially shorter while retaining favorable predicted property profiles relative to the reference drugs.

\subsection{Ablation Studies}
We provide detailed ablation studies in Appendix~\ref{sec:ablation}, characterizing pCoMole's sensitivity to terminal constraints, scalarization choice, deletion-rate scaling, insertion handling, and sampling hyperparameters.

\section{Discussion}

In this work, we present \textbf{pCoMole}, a framework for constraint-aware, multi-objective \emph{sequence editing} that steers a pre-trained Edit Flow toward feasible, Pareto-aligned outcomes under hard terminal constraints. By operating on variable-length edit trajectories rather than generating sequences \emph{de novo}, pCoMole enables systematic shrinkage while preserving task-critical properties. Across GFP, diverse Cas9 orthologs, and peptidomimetic binder compression, \textit{in silico} experiments showed that pCoMole can shorten biomolecular sequences while maintaining stringent specifications such as PAM preference or binding motifs while navigating realistic multi-objective trade-offs. We further experimentally validated shortened eGFP designs, with two 229-residue pCoMole variants retaining clear green fluorescence in BL21 cells after 10 deletions. This direct validation demonstrates that pCoMole can produce shortened proteins while retaining targeted biological phenotypes beyond computational guidance oracles. As such, despite higher computational cost from rollout-based guidance, pCoMole provides a principled bridge between discrete generative modeling and biomolecular optimization.

\section{Author Contributions}
T.C. formulated the pCoMole algorithm and conducted the GFP and peptidomimetic model training, computational experiments, and analyses. M.H. conducted the Cas9 model training, computational experiments, and analyses. L.Z. performed eGFP wet-lab experiments. P.C. conceived, designed, directed, and supervised the study. 

% \paragraph{Acknowledgments.} 

% \paragraph{Funding Statement.} 

% \paragraph{Competing Interests.} 

%%%%%%%%%%%%%%%%%%%%%%%%%%%%%%%%%%%%%%%%%%%%%%%%%%%%%%%%%%%%
%%%%%%%%%%%%%%%%%%%%%%%%%%%%%%%%%%%%%%%%%%%%%%%%%%%%%%%%%%%%
%%%%%%%%%%%%%%%%%%%%%%%%%%%%%%%%%%%%%%%%%%%%%%%%%%%%%%%%%%%%
%%%%%%%%%%%%%%%%%%%%%%%%%%%%%%%%%%%%%%%%%%%%%%%%%%%%%%%%%%%%
\bibliographystyle{unsrtnat}
\bibliography{citation}  
%%%%%%%%%%%%%%%%%%%%%%%%%%%%%%%%%%%%%%%%%%%%%%%%%%%%%%%%%%%%
%%%%%%%%%%%%%%%%%%%%%%%%%%%%%%%%%%%%%%%%%%%%%%%%%%%%%%%%%%%%
%%%%%%%%%%%%%%%%%%%%%%%%%%%%%%%%%%%%%%%%%%%%%%%%%%%%%%%%%%%%
%%%%%%%%%%%%%%%%%%%%%%%%%%%%%%%%%%%%%%%%%%%%%%%%%%%%%%%%%%%%

%%%%%%%%%%%%%%%%%%%%%%%%%%%%%%%%%%%%%%%%%%%%%%%%%%%%%%%%%%%%
\newpage
\appendix
\renewcommand{\thefigure}{S\arabic{figure}}
\setcounter{figure}{0}
\setcounter{table}{0}
\renewcommand{\thetable}{S\arabic{table}}

\begin{center}
    {\LARGE \bfseries Supplementary Information}
\end{center}
\vspace{1em}

\section{Related Work}
\label{app:related_work}
\paragraph{Molecule Editing Frameworks.}
Protein miniaturization methods such as RayGun \citep{devkota2024miniaturizing} and SCISOR \citep{baron2025diffusion} focus on producing shorter proteins, but are not designed for multi-objective editing that simultaneously optimizes developability and functional criteria; we benchmark both against pCoMole in Sections~\ref{sec:gfp} and~\ref{sec:cas9}. In contrast, many small-molecule editing and optimization methods, including MolDQN \citep{zhou2019optimization} and GCPN \citep{you2018graph}, are formulated as online reinforcement learning, relying on repeated reward evaluation during training or search. pCoMole instead operates in an offline setting with costly per-sequence evaluation, making direct comparison to online methods inappropriate.

\paragraph{Constraint-based Generation.}
Constraint-aware generative methods typically enforce feasibility through projection, constrained dynamics, or constraint-driven fine-tuning. Representative diffusion and flow-matching approaches include CDD \citep{cardei2025constrained}, DiffOPT \citep{kong2024diffusion}, and ALF$^2$ \citep{gutjahr2025constrained}. However, these methods are not formulated as sequence-editing processes with explicit insertion and deletion for variable-length control, and they generally optimize a single scalar objective with constraint penalties rather than providing explicit Pareto-style multi-objective control. Many recent constrained protein design methods further rely on structure-guided constraints, differing from pCoMole’s sequence-only editing setting.

\paragraph{Multi-Objective Optimization.}
Multi-objective guided generation is commonly studied in sequential design settings, where candidates are iteratively proposed and refined using new evaluations \citep{gruver2023protein, jain2023multi, stanton2022accelerating, ahmadianshalchi2024pareto}. Bayesian optimization methods fit surrogate models and select candidates using acquisition functions \citep{yu2020gradient, shahriari2015taking}, with multi-objective extensions based on expected hypervolume improvement \citep{emmerich2008computation}, information-theoretic criteria \citep{belakaria2021output}, or scalarization strategies \citep{knowles2006parego, zhang2007moea, paria2020flexible}. While effective for black-box trade-off exploration, these approaches typically treat feasibility through soft penalties or heuristics and do not model feasibility-gated sampling under hard biochemical or manufacturability constraints.

\paragraph{Constrained Multi-Objective Molecular Optimization.}
Recent frameworks such as CMOMO \citep{xia2025cmomo} and MCEMOL \citep{lin2026multi} study constrained multi-objective molecular optimization, but target different pipelines than ours. CMOMO performs iterative evolutionary search over RDKit-based fitness functions, which does not transfer to sequence editing under biochemical and manufacturability constraints. We do not benchmark against MCEMOL due to the lack of a public implementation. Other constrained multi-objective methods based on search or reinforcement learning, including MolSearch \citep{sun2022molsearch}, ParetoDrug \citep{yang2024enabling}, and REINVENT~4 \citep{loeffler2024reinvent}, likewise assume repeated online evaluations, whereas pCoMole performs offline, edit-based sampling from a fixed pre-trained editor.

These distinctions define the setting addressed by pCoMole: variable-length sequence editing with explicit terminal feasibility and controllable multi-objective preferences.

\section{Log-domain rollout estimation and guidance}
\label{app:log_domain}

Our guided editing procedure relies on a terminal preference function of the form $G(x)=\exp(\ell(x))$, where $\ell(x)$ is a scalar score derived from multi-objective utility and feasibility constraints. In practice, objective scores may be unbounded or have large magnitude, and direct computation of $\exp(\ell(x))$ can overflow. We therefore implement value estimation and guided selection entirely in the log domain, which removes the need to normalize objectives to a fixed range and yields numerically stable computations.

\paragraph{Log-preference and feasibility.}
Let $U(x)$ denote the scalar utility used for guidance, such as the augmented Tchebycheff utility defined in the main text. We define the log-preference as
\begin{equation}
\ell(x)\ \coloneqq\ \beta\,U(x)\ +\ \log \mathbb{I}[x\in\mathcal{F}],
\label{eq:log_pref}
\end{equation}
where $\beta>0$ is an inverse-temperature and $\mathbb{I}[x\in\mathcal{F}]$ is the feasibility indicator. We implement the indicator in log space by setting $\log \mathbb{I}[x\in\mathcal{F}]=0$ for feasible $x$ and $-\infty$ otherwise. This convention ensures that infeasible terminal states contribute zero weight in all subsequent computations.

\paragraph{Log-domain value function and Monte Carlo estimator.}
Doob-$h$ guidance depends on the space-time harmonic function $h_t(x)=\mathbb{E}[G(X_1)\mid X_t=x]$. In log space, we define
\begin{equation}
\log h_t(x)\ =\ \log \mathbb{E}\!\left[\exp(\ell(X_1)) \,\middle|\, X_t=x\right].
\label{eq:log_h_def}
\end{equation}
Given $R$ independent rollouts from $(t,x)$ under the base Edit Flow, producing terminal states $\{X_1^{(r)}\}_{r=1}^R$, the standard Monte Carlo estimator of $h_t(x)$ is $\widehat{h}_t(x)=\frac{1}{R}\sum_{r=1}^R \exp(\ell(X_1^{(r)}))$. We compute its logarithm stably using the log-mean-exp operator
\begin{equation}
\log \widehat{h}_t(x)\ =\ \log\sum_{r=1}^R \exp(\ell(X_1^{(r)}))\ -\ \log R,
\label{eq:log_h_hat}
\end{equation}
where the summation is implemented via $\mathrm{logsumexp}$. Although $\log \widehat{h}_t(x)$ is generally a biased estimator of $\log h_t(x)$ due to Jensen's inequality, it is consistent as $R\to\infty$ and remains well-defined even when $\ell(X_1^{(r)})$ takes large values. When all rollouts terminate in infeasible states, \eqref{eq:log_h_hat} evaluates to $-\infty$, which correctly assigns zero continuation value under the feasibility-weighted preference.

\paragraph{Guided selection from candidate edits in log space.}
At time $t$ and state $x_t$, we restrict attention to a candidate set $\{y_j\}_{j=1}^M$ generated by sampling admissible edits from the base rate field $u_t(\cdot\mid x_t)$. Under the Doob-$h$ transform, the guided rate for a transition $x_t\to y$ is proportional to $u_t(y\mid x_t)\,h_t(y)$. Replacing $h_t$ by the rollout estimator and working in log space yields the candidate score
\begin{equation}
s_j\ \coloneqq\ \log u_t(y_j\mid x_t)\ +\ \log \widehat{h}_t(y_j).
\label{eq:log_score}
\end{equation}
We sample the next state from the categorical distribution obtained by normalizing these scores,
\begin{equation}
\mathbb{P}(Y=y_j\mid x_t)\ =\ \frac{\exp(s_j)}{\sum_{m=1}^{M}\exp(s_m)}
\ =\ \exp\!\left(s_j - \mathrm{logsumexp}_{m\in\{1,\dots,M\}}\, s_m\right),
\label{eq:log_softmax}
\end{equation}
which is computed stably by subtracting $\max_m s_m$ before exponentiation. This computation depends only on differences of log scores and is therefore invariant to adding a constant offset to $\ell(x)$ or to all scores $\{s_j\}$, which is useful when utilities are defined on arbitrary scales.

\paragraph{Stability considerations.}
The log-domain implementation prevents overflow in two places where it commonly occurs. First, \eqref{eq:log_h_hat} avoids forming $\exp(\ell)$ explicitly by using $\mathrm{logsumexp}$. Second, \eqref{eq:log_softmax} avoids overflow and underflow when converting scores into probabilities by applying the standard stabilized softmax. In addition, infeasible rollouts propagate cleanly as $-\infty$ scores and therefore receive zero probability without requiring special-case branching. These properties allow pCoMole to operate with unnormalized objective scores and large-magnitude utilities while preserving the same guided selection rule as the corresponding computations in the original domain.

\section{Theoretical Proofs}
\label{sec:proofs}

This section states guarantees for the idealized pCoMole algorithm that uses exact Doob-$h$ guidance and exact value functions, and then describes what properties are preserved under the rollout-based approximations and practical safeguards used in pCoMole. We follow the notations in Section \ref{sec:method}.

\begin{proposition}[Doob-$h$ transform yields preference-tilted terminals]
\label{prop:doob_transform_proof}
Assume $G(X_1)$ is integrable under the base Edit Flow and $h_t(x)>0$ for all $(t,x)$ reachable from $x_0$, where $h_t(x)\coloneqq \mathbb{E}[G(X_1)\mid X_t=x]$. The CTMC with rates \eqref{eq:doob_rates} has terminal distribution $q_G(\cdot\mid x_0)$ defined in \eqref{eq:tilted_terminal}.
\end{proposition}
\begin{proof}
Let $(X_t)_{t\in[0,1]}$ be the base time-inhomogeneous CTMC on $\mathcal{X}$ with off-diagonal rates $u_t(y\mid x)$ and initial condition $X_0=x_0$. Let $(\Omega,\mathcal{F},\mathbb{P})$ denote its probability space and let $\mathcal{F}_t=\sigma(X_s:0\le s\le t)$ be the natural filtration. For $t\in[0,1]$ and $x\in\mathcal{X}$ define
\[
h_t(x)\coloneqq \mathbb{E}_{\mathbb{P}}\!\left[G(X_1)\mid X_t=x\right],
\]
which is well-defined by integrability of $G(X_1)$ and strictly positive on the reachable set by assumption. Define the likelihood-ratio process
\begin{equation}
L_t \coloneqq \frac{h_t(X_t)}{h_0(x_0)}, \qquad t\in[0,1].
\label{eq:Lt_def}
\end{equation}
By the Markov property of $X$, for any $0\le s\le t\le 1$,
\begin{equation}
\mathbb{E}_{\mathbb{P}}\!\left[h_t(X_t)\mid \mathcal{F}_s\right]
=
\mathbb{E}_{\mathbb{P}}\!\left[\mathbb{E}_{\mathbb{P}}[G(X_1)\mid X_t]\mid \mathcal{F}_s\right]
=
\mathbb{E}_{\mathbb{P}}\!\left[G(X_1)\mid \mathcal{F}_s\right]
=
\mathbb{E}_{\mathbb{P}}\!\left[G(X_1)\mid X_s\right]
=
h_s(X_s),
\label{eq:martingale_step}
\end{equation}
so $(h_t(X_t))_{t\in[0,1]}$ is a $\mathbb{P}$-martingale and therefore $(L_t)_{t\in[0,1]}$ is a nonnegative $\mathbb{P}$-martingale with $\mathbb{E}_{\mathbb{P}}[L_t]=1$. We use $L_1$ to define a new measure $\mathbb{Q}$ on $(\Omega,\mathcal{F}_1)$ by
\begin{equation}
\frac{\mathrm{d}\mathbb{Q}}{\mathrm{d}\mathbb{P}}\ \coloneqq\ L_1
=
\frac{h_1(X_1)}{h_0(x_0)}
=
\frac{G(X_1)}{h_0(x_0)},
\label{eq:change_measure}
\end{equation}
where we used $h_1(x)=\mathbb{E}_{\mathbb{P}}[G(X_1)\mid X_1=x]=G(x)$. For any measurable set $A\subseteq \mathcal{X}$,
\begin{equation}
\mathbb{Q}(X_1\in A)
=
\mathbb{E}_{\mathbb{P}}\!\left[\mathbb{I}[X_1\in A]\frac{G(X_1)}{h_0(x_0)}\right]
=
\frac{1}{h_0(x_0)}\sum_{x_1\in A}\mathbb{P}(X_1=x_1)\,G(x_1),
\label{eq:tilt_marginal}
\end{equation}
which shows that the terminal marginal under $\mathbb{Q}$ is exactly the tilted distribution in \eqref{eq:tilted_terminal}.

It remains to identify the dynamics of $(X_t)$ under $\mathbb{Q}$. Fix $t\in[0,1)$ and a reachable state $x$. For any $y\neq x$, the base CTMC satisfies
\begin{equation}
\mathbb{P}(X_{t+\Delta}=y\mid X_t=x)=u_t(y\mid x)\,\Delta+o(\Delta), \qquad \Delta\downarrow 0.
\label{eq:base_jump}
\end{equation}
Using Bayes' rule under the change of measure \eqref{eq:change_measure} and the Markov property,
\begin{align}
\mathbb{Q}(X_{t+\Delta}=y\mid X_t=x)
&=
\frac{\mathbb{E}_{\mathbb{P}}\!\left[\mathbb{I}[X_{t+\Delta}=y]\,L_1\mid X_t=x\right]}
{\mathbb{E}_{\mathbb{P}}\!\left[L_1\mid X_t=x\right]}
=
\frac{\mathbb{E}_{\mathbb{P}}\!\left[\mathbb{I}[X_{t+\Delta}=y]\;\frac{G(X_1)}{h_0(x_0)}\mid X_t=x\right]}
{\mathbb{E}_{\mathbb{P}}\!\left[\frac{G(X_1)}{h_0(x_0)}\mid X_t=x\right]} \nonumber\\
&=
\frac{\mathbb{P}(X_{t+\Delta}=y\mid X_t=x)\;\mathbb{E}_{\mathbb{P}}[G(X_1)\mid X_{t+\Delta}=y]}
{\mathbb{E}_{\mathbb{P}}[G(X_1)\mid X_t=x]}
=
\frac{\mathbb{P}(X_{t+\Delta}=y\mid X_t=x)\;h_{t+\Delta}(y)}
{h_t(x)}.
\label{eq:guided_jump}
\end{align}
Since $h_{t+\Delta}(y)=h_t(y)+o(1)$ as $\Delta\downarrow 0$ for fixed $y$, combining \eqref{eq:base_jump} and \eqref{eq:guided_jump} yields
\[
\mathbb{Q}(X_{t+\Delta}=y\mid X_t=x)
=
u_t(y\mid x)\,\frac{h_t(y)}{h_t(x)}\,\Delta + o(\Delta).
\]
Therefore, under $\mathbb{Q}$ the process is a CTMC with off-diagonal rates $u_t^{G}(y\mid x)=u_t(y\mid x)\,h_t(y)/h_t(x)$, which are precisely the Doob-$h$ guided rates in \eqref{eq:doob_rates}. Together with \eqref{eq:tilt_marginal}, this proves that the CTMC driven by \eqref{eq:doob_rates} has terminal distribution $q_G(\cdot\mid x_0)$.
\end{proof}
\begin{proposition}[ATC scalarization: Pareto optimality and coverage]
\label{prop:atc_pareto}
Let $\mathcal{F}\subseteq\mathcal{X}$ be the feasible set and let $U(x)=\min_{k\in\{1,\dots,K\}} \omega_k(f_k(x)-r_k)+\rho\sum_{k=1}^K \omega_k(f_k(x)-r_k)$ with $\omega\in\mathbb{R}_{>0}^K$ and $\rho>0$. If $x^\star\in\mathcal{F}$ maximizes $U(x)$ over $\mathcal{F}$, then $x^\star$ is Pareto-optimal in $\mathcal{F}$. Assume additionally that $f_k(x)>r_k$ for all $x\in\mathcal{F}$ and all $k$. Then for every Pareto-optimal $x^\star\in\mathcal{F}$, the weight choice
\begin{equation}
\omega_k \coloneqq \frac{1}{f_k(x^\star)-r_k},\qquad k\in\{1,\dots,K\},
\label{eq:cheby_weights}
\end{equation}
makes $x^\star$ a maximizer of the Chebyshev term $T_\omega(x)\coloneqq \min_k \omega_k(f_k(x)-r_k)$ over $\mathcal{F}$.
Moreover, if $S_\omega(x)\coloneqq \sum_{k=1}^K \omega_k(f_k(x)-r_k)$ is bounded above on $\mathcal{F}$ by $S_{\max}$ and the margin
\begin{equation}
\delta \coloneqq T_\omega(x^\star) - \sup_{x\in\mathcal{F}\setminus\{x^\star\}} T_\omega(x)
\label{eq:cheby_margin}
\end{equation}
is positive, then $x^\star$ also maximizes $U(x)=T_\omega(x)+\rho S_\omega(x)$ over $\mathcal{F}$ for all $\rho\in\left(0,\;\delta/(S_{\max}-S_\omega(x^\star))\right)$.
\end{proposition}

\begin{proof}
Assume $x^\star\in\mathcal{F}$ maximizes $U$ for some $\omega\in\mathbb{R}_{>0}^K$ and $\rho>0$. Suppose for contradiction that $x^\star$ is not Pareto-optimal. Then there exists $y\in\mathcal{F}$ such that $f_k(y)\ge f_k(x^\star)$ for all $k$ and $f_j(y)>f_j(x^\star)$ for at least one index $j$. Since $\omega_k>0$, we have $\omega_k(f_k(y)-r_k)\ge \omega_k(f_k(x^\star)-r_k)$ for all $k$, which implies
\[
\min_k \omega_k(f_k(y)-r_k)\ \ge\ \min_k \omega_k(f_k(x^\star)-r_k).
\]
In addition, $\omega_j(f_j(y)-r_j)>\omega_j(f_j(x^\star)-r_j)$ and all other summands are nondecreasing, so
\[
\sum_{k=1}^K \omega_k(f_k(y)-r_k)\ >\ \sum_{k=1}^K \omega_k(f_k(x^\star)-r_k).
\]
Multiplying the strict inequality by $\rho>0$ and adding the inequality for the minimum term yields $U(y)>U(x^\star)$, contradicting maximality of $x^\star$. Therefore $x^\star$ is Pareto-optimal.

For the coverage statement, fix a Pareto-optimal $x^\star\in\mathcal{F}$ and assume $f_k(x)>r_k$ for all $x\in\mathcal{F}$ and all $k$, so the weights in \eqref{eq:cheby_weights} are well-defined and positive. For this choice, $\omega_k(f_k(x^\star)-r_k)=1$ for every $k$, hence $T_\omega(x^\star)=1$. Consider any $x\in\mathcal{F}$ with $x\neq x^\star$. Since $x^\star$ is Pareto-optimal, $x$ cannot Pareto-dominate $x^\star$, so there exists at least one index $k$ with $f_k(x)<f_k(x^\star)$. For that index,
\[
\omega_k(f_k(x)-r_k)\ <\ \omega_k(f_k(x^\star)-r_k)\ =\ 1,
\]
which implies $T_\omega(x)=\min_j \omega_j(f_j(x)-r_j)<1=T_\omega(x^\star)$. Therefore $x^\star$ maximizes $T_\omega$ over $\mathcal{F}$.

Finally, assume $S_\omega$ is bounded above on $\mathcal{F}$ by $S_{\max}$ and the margin $\delta$ in \eqref{eq:cheby_margin} is positive. For any $x\in\mathcal{F}\setminus\{x^\star\}$, we have $T_\omega(x)\le T_\omega(x^\star)-\delta$ and $S_\omega(x)\le S_{\max}$, so
\[
U(x)\ =\ T_\omega(x)+\rho S_\omega(x)\ \le\ T_\omega(x^\star)-\delta+\rho S_{\max}.
\]
On the other hand, $U(x^\star)=T_\omega(x^\star)+\rho S_\omega(x^\star)$. If $\rho<\delta/(S_{\max}-S_\omega(x^\star))$, then $T_\omega(x^\star)-\delta+\rho S_{\max}<T_\omega(x^\star)+\rho S_\omega(x^\star)=U(x^\star)$, which shows $U(x)<U(x^\star)$ for all $x\neq x^\star$ and therefore $x^\star$ maximizes $U$ over $\mathcal{F}$.
\end{proof}

\paragraph{Consistency of rollout-based guidance.}
The practical algorithm replaces $h_t(x)=\mathbb{E}[G(X_1)\mid X_t=x]$ by the Monte Carlo estimator $\widehat{h}_t(x)$ and restricts attention to a finite candidate set $\mathcal{C}(x)\subset\mathcal{E}(x)$. The following statement controls the one-step error induced by multiplicative approximation of $h_t$ on $\mathcal{C}(x)$.

\begin{proposition}[One-step kernel stability]
\label{prop:approx_kernel}
Fix $(t,x)$ and a candidate set $\mathcal{C}(x)\subset\mathcal{E}(x)$. Let $P(\cdot\mid x)$ be the exact Doob-$h$ transition distribution restricted to $\mathcal{C}(x)$, defined by
\[
P(y\mid x)=\frac{u_t(y\mid x)\,h_t(y)}{\sum_{y'\in\mathcal{C}(x)}u_t(y'\mid x)\,h_t(y')},\qquad y\in\mathcal{C}(x).
\]
Let $\widehat{P}(\cdot\mid x)$ be the rollout-based transition distribution defined by
\[
\widehat{P}(y\mid x)=\frac{u_t(y\mid x)\,\widehat{h}_t(y)}{\sum_{y'\in\mathcal{C}(x)}u_t(y'\mid x)\,\widehat{h}_t(y')},\qquad y\in\mathcal{C}(x).
\]
If there exists $\varepsilon\in[0,1)$ such that
\begin{equation}
(1-\varepsilon)\,h_t(y)\ \le\ \widehat{h}_t(y)\ \le\ (1+\varepsilon)\,h_t(y)
\qquad \text{for all } y\in\mathcal{C}(x),
\label{eq:mult_err_h}
\end{equation}
then for all $y\in\mathcal{C}(x)$,
\begin{equation}
\frac{1-\varepsilon}{1+\varepsilon}
\ \le\
\frac{\widehat{P}(y\mid x)}{P(y\mid x)}
\ \le\
\frac{1+\varepsilon}{1-\varepsilon}.
\label{eq:kernel_mult_bound}
\end{equation}
\end{proposition}

\begin{proof}
Fix $(t,x)$ and abbreviate $\mathcal{C}=\mathcal{C}(x)$. For each $y\in\mathcal{C}$, define the unnormalized weights
\[
a_y \coloneqq u_t(y\mid x)\,h_t(y),
\qquad
\widehat{a}_y \coloneqq u_t(y\mid x)\,\widehat{h}_t(y),
\]
and the corresponding normalizing constants
\[
Z \coloneqq \sum_{y'\in\mathcal{C}} a_{y'},
\qquad
\widehat{Z} \coloneqq \sum_{y'\in\mathcal{C}} \widehat{a}_{y'}.
\]
By definition, $P(y\mid x)=a_y/Z$ and $\widehat{P}(y\mid x)=\widehat{a}_y/\widehat{Z}$.

Assumption \eqref{eq:mult_err_h} implies $(1-\varepsilon)a_y \le \widehat{a}_y \le (1+\varepsilon)a_y$ for all $y\in\mathcal{C}$. Summing these inequalities over $y\in\mathcal{C}$ yields
\begin{equation}
(1-\varepsilon)Z \le \widehat{Z} \le (1+\varepsilon)Z.
\label{eq:Z_bounds}
\end{equation}
Therefore, for any $y\in\mathcal{C}$,
\[
\frac{\widehat{P}(y\mid x)}{P(y\mid x)}
=
\frac{\widehat{a}_y}{a_y}\cdot\frac{Z}{\widehat{Z}}.
\]
The first factor satisfies $(1-\varepsilon)\le \widehat{a}_y/a_y\le (1+\varepsilon)$, and \eqref{eq:Z_bounds} implies $\frac{1}{1+\varepsilon}\le Z/\widehat{Z}\le \frac{1}{1-\varepsilon}$. Multiplying the corresponding lower and upper bounds gives \eqref{eq:kernel_mult_bound}.
\end{proof}

The multiplicative condition in \eqref{eq:mult_err_h} follows from finite-rollout estimation under mild moment assumptions. The next proposition provides a sufficient condition based on a variance bound.

\begin{proposition}[Finite-rollout multiplicative accuracy]
\label{prop:add_to_mult}
Fix $(t,x)$ and a finite candidate set $\mathcal{C}(x)$. For each $y\in\mathcal{C}(x)$, let $Z^{(y)}\coloneqq G(X_1)$ under the conditional law $(X_t=y)$, and assume $\mathbb{E}[Z^{(y)}]=h_t(y)>0$ and $\mathrm{Var}(Z^{(y)})=\sigma_t^2(y)<\infty$. Let $\widehat{h}_t(y)=\frac{1}{R}\sum_{r=1}^{R} Z_r^{(y)}$ be the Monte Carlo estimator formed from $R$ independent rollouts from $(t,y)$. Then for any $\varepsilon\in(0,1)$,
\begin{equation}
\mathbb{P}\!\left(\,|\widehat{h}_t(y)-h_t(y)| \ge \varepsilon\,h_t(y)\,\right)
\ \le\
\frac{\sigma_t^2(y)}{R\,\varepsilon^2\,h_t(y)^2}.
\label{eq:chebyshev_mult_refined}
\end{equation}
Moreover, for any $\delta\in(0,1)$,
\begin{equation}
\mathbb{P}\!\left(\,(1-\varepsilon)h_t(y)\le \widehat{h}_t(y)\le (1+\varepsilon)h_t(y)\ \text{for all }y\in\mathcal{C}(x)\,\right)
\ \ge\
1-\delta
\label{eq:uniform_mult_event}
\end{equation}
whenever
\begin{equation}
R\ \ge\ \frac{|\mathcal{C}(x)|}{\delta\,\varepsilon^2}\,\max_{y\in\mathcal{C}(x)}\frac{\sigma_t^2(y)}{h_t(y)^2}.
\label{eq:R_condition_refined}
\end{equation}
\end{proposition}

\begin{proof}
Fix $y\in\mathcal{C}(x)$. By construction, $\widehat{h}_t(y)$ is the average of $R$ independent copies of $Z^{(y)}$, hence $\mathbb{E}[\widehat{h}_t(y)]=h_t(y)$ and $\mathrm{Var}(\widehat{h}_t(y))=\sigma_t^2(y)/R$. Chebyshev's inequality gives
\[
\mathbb{P}\!\left(\,|\widehat{h}_t(y)-h_t(y)| \ge \varepsilon\,h_t(y)\,\right)
\le
\frac{\mathrm{Var}(\widehat{h}_t(y))}{\varepsilon^2 h_t(y)^2}
=
\frac{\sigma_t^2(y)}{R\,\varepsilon^2\,h_t(y)^2},
\]
which proves \eqref{eq:chebyshev_mult_refined}.

For the uniform statement, define the event $A_y=\{|\widehat{h}_t(y)-h_t(y)| \ge \varepsilon h_t(y)\}$. A union bound implies
\[
\mathbb{P}\!\left(\bigcup_{y\in\mathcal{C}(x)} A_y\right)
\le
\sum_{y\in\mathcal{C}(x)} \mathbb{P}(A_y)
\le
\sum_{y\in\mathcal{C}(x)} \frac{\sigma_t^2(y)}{R\,\varepsilon^2\,h_t(y)^2}
\le
\frac{|\mathcal{C}(x)|}{R\,\varepsilon^2}\,\max_{y\in\mathcal{C}(x)}\frac{\sigma_t^2(y)}{h_t(y)^2}.
\]
Under \eqref{eq:R_condition_refined}, the right-hand side is at most $\delta$, which yields \eqref{eq:uniform_mult_event} because $\bigcap_{y\in\mathcal{C}(x)} A_y^{c}$ is exactly the event that \eqref{eq:mult_err_h} holds for all $y\in\mathcal{C}(x)$.
\end{proof}

Proposition~\ref{prop:approx_kernel} is a one-step statement. The following proposition shows how one-step kernel stability propagates to the multi-step editing procedure when the algorithm is viewed as a discrete sequence of decision steps.

\begin{proposition}[Propagation of one-step stability to multi-step trajectories]
\label{prop:multistep_stability}
Consider a discrete-time process $\{X^{(s)}\}_{s=0}^{S}$ on $\mathcal{X}$ with initial state $X^{(0)}=x^{(0)}$. At each step $s\in\{0,\dots,S-1\}$ and state $x$, a candidate set $\mathcal{C}_s(x)\subseteq\mathcal{X}$ is formed and the next state is sampled from a transition kernel supported on $\mathcal{C}_s(x)$. Let $P_s(\cdot\mid x)$ denote the ideal Doob-$h$ transition distribution restricted to $\mathcal{C}_s(x)$ and let $\widehat{P}_s(\cdot\mid x)$ denote the rollout-based transition distribution restricted to the same candidate set. Assume there exists $\varepsilon\in[0,1)$ such that for every step $s$, every state $x$ reachable under the ideal process, and every $y\in\mathcal{C}_s(x)$ with $P_s(y\mid x)>0$,
\begin{equation}
\frac{1-\varepsilon}{1+\varepsilon}
\ \le\
\frac{\widehat{P}_s(y\mid x)}{P_s(y\mid x)}
\ \le\
\frac{1+\varepsilon}{1-\varepsilon}.
\label{eq:stepwise_ratio}
\end{equation}
Then for any trajectory $\tau=(x^{(0)},x^{(1)},\dots,x^{(S)})$ such that $\prod_{s=0}^{S-1}P_s(x^{(s+1)}\mid x^{(s)})>0$,
\begin{equation}
\left(\frac{1-\varepsilon}{1+\varepsilon}\right)^{S}
\ \le\
\frac{\widehat{\mathbb{P}}(\tau)}{\mathbb{P}(\tau)}
\ \le\
\left(\frac{1+\varepsilon}{1-\varepsilon}\right)^{S},
\label{eq:path_ratio}
\end{equation}
where
\[
\mathbb{P}(\tau)\coloneqq \prod_{s=0}^{S-1}P_s(x^{(s+1)}\mid x^{(s)}),\qquad
\widehat{\mathbb{P}}(\tau)\coloneqq \prod_{s=0}^{S-1}\widehat{P}_s(x^{(s+1)}\mid x^{(s)}).
\]
Moreover, for any terminal state $x^{(S)}$ with $\mathbb{P}(X^{(S)}=x^{(S)})>0$,
\begin{equation}
\left(\frac{1-\varepsilon}{1+\varepsilon}\right)^{S}
\ \le\
\frac{\widehat{\mathbb{P}}(X^{(S)}=x^{(S)})}{\mathbb{P}(X^{(S)}=x^{(S)})}
\ \le\
\left(\frac{1+\varepsilon}{1-\varepsilon}\right)^{S}.
\label{eq:terminal_ratio}
\end{equation}
\end{proposition}

\begin{proof}
Fix a trajectory $\tau=(x^{(0)},\dots,x^{(S)})$ with $\mathbb{P}(\tau)>0$. By definition of the path probabilities,
\[
\frac{\widehat{\mathbb{P}}(\tau)}{\mathbb{P}(\tau)}
=
\prod_{s=0}^{S-1}
\frac{\widehat{P}_s(x^{(s+1)}\mid x^{(s)})}{P_s(x^{(s+1)}\mid x^{(s)})}.
\]
Since $\mathbb{P}(\tau)>0$, each factor in the denominator is positive, and the assumption \eqref{eq:stepwise_ratio} applies to each step along the trajectory. Applying the lower and upper bounds in \eqref{eq:stepwise_ratio} termwise and multiplying over $s=0,\dots,S-1$ yields \eqref{eq:path_ratio}.

For the terminal bound, fix a terminal state $x^{(S)}$ with $\mathbb{P}(X^{(S)}=x^{(S)})>0$, and let $\mathcal{T}(x^{(S)})$ denote the set of length-$S$ trajectories ending at $x^{(S)}$ with positive $\mathbb{P}(\tau)$. Summing \eqref{eq:path_ratio} over $\tau\in\mathcal{T}(x^{(S)})$ gives
\[
\left(\frac{1-\varepsilon}{1+\varepsilon}\right)^{S}\sum_{\tau\in\mathcal{T}(x^{(S)})}\mathbb{P}(\tau)
\ \le\
\sum_{\tau\in\mathcal{T}(x^{(S)})}\widehat{\mathbb{P}}(\tau)
\ \le\
\left(\frac{1+\varepsilon}{1-\varepsilon}\right)^{S}\sum_{\tau\in\mathcal{T}(x^{(S)})}\mathbb{P}(\tau).
\]
The left and right sums equal $\mathbb{P}(X^{(S)}=x^{(S)})$ and $\widehat{\mathbb{P}}(X^{(S)}=x^{(S)})$, respectively, by definition of the terminal marginal as the sum of path probabilities over trajectories ending at $x^{(S)}$. Dividing both sides by $\mathbb{P}(X^{(S)}=x^{(S)})$ yields \eqref{eq:terminal_ratio}.
\end{proof}

\paragraph{Properties preserved by practical safeguards.}
pCoMole maintains an incumbent feasible terminal sequence among all rollouts executed during sampling and returns the best feasible terminal encountered. Let $\mathcal{S}$ denote the multiset of terminal sequences produced by all rollouts performed during sampling, including the final terminal rollout. When $\mathcal{S}\cap\mathcal{F}\neq\varnothing$, the algorithm returns
\begin{equation}
x_{\mathrm{out}} \in \arg\max_{x\in \mathcal{S}\cap \mathcal{F}} G(x),
\label{eq:incumbent_return}
\end{equation}
with an arbitrary tie-breaking rule. When $\mathcal{S}\cap\mathcal{F}=\varnothing$, the algorithm applies a predefined fallback rule that does not affect the statement below.

\begin{proposition}[Incumbent optimality among evaluated feasible terminals]
\label{prop:incumbent_feasible}
If $\mathcal{S}\cap\mathcal{F}\neq\varnothing$, then $x_{\mathrm{out}}\in\mathcal{F}$ and $G(x_{\mathrm{out}})\ge G(x)$ for all $x\in\mathcal{S}\cap\mathcal{F}$. Moreover, any pruning rule that only removes candidate evaluations or rollouts, without altering the preference function $G$ or the incumbent update rule, preserves \eqref{eq:incumbent_return} with $\mathcal{S}$ interpreted as the multiset of terminal sequences actually evaluated after pruning.
\end{proposition}

\begin{proof}
Assume $\mathcal{S}\cap\mathcal{F}\neq\varnothing$. By definition \eqref{eq:incumbent_return}, $x_{\mathrm{out}}$ is selected from $\mathcal{S}\cap\mathcal{F}$ and maximizes $G$ over this set, which implies $x_{\mathrm{out}}\in\mathcal{F}$ and $G(x_{\mathrm{out}})\ge G(x)$ for all $x\in\mathcal{S}\cap\mathcal{F}$.

For the pruning statement, let $\mathcal{S}_{\mathrm{full}}$ denote the multiset of terminal sequences that would be produced without pruning and let $\mathcal{S}_{\mathrm{pruned}}$ denote the multiset of terminal sequences actually produced with pruning. A pruning rule that only removes evaluations can only decrease the set of executed rollouts, hence $\mathcal{S}_{\mathrm{pruned}}\subseteq \mathcal{S}_{\mathrm{full}}$ as multisets. The incumbent update rule is applied to the terminal sequences that are evaluated, so under pruning the algorithm returns an element of $\arg\max_{x\in \mathcal{S}_{\mathrm{pruned}}\cap\mathcal{F}} G(x)$ whenever $\mathcal{S}_{\mathrm{pruned}}\cap\mathcal{F}\neq\varnothing$. This is exactly \eqref{eq:incumbent_return} with $\mathcal{S}$ interpreted as the evaluated multiset after pruning, which completes the proof.
\end{proof}

\paragraph{Approximate guarantees for rollout-based pCoMole with pruning.}
The ideal Doob-$h$ construction yields the preference-tilted terminal law in Proposition~\ref{prop:doob_transform_proof}, and Pareto-optimality of ATC maximizers follows from Proposition~\ref{prop:atc_pareto}. The practical algorithm replaces $h_t$ by the rollout estimator $\widehat{h}_t$ and evaluates guidance on a sampled candidate set. Proposition~\ref{prop:add_to_mult} shows that, with sufficiently many rollouts, $\widehat{h}_t$ satisfies the multiplicative accuracy condition \eqref{eq:mult_err_h} on a fixed candidate set with high probability, and Proposition~\ref{prop:approx_kernel} then implies that the resulting one-step selection distribution is a controlled multiplicative perturbation of the ideal Doob-$h$ choice on that set. Proposition~\ref{prop:multistep_stability} propagates this one-step stability across decision steps, implying that the terminal distribution induced by rollout-based guidance remains multiplicatively close to the ideal guided terminal distribution over any fixed editing horizon, and therefore approximately preserves the preference tilting that drives multi-objective improvement under a given weight vector.

Pruning and incumbent caching affect which rollouts are executed but do not alter the preference function $G$ used to evaluate terminals. Proposition~\ref{prop:incumbent_feasible} guarantees that whenever any feasible terminal is generated by the executed rollouts, the returned output is feasible and maximizes $G$ among all feasible terminals that were evaluated. This safeguard preserves feasibility and best-observed preference within the evaluated set, while the approximation results above quantify how rollout-based guidance approaches the ideal Doob-$h$ behavior as rollout and candidate budgets increase. Coverage of distinct Pareto trade-offs is obtained by varying the weight vector across runs, and the same stability reasoning applies to each run independently, whereas pruning may reduce empirical coverage under a fixed computational budget by reducing the set of evaluated terminals.

\section{Ablation Studies}
\label{sec:ablation}

Beyond the GFP objective-guidance ablation in Section~\ref{sec:gfp}, we evaluate the contributions of terminal constraints, scalarization, Cas9-specific sampling choices, and sampling budgets.

\begin{table}[t]
\centering
\tiny
\setlength{\tabcolsep}{4pt}
\renewcommand{\arraystretch}{1.5}
\caption{\textbf{Constraint ablation for peptidomimetic binder design on 3AMA and 5JHF.} For each target, we report the property scores of the known binder (“Baseline”) and the average scores over 100 pCoMole-generated sequences under different constraint settings. Length is the number of tokenized SMILES tokens. PM denotes the hard peptidomimetic constraint and LEN denotes the soft length constraint. $R_{\mathrm{PM}}$ and $R_{\mathrm{LEN}}$ are the fractions (out of 100) of generated sequences that satisfy the PM and LEN constraints, respectively. Please refer to Section~\ref{sec:ablation} for details.}
\label{tab:constraint_ablation}

\begin{tabular}{cc|cccccccccc}
\toprule
\textbf{PDB}&\textbf{Constraints}& \textbf{Length}& \textbf{Non-Toxicity}&  \textbf{Solubility}& \textbf{Permeability}& \textbf{Half-life (h)}& \textbf{Affinity}& \textbf{Motif Score}&  \textbf{Specificity} & \textbf{$R_\text{PM}$}& \textbf{$R_\text{LEN}$}\\
\midrule
\multirow{5}{*}{3AMA}& \makecell{Baseline\\(Known Binder)}& 141& 0.901& 0.854& 0.2528& 1.4508& 0.6731& 0.9171& 0.9941 & \textemdash&\textemdash\\
&PM+LEN& 129& 0.8085& 0.4991& 0.3049& 4.8599& 0.6946& 0.9798& 0.9429 & 1.00&1.00\\
&PM only& 112& 0.7444& 0.5769& 0.3537& 14.4684& 0.6666& 0.8521& 0.9482 & 1.00&0.27\\
&LEN only& 131& 0.8261& 0.486& 0.3081& 4.9182& 0.6947& 0.9782& 0.9443 & 1.00&1.00\\
&None& 21& 0.5104& 0.8422& 0.6186& 68.4154& 0.4696& 0.4247& 0.9613 & 0.10&0.02\\
\midrule
\multirow{5}{*}{5JHF}&\makecell{Baseline\\(Known Binder)}& 95& 0.8029& 0.8744& 0.2431& 1.5936& 0.5457& 0.8484& 0.9926 & \textemdash&\textemdash\\
&PM+LEN& 70& 0.6989& 0.6496& 0.3296& 9.2383& 0.6327& 0.9327& 0.9576 & 1.00&1.00\\
&PM only& 54& 0.664& 0.717& 0.4032& 19.8692& 0.5861& 0.7769& 0.97 & 1.00&0.42\\
&LEN only& 72& 0.7131& 0.6319& 0.3342& 8.6049& 0.6319& 0.9337& 0.957 & 0.99&1.00\\
&None& 7& 0.3851& 0.7816& 0.7176& 175.339& 0.4342& 0.5186& 0.9997 & 0.00&0.00\\
\bottomrule
\end{tabular}
\end{table}
\paragraph{Constraint Ablation.}
We ablate the terminal constraints used in peptidomimetic binder design to quantify their impact (Table~\ref{tab:constraint_ablation}). The peptidomimetic hard constraint (PM) enforces chemical validity, while the soft length constraint (LEN) encourages shrinkage without reducing below half of the original size. Enforcing both constraints yields short, feasible designs with strong predicted properties and 100\% satisfaction rates. Removing LEN leads to over-shrinking and degraded trade-offs, often sacrificing binding-related properties for extended half-life. Removing PM rarely produces invalid outputs due to the strong inductive bias of the underlying Edit Flow, but explicit enforcement provides an additional safety layer against off-manifold samples. Removing both constraints results in degenerate behavior, including extreme length collapse and loss of affinity and motif consistency, underscoring the necessity of terminal constraints for controlled shrinkage.

\begin{table}[t]
\centering
\footnotesize
\setlength{\tabcolsep}{3pt}
\renewcommand{\arraystretch}{1.3}
\caption{\textbf{Pareto coverage comparison on PDB 4O56 under different scalarizations and weight vectors.} We evaluate three scalarization choices (ATC, Tchebycheff, and linear weighted sum) for multi-objective peptidomimetic binder design optimizing solubility, permeability, and binding affinity. For each method, we run pCoMole with three sets of objective weight vectors and report the average properties of 100 generated sequences. The rightmost column reports each method’s empirical Pareto coverage against a reference front constructed by pooling solutions across all weight settings and methods (higher is better).}
\label{tab:coverage}

\begin{tabular}{cc|cccc|c}
\toprule
\textbf{Scalarization}& \textbf{\makecell{Objective Weights\\Solubility\hspace{0.75em}Permeability\hspace{0.75em}Affinity}}& \textbf{Length}&  \textbf{Solubility}& \textbf{Permeability}& \textbf{Affinity}& \textbf{\makecell{Coverage vs.\\Reference Front}}\\
\midrule
\multirow{3}{*}{ATC}& 1\hspace{5em}1\hspace{5em}8& 38& 0.6856& 0.3909& 0.7488& \multirow{3}{*}{0.7297}\\
&1\hspace{5em}8\hspace{5em}1& 34& 0.7131& 0.4724& 0.6832& \\
&8\hspace{5em}1\hspace{5em}1& 37& 0.8002& 0.3851& 0.6381& \\
\midrule
\multirow{3}{*}{Tchebycheff}&1\hspace{5em}1\hspace{5em}8& 34& 0.6973& 0.4829& 0.6748& \multirow{3}{*}{0.6486}\\
&1\hspace{5em}8\hspace{5em}1& 36& 0.7248& 0.4002& 0.701& \\
&8\hspace{5em}1\hspace{5em}1& 35& 0.7197& 0.4625& 0.6509& \\
\midrule
\multirow{3}{*}{Weighted Sum}&1\hspace{5em}1\hspace{5em}8& 38& 0.7064& 0.4052& 0.7711& \multirow{3}{*}{0.5946}\\
&1\hspace{5em}8\hspace{5em}1& 34& 0.7137& 0.4866& 0.6538& \\
&8\hspace{5em}1\hspace{5em}1& 37& 0.8051& 0.3702& 0.6348& \\
\bottomrule
\end{tabular}
\end{table}
\paragraph  {Pareto Coverage Under Different Scalarizations.}
We assess the effect of scalarization and weight choice on Pareto coverage for PDB~4O56, optimizing solubility, permeability, and binding affinity (Table~\ref{tab:coverage}). Across representative weight vectors, varying the weights produce the expected trade-offs, confirming controllable navigation of the Pareto front. The augmented Tchebycheff (ATC) utility achieves the highest empirical coverage against a pooled reference front, while standard Tchebycheff yields moderate coverage and linear weighted sums perform worst, consistent with their tendency to under-explore non-convex regions. These results support ATC as the default scalarization in pCoMole.

\begin{table}[t]
\setlength{\tabcolsep}{3pt}
\renewcommand{\arraystretch}{1.3}
\caption{\textbf{Deletion-rate scaling ablation for pCoMole on St1Cas9 edits.}
In all cases, the corresponding model was used to generate 50 sequences. ``Base EF'' denotes sampling from the Edit Flows directly without any guidance. ``s'' denotes the deletion rate scaling term.}
\label{tab:del_rate_ablation}
\begin{center}
\begin{tabular}{c|c|c|c|c}
\toprule
\textbf{Model} &
\textbf{Length Diff.} &
\textbf{Cas9 Likeli.} &
\textbf{PAM Distr. CE} &
\textbf{PAM Match \%} \\
\midrule
Base EF (\(s=1\)) &
\(1121 \rightarrow 1119\;\) &
\(1.00 \rightarrow 1.00\) &
\(0.91 \rightarrow 0.89\) &
\(36\) \\
Base EF (\(s=1000\)) &
\(1121 \rightarrow 1004\;\) &
\(1.00 \rightarrow 0.70\) &
\(0.91 \rightarrow 0.78\) &
\(0\) \\
pCoMole (\(s=1\)) &
\(1121 \rightarrow 1118\;\) &
\(1.00 \rightarrow 1.00\) &
\(0.91 \rightarrow 0.94\) &
\(100\) \\
pCoMole (\(s=1000\)) &
\(1121 \rightarrow 1016\;\) &
\(1.00 \rightarrow 0.95\) &
\(0.91 \rightarrow 0.89\) &
\(100\) \\
\bottomrule
\end{tabular}
\end{center}
\end{table}
\paragraph{Deletion-rate scaling ablation.}
For Cas9 shrinkage, we amplify the deletion rate during sampling to encourage substantial length reduction from the input wild-type sequence. To test whether this sampling bias alone explains the results, we compare pCoMole with the unguided base Edit Flow under matched deletion-rate scaling and insertion-disabled sampling (Table~\ref{tab:del_rate_ablation}). The unguided Edit Flow can shorten the sequence, but it fails to preserve functional feasibility: at $1000\times$ deletion scaling, PAM match collapses to 0\% and Cas9-likeness drops substantially. In contrast, pCoMole maintains high Cas9-likeness and exact PAM matching while achieving comparable shrinkage. Thus, deletion-rate scaling controls the strength of the shrinkage pressure, while Doob-$h$ guidance controls whether the resulting sequence remains feasible.

\begin{table}[t]
\caption{\textbf{Ablation of allowing pCoMole edit steps to make insertions.} 
In both cases, pCoMole was used to generate 50 shrunk variants of St1Cas9.}
\label{tab:cas9_ablation}
\begin{center}
        \begin{tabular}{c|cccc}
        \toprule
        \textbf{pCoMole Config.} &
        \textbf{Length Diff.} &
        \textbf{Cas9 Likeli.} &
        \textbf{PAM Distr. CE} &
        \textbf{PAM Match \%}\\
        \midrule
        pCoMole w/ no ins. &
        $1121 \rightarrow 1021\; $ &
        $1.00 \rightarrow 0.98$ &
        $0.91 \rightarrow 0.89$ &
        100 \\
        pCoMole w/ ins.&
        $1121 \rightarrow 1049\; $ &
        $1.00 \rightarrow 0.97$ &
        $0.91 \rightarrow 0.90$ &
        100 \\
        \bottomrule
        \end{tabular}
\end{center}
\end{table}

\paragraph{Disabling Insertions Ablation.}
We ablate insertions during editing on the Cas9 task (Table~\ref{tab:cas9_ablation}). Because the underlying Edit Flow is trained as an unconditional generator, insertions can counteract shrinkage; disabling them consistently yields greater length reduction with negligible impact on Cas9-likelihood or PAM preservation, making insertion-free editing the preferred configuration for Cas9 shrinkage in pCoMole.

\begin{table}[t]
\centering
\scriptsize
\setlength{\tabcolsep}{4pt}
\renewcommand{\arraystretch}{1.5}
\caption{\textbf{Sampling-step ablation for peptidomimetic binder design on 2KXQ and 8EYA.} For each target, the row with \#-steps$=0$ reports the known binder (baseline), and the remaining rows report averages over 100 pCoMole-generated sequences under different numbers of sampling steps. Length is the number of tokenized SMILES tokens. We also report the average wall-clock time (seconds) to generate one sequence for each step setting.}
\label{tab:step_ablation}

\begin{tabular}{cc|ccccccccc}
\toprule
\textbf{PDB}&\textbf{\#-steps}& \textbf{Length}& \textbf{Non-Toxicity}&  \textbf{Solubility}& \textbf{Permeability}& \textbf{Half-life (h)}& \textbf{Affinity}& \textbf{Motif Score}&  \textbf{Specificity} & \textbf{Time (s)}\\
\midrule
\multirow{5}{*}{2KXQ}& 0& 130& 0.5343& 0.8607& 0.2882& 2.2298& 0.6373& 0.9528& 0.8556& \textemdash\\
&10& 116& 0.8017& 0.787& 0.3000& 3.5262& 0.7565& 0.9719& 0.7786& 69.73\\
&20& 111& 0.7683& 0.7279& 0.3113& 4.4267& 0.8055& 0.9690& 0.7882& 150.94\\
&30& 108& 0.7626& 0.7036& 0.3180& 5.3392& 0.8106& 0.9705& 0.7837& 257.37\\
&50& 99& 0.7541& 0.6664& 0.3330& 6.7532& 0.8037& 0.9699& 0.7928& 530.83\\
\midrule
\multirow{5}{*}{8EYA}&0& 114& 0.7220& 0.8249& 0.2410& 4.2260& 0.6285& 0.9392& 0.9967& \textemdash\\
&10& 99& 0.7137& 0.7873& 0.2848& 4.6720& 0.7010& 0.9796& 0.9840& 67.33\\
&20& 96& 0.6737& 0.7670& 0.3054& 5.8047& 0.7306& 0.9796& 0.9800& 139.69\\
&30& 93& 0.6676& 0.7459& 0.3152& 6.9622& 0.7432& 0.9791& 0.9765& 215.27\\
&50& 84& 0.6721& 0.7146& 0.3504& 11.0632& 0.7511& 0.9788& 0.9703& 444.16\\
\bottomrule
\vspace{1em}
\end{tabular}
\end{table}
\begin{table}[t]
\centering
\scriptsize
\setlength{\tabcolsep}{3pt}
\renewcommand{\arraystretch}{1.5}
\caption{\textbf{Rollout ablation for peptidomimetic binder design on 2QSC and 7ZPY.} For each target, the row with \#-rollouts$=0$ reports the known binder (baseline), and the remaining rows report averages over 100 pCoMole-generated sequences under different numbers of rollouts. Length is the number of tokenized SMILES tokens. We also report the average wall-clock time (seconds) to generate one sequence for each rollout setting.}
\label{tab:rollout_ablation}

\begin{tabular}{cc|ccccccccc}
\toprule
\textbf{PDB}&\textbf{\#-rollouts}& \textbf{Length}& \textbf{Non-Toxicity}&  \textbf{Solubility}& \textbf{Permeability}& \textbf{Half-life (h)}& \textbf{Affinity}& \textbf{Motif Score}&  \textbf{Specificity} & \textbf{Time (s)}\\
\midrule
\multirow{5}{*}{2QSC}& 0& 111& 0.8861& 0.7156& 0.2874& 1.7210& 0.6775& 0.9323& 0.9621& \textemdash\\
&5& 87& 0.7084& 0.6840& 0.2940& 5.9379& 0.7745& 0.9371& 0.9470 & 78.79\\
&10& 86& 0.7172& 0.6858& 0.2974& 6.2827& 0.7809& 0.9287& 0.9477& 124.57\\
&20& 87& 0.7447& 0.6938& 0.3063& 6.7057& 0.7826& 0.9335& 0.9464& 216.92\\
&30& 84& 0.7175& 0.6967& 0.3042& 7.1078& 0.7756& 0.9294& 0.9462& 306.83\\
\midrule
\multirow{5}{*}{7ZPY}&0& 84& 0.4621& 0.7744& 0.2787& 2.5357& 0.5218& 0.3083& 0.9954& \textemdash\\
&5& 63& 0.6625& 0.7089& 0.4353& 15.4066& 0.7062& 0.6002& 0.9914& 131.07\\
&10& 60& 0.6886& 0.7144& 0.4490& 17.4544& 0.6988& 0.6288& 0.9936& 205.87\\
&20& 59& 0.6694& 0.7122& 0.4435& 18.1834& 0.6940& 0.6113& 0.9920& 344.24\\
&30& 61& 0.6729& 0.7230& 0.4505& 19.0042& 0.7020& 0.616& 0.9922& 271.69\\
\bottomrule
\end{tabular}
\vspace{1em}
\end{table}
\begin{table}[t]
\centering
\tiny
\setlength{\tabcolsep}{4pt}
\renewcommand{\arraystretch}{1.5}
\caption{\textbf{Number of candidates ablation for peptidomimetic binder design on 2O9V and 8PFT.} For each target, the row with \#-candidates$=0$ reports the known binder (baseline), and the remaining rows report averages over 100 pCoMole-generated sequences under different numbers of candidates. Length is the number of tokenized SMILES tokens. We also report the average wall-clock time (seconds) to generate one sequence for each candidate-count setting.}
\label{tab:candidate_ablation}

\begin{tabular}{cc|ccccccccc}
\toprule
\textbf{PDB}&\textbf{\#-candidates}& \textbf{Length}& \textbf{Non-Toxicity}&  \textbf{Solubility}& \textbf{Permeability}& \textbf{Half-life (h)}& \textbf{Affinity}& \textbf{Motif Score}&  \textbf{Specificity} & \textbf{Time (s)}\\
\midrule
\multirow{5}{*}{2O9V}& 0& 65& 0.0826& 0.8305& 0.2977& 2.7590& 0.4128& 0.9671& 0.8806& \textemdash\\
&10& 47& 0.6998& 0.7226& 0.4626& 16.5440& 0.6143& 0.9692& 0.8327& 41.42\\
&20& 45& 0.7330& 0.7345& 0.4633& 17.5760& 0.6241& 0.9646& 0.8309& 76.76\\
&30& 46& 0.7511& 0.7323& 0.4574& 19.6885& 0.6278& 0.9658& 0.8339& 58.32\\
&50& 47& 0.7532& 0.7379& 0.4541& 20.1591& 0.6242& 0.9683& 0.8312& 163.55\\
\midrule
\multirow{5}{*}{8PFT}&0& 143& 0.6156& 0.7538& 0.185& 6.1210& 0.5750& 0.6579& 0.9502& \textemdash\\
&10& 124& 0.5925& 0.7336& 0.2478& 6.2114& 0.7314& 0.7429& 0.9300& 113.14\\
&20& 120& 0.5730& 0.7208& 0.2586& 6.1821& 0.7381& 0.7548& 0.9255& 128.26\\
&30& 115& 0.5833& 0.7183& 0.2619& 6.6555& 0.7481& 0.7498& 0.9298& 166.76\\
&50& 116& 0.5679& 0.7094& 0.2642& 6.4031& 0.7524& 0.7631& 0.9250& 480.98\\
\bottomrule
\end{tabular}
\end{table}
\paragraph{Hyperparameter Ablation.}
We evaluate sensitivity to three sampling hyperparameters: number of sampling steps, rollouts per candidate, and candidates per step (Tables~\ref{tab:step_ablation}, \ref{tab:rollout_ablation}, \ref{tab:candidate_ablation}). Increasing any budget improves shrinkage quality and predicted properties, particularly affinity and motif preservation, but with clear diminishing returns and increased runtime. More steps enable progressive refinement, additional rollouts reduce myopic decisions at higher cost, and more candidates improve exploration but saturate beyond moderate values. These trends motivate intermediate settings for all experiments (Appendix \ref{sec:sampling}). pCoMole is relatively robust to these choices because Doob-$h$ guidance concentrates probability mass on high-utility trajectories, so increased budgets often refine already strong preferences.

\section{Edit Flow details}

\subsection{Dataset Curation}
\paragraph{UniRef Dataset.} We trained our general protein models on two datasets of UniRef sequences with differing lengths. To curate the dataset for our UniRef$_{S}$ Edit Flow, we sampled 30k reviewed UniRef sequences with length $\leq350$. We distributed these sampled sequences into training, test, and validation sets following an 80-10-10 split. The training data for UniRef$_{L}$ was curated by splitting UniRef50 data into buckets based on the distribution of lengths in our Cas9 dataset. We sampled sequences from each bucket such that the resulting set was as large as possible while still maintaining the Cas9 dataset's length distribution. We used MMSEQS to reduce redundancy of the set, and distributed sequences into training, test, and validation sets following an 80-10-10 split. The final dataset was $\sim$30k sequences. 

\paragraph{GFP Dataset.} We curate our GFP dataset from two sources: FPBase (an online dataset of fluorescent proteins) and UniRef. For FPBase, we extract basic proteins with an emission maximum within a desired range. From UniRef, we extract sequences in the PF01353 PFam family. We use MMSEQS to remove redundancy from the set and distribute sequences into training, test, and validation sets following a 80-10-10 split. 

\paragraph{Cas9 Dataset.}
We curate a Cas9-family protein dataset from the CRISPR-Cas Atlas \citep{ruffolo2024design}, a large, publicly available resource that aggregates CRISPR-Cas systems and associated effector protein sequences across diverse microbial genomes. Starting from the full Atlas release, we filter entries to retain only sequences annotated as Cas9 (discarding other Cas effectors and non-Cas proteins). This yields an initial pool of roughly $10^5$ candidate Cas9 sequences.

To reduce redundancy and obtain a compact but representative training set, we apply MMseqs2 redundancy reduction. Specifically, we cluster sequences by similarity and retain a representative sequence per cluster, reducing the dataset to around $2\times 10^4$ diverse Cas9 sequences while preserving broad family coverage.

Finally, to prevent information leakage across splits, we assign train/validation/test sets at the cluster level rather than at the individual-sequence level. We perform a final clustering pass and then allocate entire clusters into the training, validation, and test splits using an 80-10-10 partition. This ensures that highly similar sequences do not appear in multiple splits, making generalization metrics more meaningful.

\paragraph{Peptidomimetics Dataset.} For training the peptidomimetics Edit Flow, we curated a peptidomimetic sequence dataset by combining two sources: a commercial peptidomimetic library from ChemDiv and a manually expanded set generated from peptide SMILES using an RDKit-based transformation pipeline. For each input peptide, our script constructs candidate peptidomimetics through chemically motivated edits, including peptoid conversion, triazole isostere replacement, and for noncyclic peptides, N-terminal and C-terminal residue deletion, as well as a combined peptoid plus isostere edit. We further filter the generated peptidomimetics using synthetic accessibility and SCScore thresholds to remove implausible chemistries, and we remove exact duplicates to keep a clean dataset. Finally, we transform the remaining SMILES strings into SELFIES strings and use these SELFIES sequences to train the peptidomimetics Edit Flow.

\subsection{Model Architecture}
The backbone of our Edit Flow is a Transformer encoder that parameterizes a continuous time Markov chain over discrete edit operations. Input token sequences are first mapped to continuous vectors using a learned token embedding layer, and the normalized flow time is embedded with a time embedding module. Except for the peptidomimetic Edit Flow that applies a learned token embedding layer, the other Edit Flows apply a pre-trained ESM-2-650M model to compute the sequence embeddings \citep{lin2023evolutionary}. These two embeddings are combined and processed by a stack of Transformer blocks equipped with multi-head self attention, enabling the model to capture long range dependencies along the sequence. We use rotary positional embeddings to provide length generalization while avoiding fixed absolute position tables. Residual connections and layer normalization stabilize optimization, and a final layer normalization produces the hidden representation used by multiple output heads.

To model edit dynamics, the network predicts both the total edit rate and the distribution over edit types, using a reparameterization that separates the scalar intensity from the categorical choice among insertion, deletion, and substitution. Specifically, one head outputs a nonnegative total rate through a Softplus transformation, while a second head produces logits over the three edit types. Conditioned on the chosen type, token level heads output vocabulary logits for insertion and substitution, which define the token proposals at each position. This factorized design yields a flexible, stable parameterization of per position edit intensities and token proposals, and supports variable length sequence generation through insertion and deletion operations.

\subsection{Training Strategy}
During Edit Flow training, we use the least optimal alignment objective to supervise the edit process with an alignment between the source and target token sequences. For each terminal sequence $x_1$ in the dataset, we construct an initial sequence $x_0$ by sampling tokens from the vocabulary, with the initial length drawn uniformly from $[0,\,2|x_1|]$ for the Peptidomimetic Edit Flow and $[(1-a)|x_1|,\,(1+a)|x_1|]$ for the protein Edit Flows. For UniRef$_S$, $a=0.5$ and for UniRef$_L$ and Cas9 $a=0.25$. This randomized initialization exposes the model to a broad range of starting lengths and content, encouraging robust insertion and deletion behavior and avoiding over-reliance on near identity alignments. Given the sampled $x_0$ and target $x_1$, we compute the least optimal alignment and train the model to match the induced edit dynamics, including both edit type intensities and token proposal distributions, along the continuous time editing trajectory.

\paragraph{UniRef and Cas9 Edit Flow models.}
Training was conducted on one NVIDIA RTX A6000 GPU with 48\,GB of VRAM. The UniRef$_S$, UniRef$_L$, and Cas9 Edit Flow models were trained for around 30, 10, and 100 epochs respectively using the AdamW optimizer. We use a learning rate of $1\times 10^{-5}$ and weight decay of 0.03. A learning-rate scheduler with cosine decay is applied. We use the same warmup and minimum learning-rate setting as in our other Edit Flow trainings, with both the initial and minimum learning rates set to $1\times 10^{-5}/10 = 1\times 10^{-6}$. The network architecture uses a model dimension of 768, 8 Transformer layers, and 12 attention heads.

\paragraph{GFP Edit Flow model.}
Training was conducted on one NVIDIA RTX A6000 GPU with 48\,GB of VRAM. We initialize the GFP Edit Flow from the pre-trained UniRef$_S$ checkpoint and fine-tune on the curated GFP dataset. The model is trained using the AdamW optimizer with weight decay 0.03. We use a learning rate of $1\times 10^{-6}$ and apply a learning-rate scheduler with 10 warm-up epochs and cosine decay. As above, we set both the initial and minimum learning rates to $1\times 10^{-6}/10 = 1\times 10^{-7}$. The network architecture uses a model dimension of 768, 8 Transformer layers, and 12 attention heads.

\paragraph{Peptidomimetic Edit Flow.} Training was conducted on one H100 NVIDIA NVL GPU with 94 GB of VRAM. The model was trained for 100 epochs using the AdamW optimizer and a learning rate of 3e-4 with weight decay of 0.03. A learning rate scheduler with 10 warm-up epochs and cosine decay was used, with initial and minimum learning rates both 3e-5. The network architecture included a model dimension of 768, 8 transformer layers, and 12 attention heads. 

\paragraph{Diversity Evaluation.}
\label{sec:diversity}
To quantify sample diversity for each Edit Flow, we draw 1000 sequences from the model and report an alignment-free $k$-mer Jaccard diversity computed on the generated sequences. For each sequence $s$, we form the set of all contiguous $k$-grams (we use $k=3$) and estimate the average Jaccard similarity over a large number of randomly sampled sequence pairs. The reported diversity is defined as one minus this average similarity, $\mathrm{Div}_{k}=1-\mathbb{E}[\mathrm{Jaccard}(K_k(s_i),K_k(s_j))]$, which is robust to variable-length sequences and captures lexical diversity without requiring sequence alignment.

\section{Score Model Details}

\subsection{Cas9 Score Models}
\paragraph{Cas9 Classifier.}
We train a binary Cas9 classifier to predict whether a candidate protein sequence is a valid Cas9, and use this score both as an objective signal and as a quality-control component in Cas9 shrinkage experiments. The classifier is a lightweight MLP on top of a frozen ESM-2-650M sequence encoder: given an amino-acid sequence, we compute per-token embeddings with ESM-2-650M, apply mean pooling over non-padding tokens to obtain a fixed-length sequence representation, and then pass this representation through a small multi-layer perceptron to output a single logit (converted to a probability with a sigmoid).

We construct the positive set by sampling from the same curated Cas9 sequence corpus used to train the Cas9 Edit Flow. The negative set is designed to be challenging and length-matched: it includes (i) UniRef protein sequences sampled to have a similar length distribution to Cas9s, (ii) non-Cas9 Cas effector sequences from the CRISPR-Cas Atlas (e.g., other Cas families), and (iii) broken or corrupted Cas9-like sequences (e.g., truncated or otherwise implausible Cas9 sequences) to explicitly teach the classifier to reject malformed Cas9 candidates. Training was conducted on one NVIDIA RTX A6000 GPU with 48\,GB of VRAM.

% \paragraph{PAM prediction.} For PAM prediction, we use the full-length protein2PAM model, available at `Profluent-Bio/protein2pam-cas9\textunderscore
% full`. Given that Protein2PAM outputs a distribution over predicted PAMs, enforcing exact PAM matching requires setting a threshold to determine what level of model confidence constitutes a nucleotide prediction vs. 'N'. Throughout all experiments, we set this threshold to 0.7, ensuring that only PAM predictions for which Protein2PAM is relatively confident are considered the 'correct PAM' for pCoMole to match.

\paragraph{PAM prediction.} For PAM prediction, we use the full-length protein2PAM model, available at `Profluent-Bio/protein2pam-cas9\textunderscore
full`. Given that Protein2PAM outputs a distribution over predicted PAMs, enforcing exact PAM matching requires setting a threshold to determine what level of model confidence constitutes a nucleotide prediction vs. ``N''. Throughout all experiments, we set this threshold to 0.7, ensuring that only PAM predictions for which Protein2PAM is relatively confident are considered the ``correct PAM'' for pCoMole to match.

\subsection{GFP Score Models}
\paragraph{GFP Classifier.} We train a binary GFP classifier to enforce that generated sequences remain in the GFP family during GFP shrinkage experiments. The classifier is trained on the same curated GFP dataset used for GFP Edit Flows training. Architecturally, we encode each input sequence with a frozen ESM-2-650M encoder and apply mean pooling over non-special, non-padding tokens, followed by a small MLP head with two hidden layers and dropout to produce a single logit. We use stratified train/validation/test splits (0.70/0.15/0.15) and optimize a class-balanced binary cross-entropy loss with a positive-class weight computed from the training split.

Training was conducted on one H100 NVIDIA NVL GPU system with 94 GB of VRAM. The model was trained for 20 epochs using the AdamW optimizer and a learning rate of 2e-4 with weight decay of 0.01. A learning rate scheduler with 2 warm-up epochs and cosine decay was used, with initial and minimum learning rates both 2e-5. The network architecture included a model dimension of 512 and a 0.3 dropout rate. Training uses AdamW with a warmup-cosine learning-rate schedule and gradient clipping of 1.0. We select the best checkpoint by validation AUROC. On the held-out test set, the GFP classifier achieves loss $0.1834$, AUROC $0.9784$, AUPRC $0.9480$, F1 $0.9457$, and accuracy $0.9506$.

\paragraph{FPredX.} We use FPredX as our fluorescence predictor for GFP variants \citep{tam2022fpredx}. Given an amino acid sequence, FPredX outputs predicted brightness, excitation wavelength, and emission wavelength. We use these predictions to define the GFP objectives (maximize brightness and align excitation to the 488\,nm laser line) and the emission-range constraint requiring the emission peak to fall within the green band. We use the pre-trained FPredX model without additional fine-tuning.

\subsection{Peptidomimetics Score Models}
\paragraph{Chemotype-adaptive property predictors.}
For the peptidomimetic design tasks, we compute developability-related objectives using a chemotype-adaptive ensemble of PeptiVerse \citep{zhang2026peptiverse} and Admetica predictors \citep{admetica}. Peptidomimetics occupy an intermediate regime between peptide-like and small-molecule-like chemical space, so applying a single property predictor uniformly can be inappropriate. We therefore compute, for each candidate SMILES $x$, a continuous peptide-likeness score $\alpha(x)\in[0,1]$ and define the property score for objective $j$ as
\begin{equation}
    \hat f_j(x)
    =
    \alpha(x)\hat f^{\mathrm{Pep}}_j(x)
    +
    \left(1-\alpha(x)\right)\hat f^{\mathrm{ADM}}_j(x),
    \label{eq:peptide_Admetica_mixture}
\end{equation}
where $\hat f^{\mathrm{Pep}}_j$ denotes the PeptiVerse prediction and $\hat f^{\mathrm{ADM}}_j$ denotes the corresponding Admetica prediction. This mixture weights peptide-specialized predictions more heavily for peptide-like candidates and small-molecule ADMET predictions more heavily for less peptide-like candidates. The predictor ensemble is fixed throughout all experiments and is used consistently during pCoMole-guided sampling, baseline evaluation, and ablations.

\paragraph{Peptide-likeness score.}
The mixture coefficient $\alpha(x)$ is computed from RDKit-derived substructure features that capture peptide-like amide connectivity. Given a candidate SMILES $x$, we count heavy atoms, heavy bonds, amide bonds, carbonyl carbons, and adjacent amide pairs. Let $f_1(x)$ be the fraction of carbonyl carbons adjacent to nitrogen, $f_2(x)$ be the amide-bond density normalized by a target density $\delta_{\mathrm{amide}}=0.12$, and $f_3(x)=1-\exp(-n_{\mathrm{adj}}(x))$ be a saturating function of the number of adjacent amide pairs. We additionally penalize non-peptidic carbonyl environments using $p(x)=1-f_{\mathrm{pure}}(x)$, where $f_{\mathrm{pure}}(x)$ is the fraction of carbonyls classified as pure amides. The final peptide-likeness score is
\begin{equation}
\begin{aligned}
    \alpha(x)
    =
    \mathrm{clip}_{[0,1]}
    \Big[
    0.55 f_1(x)
    + 0.25 f_2(x)
    + 0.20 f_3(x)
    - 0.25 p(x)
    \Big].
\end{aligned}
\label{eq:peptide_likeness}
\end{equation}
This score approaches one for molecules with dense peptide-like amide connectivity and approaches zero for molecules whose carbonyl chemistry is less peptide-like. We include ring amides by default so that cyclic and macrocyclic peptidomimetics are not systematically penalized.

\paragraph{SMILES BindEvaluator. } To estimate pocket-level binding consistency for peptidomimetic editing, we train a SMILES BindEvaluator that predicts residue-wise binding sites given a binder SMILES sequence and a target protein amino-acid sequence. The model follows the same architecture as BindEvaluator \citep{chen2025moppit}, but uses a pre-trained ChemBERTa-zinc250k-v2-40k model for tokenization and embedding pipeline for SMILES \citep{chithrananda2020chemberta}, while the target sequence is encoded by pre-trained ESM-2-650M \citep{lin2023evolutionary}. The two representations are fused to produce per-residue binding-site probabilities. Training data are curated from PLINDER \citep{durairaj2024plinder}: we remove duplicates, filter target proteins with lengths in $[30,800]$, merge binding-site annotations by majority vote for repeated (binder, target) pairs, drop entries with invalid SMILES or metals, and remove targets containing non-canonical amino acids. We further prevent leakage by removing validation examples that appear in the training set, and we augment the corpus with additional binder-target pairs from PepNN and BioLip2 after removing any overlap with PLINDER \citep{abdin2022pepnn, zhang2024biolip2}. After filtering, the final dataset contains 48155 training pairs and 6451 validation pairs.

Training was conducted on a 2xH100 NVIDIA NVL GPU system with 94 GB of VRAM. The model was trained for 50 epochs using the AdamW optimizer and a learning rate of 1e-4 with weight decay of 0.03. A learning rate scheduler with 5 warm-up epochs and cosine decay was used, with initial and minimum learning rates both 1e-5. The network architecture included a model dimension of 128, hidden dimension of 128, 8 transformer layers and a 0.3 dropout. During training, we tune the classification threshold on the validation set to maximize MCC, and we use the resulting threshold $\tau=0.918$ for downstream scoring. SMILES BindEvaluator achieves a high prediction performance across all metrics.

\paragraph{Motif and specificity scores.}
Given a target of length $L$ (excluding special tokens) and SMILES BindEvaluator probabilities $p_i \in [0,1]$ for each residue $i$ being in the binding site, we define two pocket-consistency scores. Let $\mathcal{M}$ denote the set of residues corresponding to the known binding motif/pocket region for the reference binder. The \emph{motif score} is the mean predicted binding-site probability over motif residues,
\[
\mathrm{Motif}(x) \;=\; \frac{1}{|\mathcal{M}|}\sum_{i\in \mathcal{M}} p_i \, .
\]
To quantify off-motif binding, we define \emph{specificity} using the threshold $\tau$. We count the number of residues outside the motif whose predicted binding probability exceeds $\tau$, and normalize by sequence length,
\[
\mathrm{Specificity}(x) \;=\; 1 \;-\; \frac{1}{L}\sum_{i\notin \mathcal{M}} \mathbf{1}\!\left[p_i \ge \tau\right].
\]
Higher motif scores indicate stronger predicted binding in the intended pocket region, while higher specificity indicates fewer high-confidence binding-site predictions outside the known motif.

\paragraph{Peptidomimetic Constraint.}
We classify each SMILES as \textit{natural peptide}, \textit{peptidomimetic}, or \textit{non-peptide/uncertain} using RDKit substructure matching and topology-based checks. We first detect amide (C=O-N) links and backbone ``tiles'' with SMARTS patterns for an $\alpha$-peptide segment (O=C-N-C(sp$^3$)-C=O) and for $\beta/\gamma$ spacing around an amide. A molecule is deemed \emph{$\alpha$-peptide-only} if it contains at least two amide bonds, has at least one secondary amide (N-H), exhibits no $\beta/\gamma$ tiles, and its amides are largely covered by the $\alpha$-tile pattern (allowing up to two terminal amides). We then flag \emph{peptidomimetic indicators} capturing deviations from a natural $\alpha$-peptide backbone, including backbone isosteres (depsipeptide ester links, thioamides, peptoid motifs), non-proline tertiary amides, and common linkage replacements such as ureas/thioureas, guanidines, sulfonamides, and imide/lactam/diacyl-amide motifs. If the $\alpha$-peptide-only criterion holds and no indicators are present we label the SMILES \texttt{natural\_peptide}; if any indicator is detected we label it \texttt{peptidomimetic\_not\_natural}; otherwise we mark it \texttt{non\_peptide\_or\_uncertain} and log an audit record of all SMARTS counts and triggered indicators.

\paragraph{Length Constraint.} To determine whether a generated candidate is \emph{shorter} or \emph{smaller} than the starting molecule, we compute a backbone-length proxy and standard size descriptors with RDKit. We define \emph{backbone centers} as amide carbonyl carbons (SMARTS \texttt{[CX3](=[OX1])[NX3]}) and optionally sulfonamide sulfur centers (SMARTS \texttt{[SX4](=[OX1])(=[OX1])[NX3]}), then build an undirected graph over centers by connecting two centers if there exists a single-bond path consistent with a peptide/peptidomimetic linkage: $\textit{center}_i$-N-C(sp$^3$)-$\textit{center}_j$ (length 3) or $\textit{center}_i$-N-(O/S/N)-C(sp$^3$)-$\textit{center}_j$ (length 4). We extract the largest connected component and use its number of nodes/edges (\texttt{backbone\_centers\_main}, \texttt{backbone\_links\_main}) as a proxy for backbone length, while also computing MW and heavy-atom count (HAC). A candidate is labeled \texttt{shorter\_backbone} if its main-component nodes or edges decrease relative to the original, and \texttt{smaller\_bulk} if both MW and HAC decrease (optionally by margins); we accept \texttt{overall\_shorter\_or\_smaller} if either condition holds. During generation, we additionally enforce a heuristic token-length range $[L_0/2,\,L_0)$ to prevent extreme collapse while encouraging substantial compression.

\begin{table*}[t]
\centering
\footnotesize
\caption{\textbf{Held-out CICERO evaluation of pCoMole-generated Cas9 designs.} We evaluated 25 designs per ortholog and condition. The guidance-disabled control removes the PAM objective and hard PAM constraint while keeping the remaining pCoMole settings fixed. Exact-match values are percentages with Wilson 95\% confidence intervals; PWM similarity compares the full predicted PAM distributions.}
\label{tab:cicero}
\resizebox{\textwidth}{!}{%
\begin{tabular}{llcccc}
\toprule
\textbf{Ortholog} & \textbf{PAM guidance} & \textbf{Protein2PAM exact match} & \textbf{Protein2PAM PWM sim.} & \textbf{CICERO exact match} & \textbf{CICERO PWM sim.} \\
\midrule
St3Cas9 & Enabled  & 100 [87, 100] & 0.896 & 88 [70, 96] & 0.837 \\
St3Cas9 & Disabled & 16 [6, 35]    & 0.634 & 4 [1, 20]   & 0.758 \\
GeoCas9 & Enabled  & 100 [87, 100] & 0.792 & 72 [52, 86] & 0.766 \\
GeoCas9 & Disabled & 88 [70, 96]   & 0.392 & 20 [9, 39]  & 0.619 \\
\bottomrule
\end{tabular}%
}
\end{table*}

\section{Independent and Experimental Validation Details}
\label{sec:independent_validation}

\paragraph{Held-out PAM evaluation with CICERO.}
We used CICERO \citep{fang2026cicero} only for post hoc evaluation of PAM preservation in St3Cas9 and GeoCas9. For each ortholog and condition, we evaluated 25 generated designs. The guidance-disabled control removed the Protein2PAM PAM objective and hard PAM constraint while keeping the remaining pCoMole settings fixed. Exact-match confidence intervals in Table~\ref{tab:cicero} are Wilson 95\% intervals.

\paragraph{Bacterial expression and fluorescence imaging.}
Competent \textit{Escherichia coli} BL21 cells were transformed individually with sequence-verified plasmids encoding pCoMole\_eGFP\_1 or pCoMole\_eGFP\_2. Both variants were expressed from a pET21a-derived bacterial expression backbone containing the constitutive J23119 promoter. Plasmid DNA was isolated from selected clones, verified by Sanger sequencing, and retransformed into BL21 cells to confirm fluorescence retention. The original 239-residue eGFP construct and the 228-residue negative control were included for comparison. Green fluorescence was imaged using 488-nm excitation, and phase-contrast and fluorescence images were collected for each construct. Merged images were generated to visualize fluorescence relative to the bacterial cells, and fluorescence retention was assessed qualitatively from the collected images. Complete amino-acid sequences of both variants and the control constructs are provided below in Appendix~\ref{app:wetlab_sequences}.

\subsection{Wet-lab eGFP Sequences}
\label{app:wetlab_sequences}
The 239-residue in-house eGFP starting construct and the three edited constructs tested in the wet-lab experiment are listed below.

\paragraph{Original eGFP (239 aa).}
\begin{quote}\ttfamily\small
\seqsplit{MVSKGEELFTGVVPILVELDGDVNGHKFSVSGEGEGDATYGKLTLKFICTTGKLPVPWPTLVTTLTYGVQCFSRYPDHMKQHDFFKSAMPEGYVQERTIFFKDDGNYKTRAEVKFEGDTLVNRIELKGIDFKEDGNILGHKLEYNYNSHNVYIMADKQKNGIKVNFKIRHNIEDGSVQLADHYQQNTPIGDGPVLLPDNHYLSTQSALSKDPNEKRDHMVLLEFVTAAGITLGMDELYK}
\end{quote}

\paragraph{pCoMole\_eGFP\_1 (229 aa; 10 deletions, 1 substitution).}
\begin{quote}\ttfamily\small
\seqsplit{MSKGEELFTGVVPILVELDGDVNGHKFSVSGEGEGDATYGKLTLKFICTTGKLPVPWPTLVTTLTYGVQCFSRYPDHMKQHDFFKSAMPEGYVQERTISFKDDGNYKTRAEVKFEGDTLVNRIELKGIDFKEDGNILGHKLEYNYNSHNVYIMADKQKNGIKVNFKIRHNIEDGSVQLADHYQQNTPIGDGPVLLPDNHYLSTQSALSKDPNEKRDHMVLLEFVTAAGK}
\end{quote}

\paragraph{pCoMole\_eGFP\_2 (229 aa; 10 deletions, 2 substitutions).}
\begin{quote}\ttfamily\small
\seqsplit{MSKGEELFTGVVPILVELDGDVNGHKFSVSGEGEGDATYGKLTLKFICTTGKLPVPWPTLVTTLTYGVQCFSRYPDHMKQHDFFKSAMPEGYVQERTISFKDDGNYKTRAEVKFEGDTLVNRIELKGIDFKEDGNILGHKLEYNYNSHNVYIMADKQKNGIKVNFKIRHNIEDGSVQLADHYQQNTPIGDGPVLLPDNHYLSTQSVLSKDPNEKRDHMVLLEFVAGITK}
\end{quote}

\paragraph{Negative control (228 aa; 11 deletions, 2 substitutions).}
\begin{quote}\ttfamily\small
\seqsplit{MSKGEELFTGVVPILVELDGDVNGHKFSVSGEGEGDATYGKLTLKFICTTGKLPVPWPTLVTTLTYGVQCFSRYPDHMKQHDFFKSAMPEGYVQERTISFKDDGNYKTRAEVKFEGDTLVNRIELKGIDFKEDGNILGHKLEYNYNSHNVYIMADKQKNGIKVNFKIRHNIEDGSVQLADHYQQNTPIGDGPVLLPDYLSTSSPEKRDHLLEFVTAAGITLGMDELYK}
\end{quote}
% TODO (internal, do not expose in compiled text): add BL21 subtype/source, vector and cloning details, culture/induction conditions, imaging instrument/settings, replicate counts, and control-selection details once the experimental record is available.

% TODO (internal): integrate the rebuttal 95% CIs into the main GFP table only after reconciling the canonical benchmark outputs, since the current manuscript table and rebuttal summary use different GFP means.

\section{pCoMole Sampling Details.}
\label{sec:sampling}
\paragraph{GFP Shrinkage.} For GFP shrinkage, unless otherwise stated, we use \texttt{num\_steps=10},\texttt{num\_candidates=50}, and \texttt{num\_rollouts=10}, with \texttt{objective\_weights=[3, 1, 1]} for length reduction, brightness, and excitation alignment, respectively. To further bias sampling toward shorter sequences in this setting, we disable insertions by setting $\lambda_{\mathrm{ins}}=0$ and amplify deletions by multiplying the deletion rate by $100$, which empirically encourages aggressive shrinkage while maintaining the GFP validity and emission constraints.

\paragraph{Cas9 Shrinkage.}
Across all Cas9 shrinkage experiments, we amplify the deletion rate by multiplying the base Edit Flow deletion intensity, which empirically encourages aggressive length reduction while still allowing the Doob-$h$ guidance to preserve Cas9-likeness and PAM constraints. Unless otherwise stated, we disable insertions by setting the insertion rate to zero (i.e., $\lambda_{\mathrm{ins}}=0$), since insertions are systematically misaligned with the shrinkage objective and can dilute search effort. Our PAM matching CE loss comprises of two components: (i) the CE loss between the PAM logits and the target PAM nucleotide for all  non-``N'' positions; and (ii) the CE loss between the PAM logits and a uniform distribution. We use the PAM predicted by protein2PAM for the input sequence as the PAM to be maintained during pCoMole shrinkage.

\paragraph{Main Cas9 results.}
For the main Cas9 shrinkage experiments, we use \texttt{num\_steps=20}, \texttt{num\_candidates=10}, and \texttt{num\_rollouts=5}. We set equal objective weights \texttt{objective\_weights=[1,1,1]} (Cas9-likeness, PAM preservation loss, and length reduction, respectively), scale the deletion rate by $1000$ and keep insertions disabled via $\lambda_{\mathrm{ins}}=0$. 

\paragraph{Cas9 compute-matched time-to-usable design benchmark.}
For the compute-matched Cas9 benchmark, we use \texttt{num\_steps=5}, \texttt{num\_candidates=10}, and \texttt{num\_rollouts=5}. We set equal objective weights \texttt{objective\_weights=[1,1,1]} (Cas9-likeness, PAM preservation loss, and length reduction, respectively), scale the deletion rate by $2500$ and keep insertions disabled via $\lambda_{\mathrm{ins}}=0$. Because SCISOR exhibits high sample variance on Cas9, we average time-to-usable design across all sequences sampled within a 3600-second budget for all three methods. We used CjCas9 for this experiment. 

\paragraph{Compute-matched benchmark.}
For this benchmark, we use \texttt{num\_steps=10}, \texttt{num\_candidates=10}, and \texttt{num\_rollouts=5}. We set equal objective weights \texttt{objective\_weights=[1,1,1]} (Cas9-likeness, PAM preservation loss, and length reduction, respectively), scale the deletion rate by $3000$ and keep insertions disabled via $\lambda_{\mathrm{ins}}=0$. 

\paragraph{Benchmark setting with fixed terminal length.}
For the RayGun/SCISOR benchmark where we enforce an exact terminal length match, we use a shorter horizon but substantially increase lookahead to reduce variance under the hard length target: \texttt{num\_steps=10}, \texttt{num\_candidates=20}, and \texttt{num\_rollouts=20}. We use \texttt{objective\_weights=[1,3,1]} to place additional emphasis on PAM preservation during the more tightly constrained generation, and we again disable insertions ($\lambda_{\mathrm{ins}}=0$).

\paragraph{Deletion rate scaling ablation.}
For this benchmark, we use \texttt{num\_steps=10}, \texttt{num\_candidates=10}, and \texttt{num\_rollouts=5}. We set equal objective weights \texttt{objective\_weights=[1,1,1]} (Cas9-likeness, PAM preservation loss, and length reduction, respectively), and keep insertions disabled via $\lambda_{\mathrm{ins}}=0$. 

\paragraph{Insertion ablation.}
For the Cas9 insertion ablation, we match the main-results hyperparameters \texttt{num\_steps=20}, \texttt{num\_candidates=10}, \texttt{num\_rollouts=5}, and \texttt{objective\_weights=[1,1,1]}. The only change is whether insertions are permitted: in the \emph{no-insertions} condition we set $\lambda_{\mathrm{ins}}=0$, while in the \emph{insertions-enabled} condition we retain the base insertion rate from the pre-trained Edit Flow. In both conditions we keep the $1000\times$ deletion-rate amplification.

\paragraph{Peptidomimetic Binder Design.}
Unless otherwise noted (e.g., in the hyperparameter ablations), we run pCoMole for peptidomimetic design with \texttt{num\_steps=30}, \texttt{num\_candidates=30}, and \texttt{num\_rollouts=10}. We use the fixed \texttt{objective\_weights=[4, 2, 4, 0.2, 4, 2, 1]} for the seven peptidomimetic objectives respectively: non-toxicity, solubility, permeability, half-life, affinity, motif score, and specificity. For developability endpoints with both peptide-specific and small-molecule predictor outputs, the reported objective value is the peptide-likeness-adaptive mixture in Eq.~\eqref{eq:peptide_Admetica_mixture}. Because the peptidomimetics Edit Flow backbone operates on SELFIES strings while all property and binding-score predictors take SMILES as input, we convert each sampled SELFIES sequence to SMILES on the fly for objective and constraint evaluation at each sampling step. Reported average lengths (tokenized SMILES length) are rounded to integers for readability.

\section{Benchmark Models}

\paragraph{RayGun.}
For RayGun baselines, we use the publicly available pre-trained RayGun implementation with 8.8M parameters trained on UniRef50. All RayGun generations reported in this work are performed using this fixed pre-trained checkpoint. RayGun experiments were run on the H100 NVIDIA NVL GPU system with 94\,GB of VRAM.

\paragraph{SCISOR.}
For SCISOR baselines, we use the publicly available pre-trained \texttt{UniRef50\_L} variant. We run SCISOR generation using this fixed pre-trained model without further fine-tuning. SCISOR experiments were run on one NVIDIA RTX A6000 GPU with 48\,GB of VRAM.

\section{Impact Statement}
pCoMole directly addresses a central bottleneck in therapeutic translation: making powerful molecules smaller, more controllable, and more deliverable. By unifying multi-objective optimization, hard feasibility, and variable-length sequence editing within a single generative framework, pCoMole makes it possible to systematically shrink complex biomolecules such as fluorescent proteins, genome editors, and peptide therapeutics while preserving task-critical properties. In particular, the ability to reduce the size of systems like CRISPR-Cas9 without altering target specificity has the potential to significantly expand the space of deliverable and clinically viable genome-editing tools. At the same time, improved deliverability necessarily raises safety considerations, including off-target activity, immunogenicity, and context-dependent behavior \emph{in vivo}. We view pCoMole as a tool that must be coupled with careful experimental validation, safety screening, and responsible deployment, especially for genome-editing and therapeutic applications.

\newpage
\section{Additional Benchmark Results}
\begin{table}[t]
\setlength{\tabcolsep}{3pt}
\renewcommand{\arraystretch}{1.3}
\caption{\textbf{Comparison of pCoMole to SCISOR and RayGun at shrinking CjCas9.} All three methods were used to shrink WT CjCas9 to a sequence length of 934.}
\label{tab:cas9_benchmark}
\begin{center}
        \begin{tabular}{c|ccc}
        \toprule
        \textbf{Method} &
        \textbf{Length Diff.} &
        \textbf{Avg. Cas9 Likelihood} &
        \textbf{PAM Match \%}\\
        \midrule
        pCoMole w/ Cas9 Edit Flows   &
        984 $\longrightarrow$ 934 &
        0.95 &
        100\\
        pCoMole w/ UniRef Edit Flows   &
        984 $\longrightarrow$ 934 &
        0.98 &
        100\\
        SCISOR &
        984 $\longrightarrow$ 934 &
        0.82&
        35\\
        RayGun &
        984 $\longrightarrow$ 934 &
        0.0 &
        0 \\
        \bottomrule
        \end{tabular}
\end{center}
\end{table}

\begin{table}[t]
\centering
\scriptsize
\setlength{\tabcolsep}{4pt}
\renewcommand{\arraystretch}{1.5} % Default value: 1
\caption{\textbf{pCoMole designs short peptidomimetic binders for 15 PDB targets with known peptide binders, achieving substantial length reduction and improved predicted properties.} For each target, we report the average scores over 100 designed peptidomimetics and the property score change ($\Delta$) relative to the pre-existing binder. Length is the number of SMILES tokens after tokenization. All designed sequences satisfy the specified constraints, and positive $\Delta$ indicates improvement.}
\label{tab:mimetics_pdb}

\begin{tabular}{c|c|ccccccc}
\toprule
\textbf{PDB} & \textbf{\makecell{Length\\Change}}& \textbf{$\Delta$Non-Toxicity}& \textbf{$\Delta$Solubility}& \textbf{$\Delta$Permeability}& \textbf{$\Delta$Half-life (h)}& \textbf{$\Delta$Affinity}& \textbf{$\Delta$Motif Score}& \textbf{$\Delta$Specificity} \\
\midrule
1AYC & 59 $\rightarrow$ 38& -0.1288 & +0.0163 & +0.1436 & +12.8558 & +0.1001 & +0.0268 & -0.0386 \\
1B8Q & 51 $\rightarrow$ 33& -0.2324 & +0.0456 & +0.1512 & +33.3370 & -0.0643 & +0.0024 & -0.0389 \\
1DDV & 41 $\rightarrow$ 28& +0.3973 & +0.0293 & +0.1945 & +31.1299 & +0.2025 & -0.0214 & +0.0098 \\
1E6I & 43 $\rightarrow$ 29& -0.1412 & +0.1025 & +0.1748 & +23.5607 & +0.1515 & +0.0547 & -0.0537 \\
2LTV & 85 $\rightarrow$ 64& +0.3128 & -0.0909 & +0.0584 & +9.6710  & +0.1178 & +0.1034 & -0.0978 \\
2Q8Y & 72 $\rightarrow$ 49& -0.0639 & -0.0591 & +0.0971 & +18.0023 & +0.2046 & +0.0122 & -0.0437 \\
3IDJ & 50 $\rightarrow$ 38& +0.1706 & -0.0250 & +0.2047 & +23.6463 & +0.0487 & -0.0953 & -0.0141 \\
4GNE & 54 $\rightarrow$ 39& -0.2138 & +0.2024 & +0.0570 & +17.7531 & +0.2295 & +0.0346 & -0.0570 \\
5AZ8 & 80 $\rightarrow$ 63& -0.2233 & -0.0964 & +0.0061 & +6.1258  & +0.0782 & +0.0934 & -0.0162 \\
5KRI & 105$\rightarrow$ 87& -0.1807 & -0.0622 & +0.0079 & +2.6242  & +0.0738 & +0.0252 & -0.0448 \\
5M02 & 61 $\rightarrow$ 45& -0.2005 & -0.0956 & +0.0699 & +20.2843 & +0.2425 & +0.1078 & -0.0021 \\
6MLC & 45 $\rightarrow$ 30& -0.1460 & +0.0722 & +0.1717 & +20.0400 & +0.1731 & +0.0259 & -0.0053 \\
7JVS & 79 $\rightarrow$ 58& +0.0537 & -0.1033 & +0.0425 & +7.7421  & +0.1060 & +0.0240 & +0.0021 \\
7LUL & 70 $\rightarrow$ 47& +0.0936 & -0.0911 & +0.1835 & +23.2782 & +0.0769 & +0.0373 & -0.0320 \\
8CN1 & 30 $\rightarrow$ 23& +0.0239 & +0.0481 & +0.1976 & +31.8092 & +0.0371 & +0.0021 & -0.0416 \\
\bottomrule
\end{tabular}
\end{table}

% \input{Tables/losses}
% \input{Tables/generation_quality}

% \input{Tables/gfp_raygun}
% \input{Tables/losses}
% \input{Tables/gfp_objective_ablation}

% \input{Tables/cas9_del_rate_scaling_ablation}

% \input{Tables/cas9_main_table}
% 
% % \input{Tables/mimetic_drugs}
% \input{Tables/mimetic_constraint_ablation}
% \input{Tables/mimetic_coverage}
% \input{Tables/runtimes}
% \input{Tables/runtimes}
% \input{Tables/gfp_budget}

\clearpage
\begin{figure}[t]

    \centering
        \includegraphics[width=0.8\textwidth]{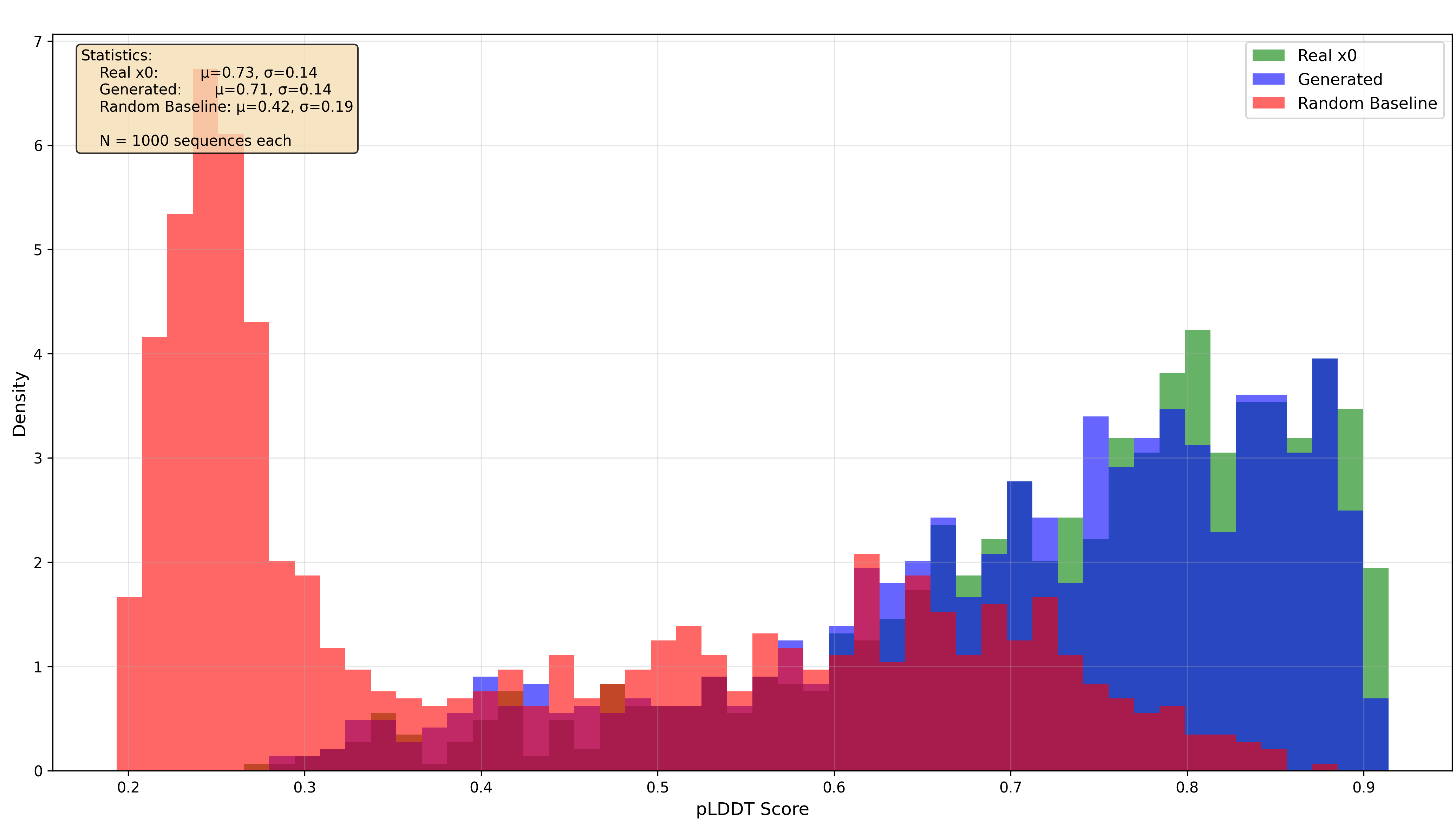}
    \caption{\textbf{Distribution of average pLDDT scores across a set of 1000 sequences generated by the UniRef$_{S}$ Edit Flow, compared to the 1000 true Uniref sequences used as inputs.} 1000 UniRef sequences with length $\leq350$ were sampled and used as $x_0$ for these generations.}
    \label{fig:plddts}
    % \label{fig:dock1}
    \vskip -0.1in
\end{figure}

\begin{figure}[t]
  \vskip 0.2in
  \begin{center}
    \centerline{\includegraphics[width=\columnwidth]{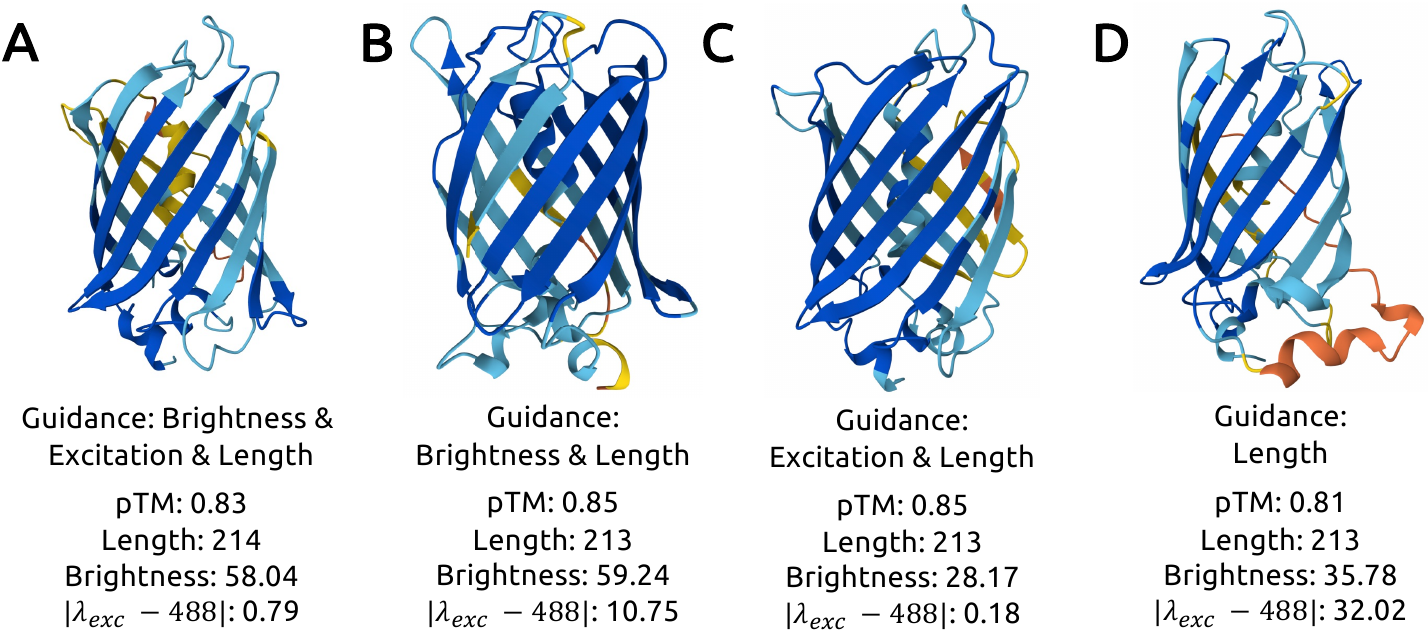}}
    \caption{\textbf{Objective guidance controls optical trade-offs while preserving GFP structure under shrinkage.} Representative pCoMole-designed GFP variants generated under four objective-guidance settings: \textbf{(A)} brightness, excitation alignment, and length reduction \textbf{(B)} brightness and length reduction \textbf{(C)} excitation alignment and length reduction \textbf{(D)} length reduction only. We present their secondary structures together with AlphaFold3-predicted pTM scores, sequence length, predicted brightness, and excitation alignment error.}
    \label{fig:gfp}
  \end{center}
\end{figure}

\begin{figure}[t]
    \centering
      \includegraphics[width=\textwidth]{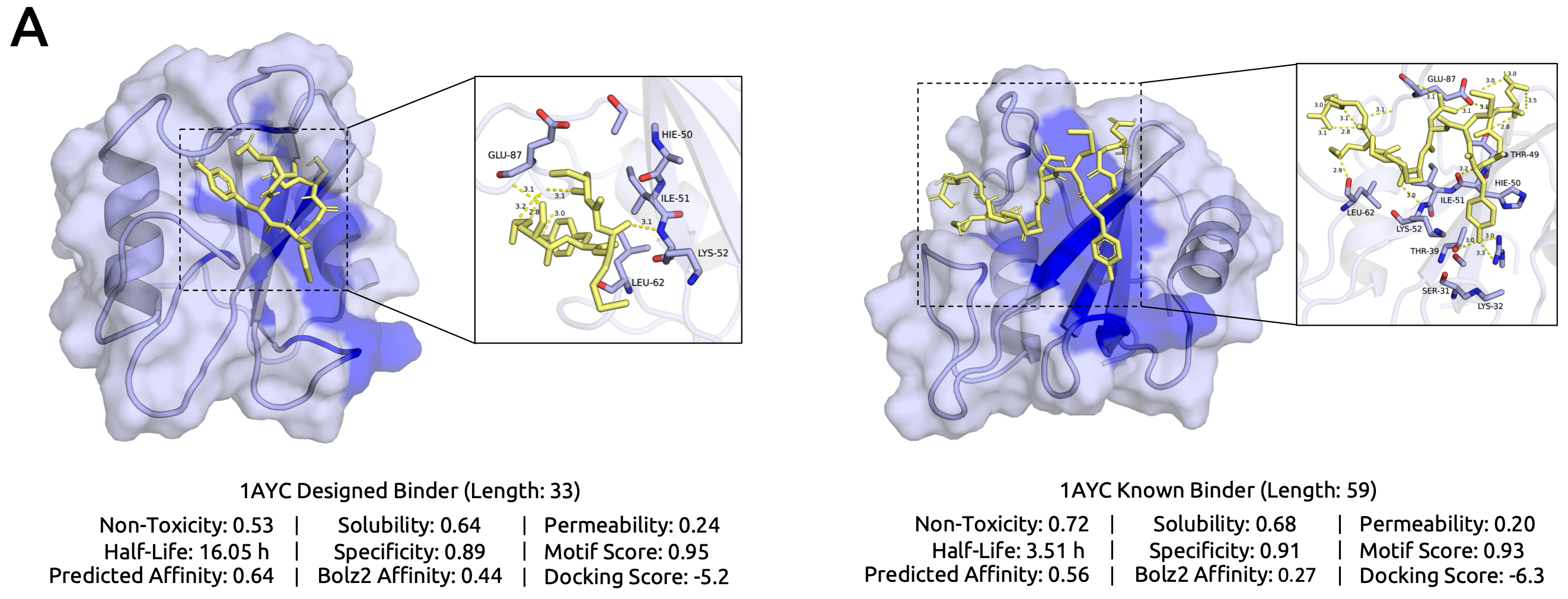}
      \vspace{2em}\\
      \includegraphics[width=\textwidth]{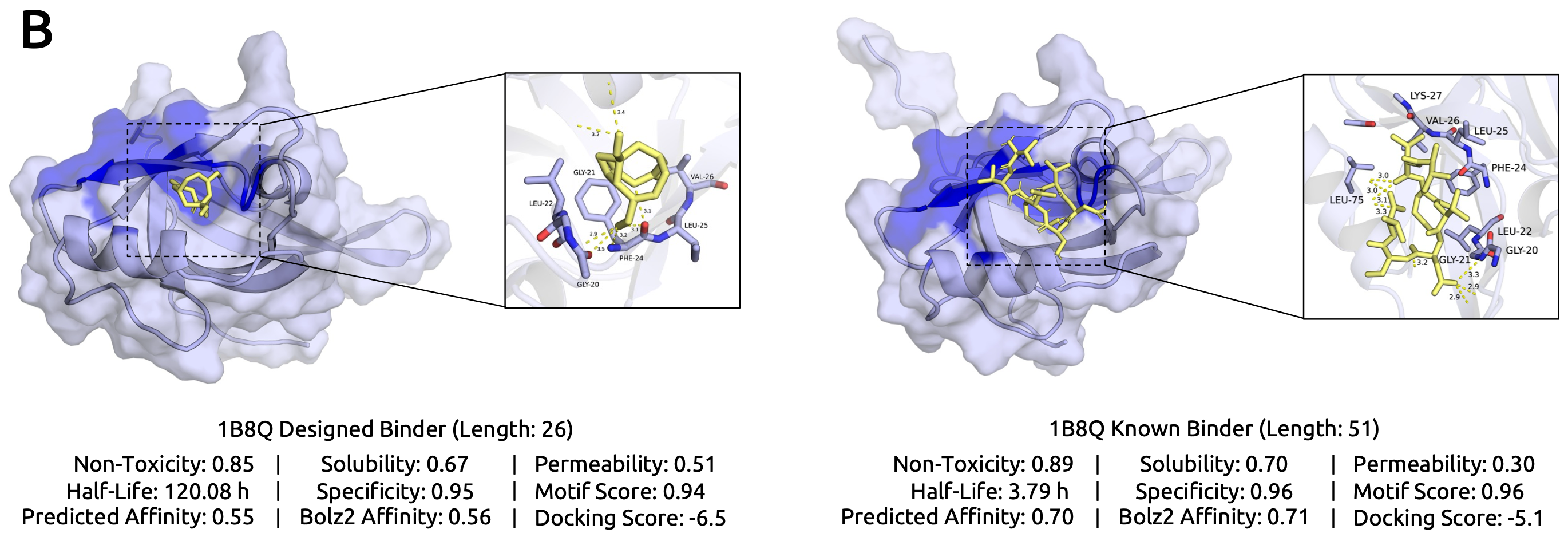}
         \vspace{2em}\\
      \includegraphics[width=\textwidth]{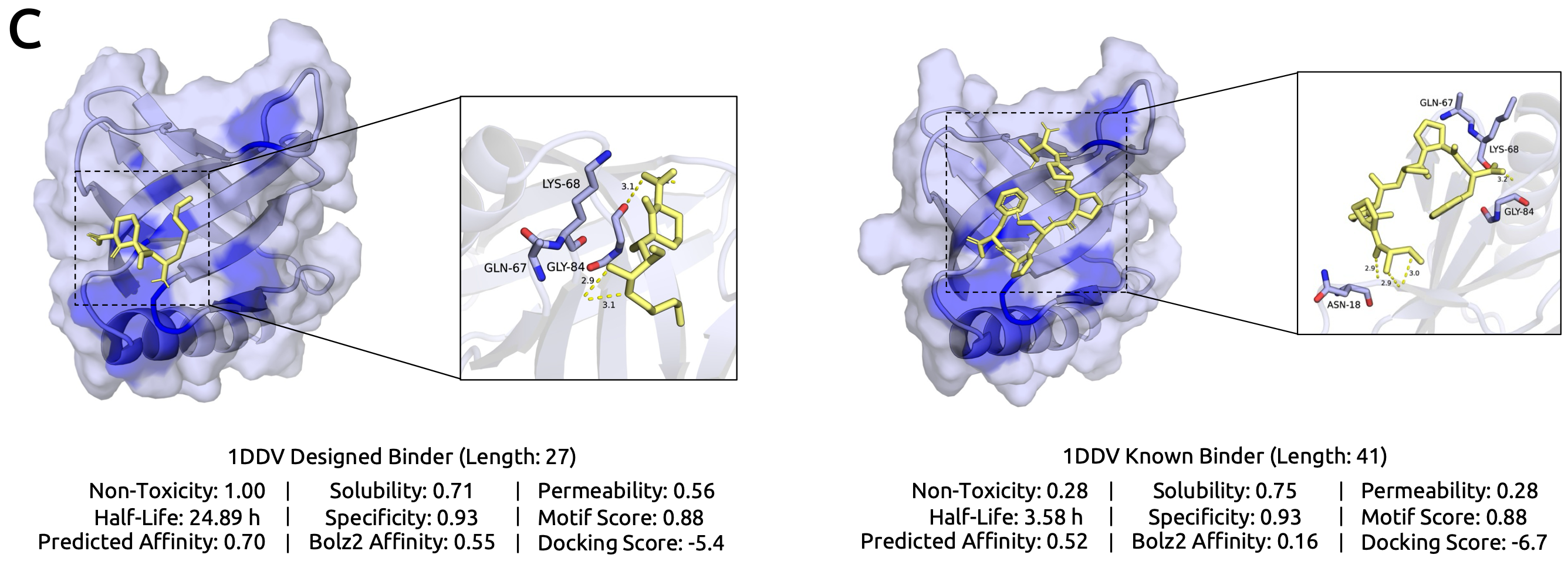}
    \caption{\textbf{Complex structures of known peptide binders and pCoMole-designed peptidomimetic binders with three target PDB proteins (A) 1AYC (B) 1B8Q (C) 1DDV.} Binders are shown in yellow, target proteins in light blue, and the dark-blue surface highlights the target motifs used by pCoMole to enforce motif-specific binding during design. Insets zoom into the binding interface to illustrate key contacts. Predicted property scores for each binder are reported.}
    \label{fig:dock1}
    \vskip -0.1in
\end{figure}

\begin{figure}[t]
    \centering
      \includegraphics[width=\textwidth]{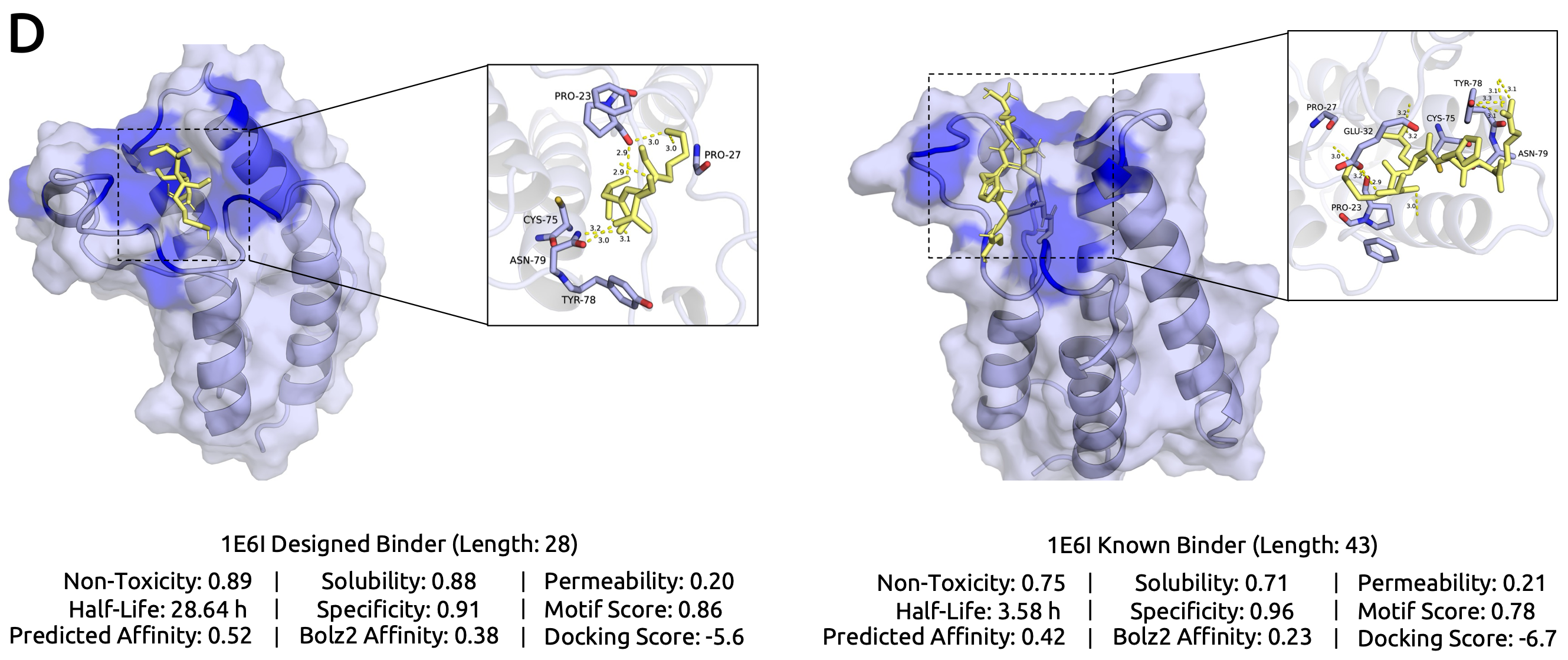}
      \vspace{2em}
      \includegraphics[width=\textwidth]{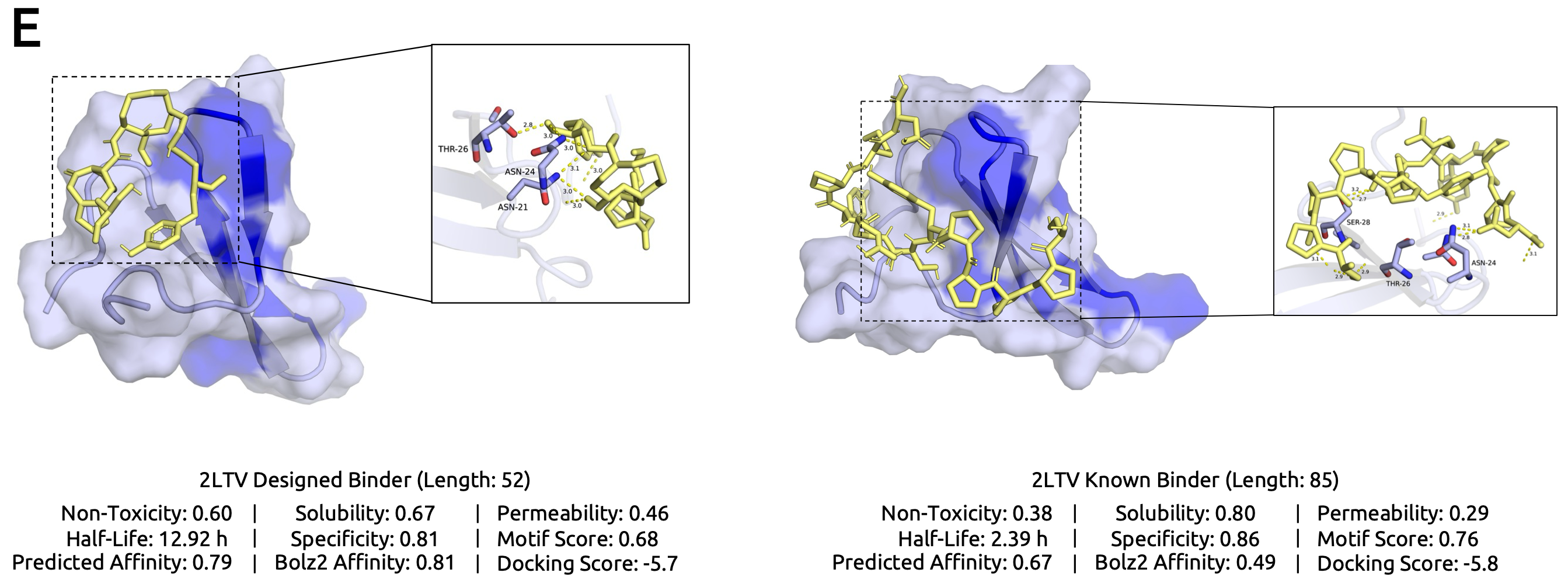}
         \vspace{2em}
      \includegraphics[width=\textwidth]{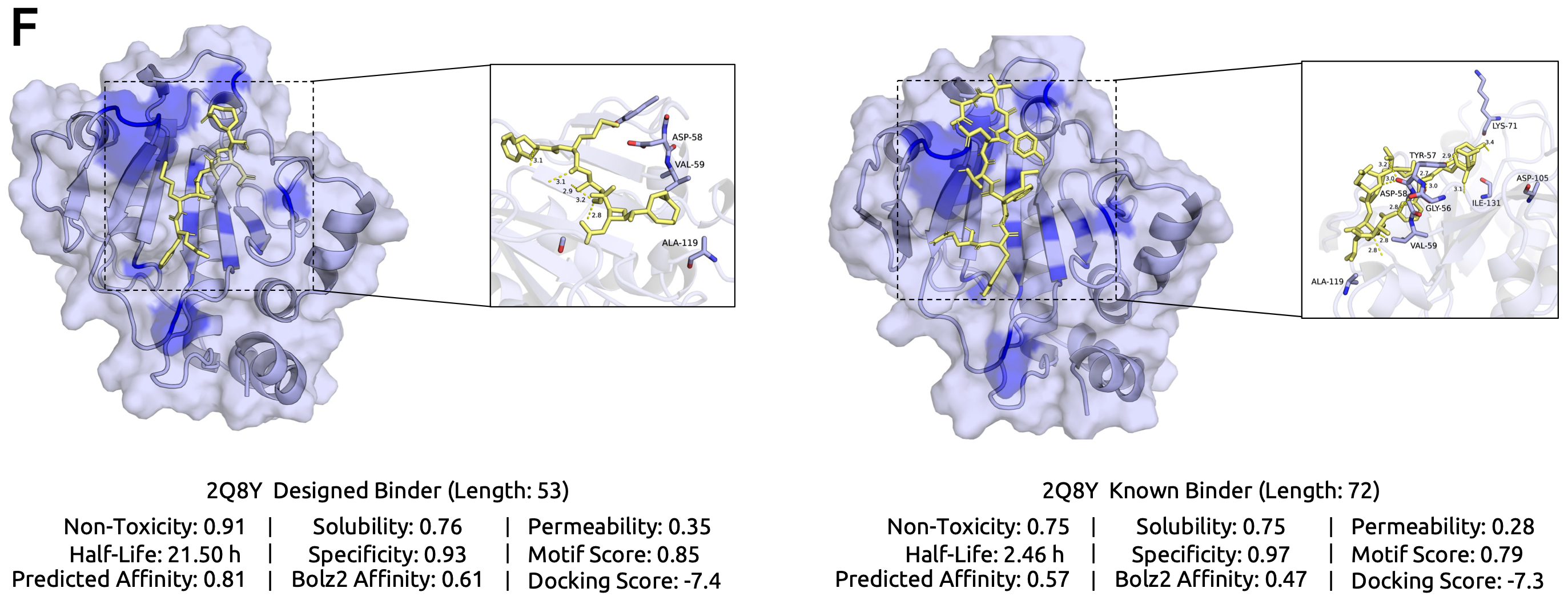}
    \caption{\textbf{Complex structures of known peptide binders and pCoMole-designed peptidomimetic binders with three target PDB proteins (D) 1E6I (E) 2LTV (F) 2Q8Y.} Binders are shown in yellow, target proteins in light blue, and the dark-blue surface highlights the target motifs used by pCoMole to enforce motif-specific binding during design. Insets zoom into the binding interface to illustrate key contacts. Predicted property scores for each binder are reported.}
    \label{fig:dock2}
    \vskip -0.1in
\end{figure}

\begin{figure}[t]
    \centering
      \includegraphics[width=\textwidth]{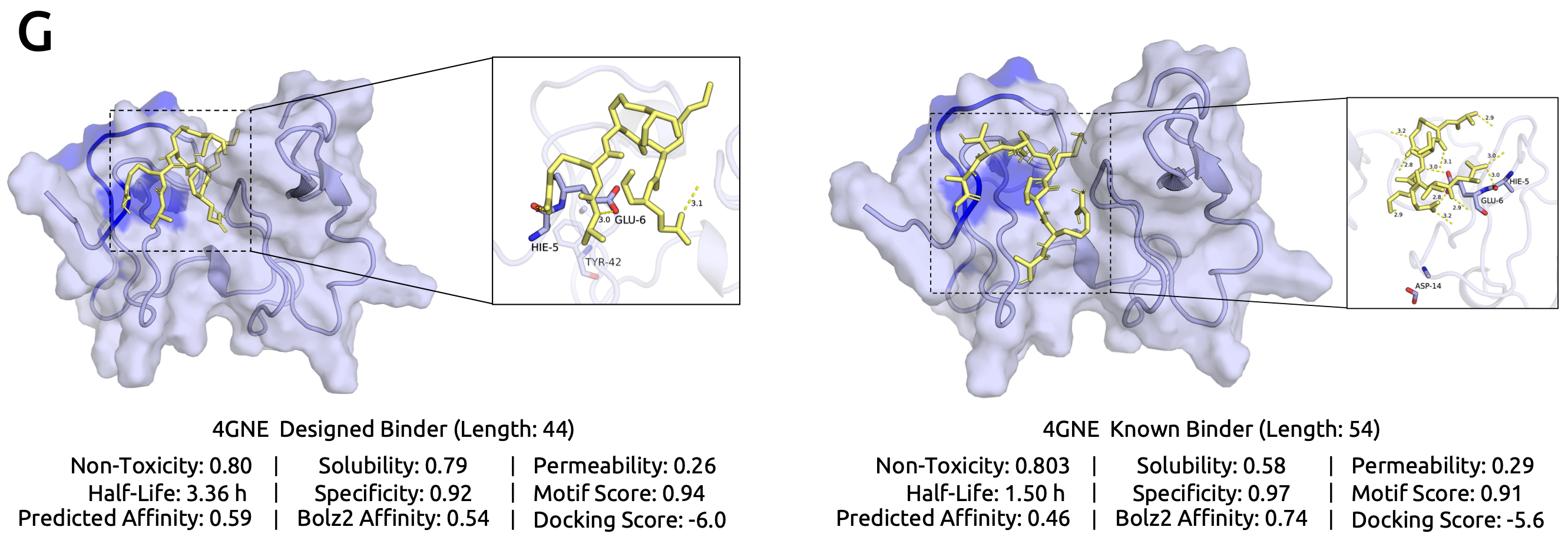}
      \vspace{2em}\\
      \includegraphics[width=\textwidth]{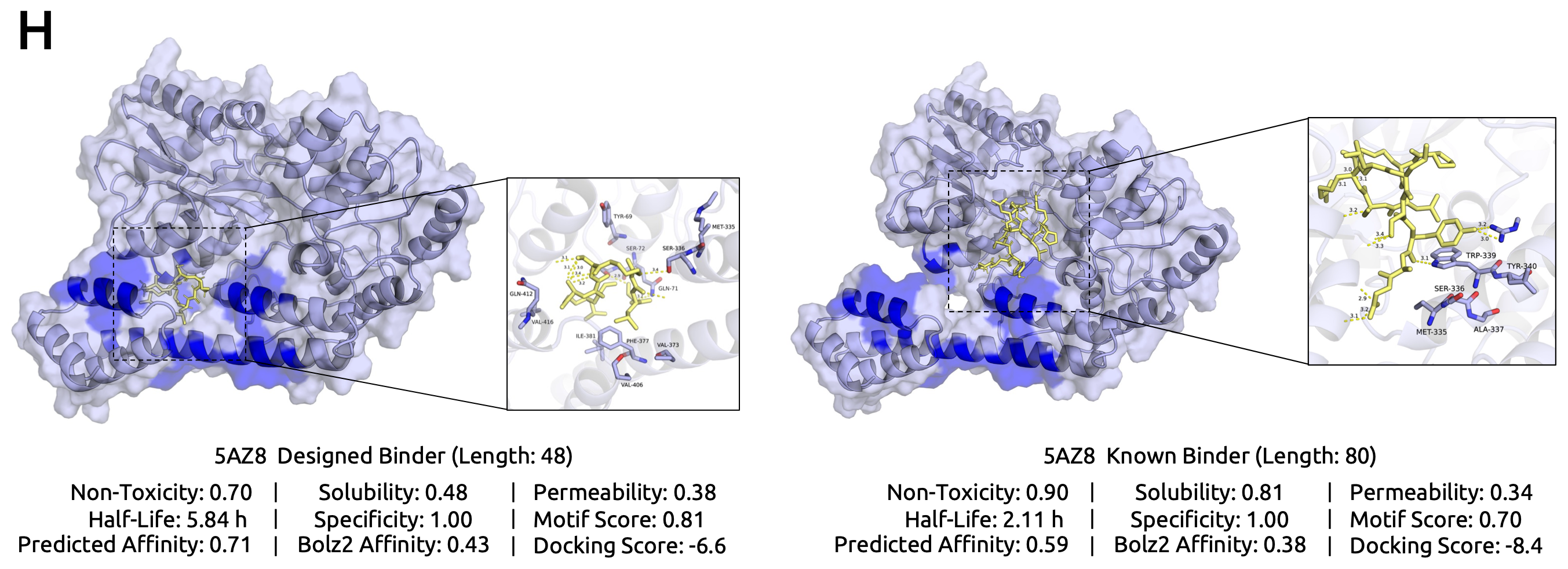}
      \vspace{2em} \\
      \includegraphics[width=\textwidth]{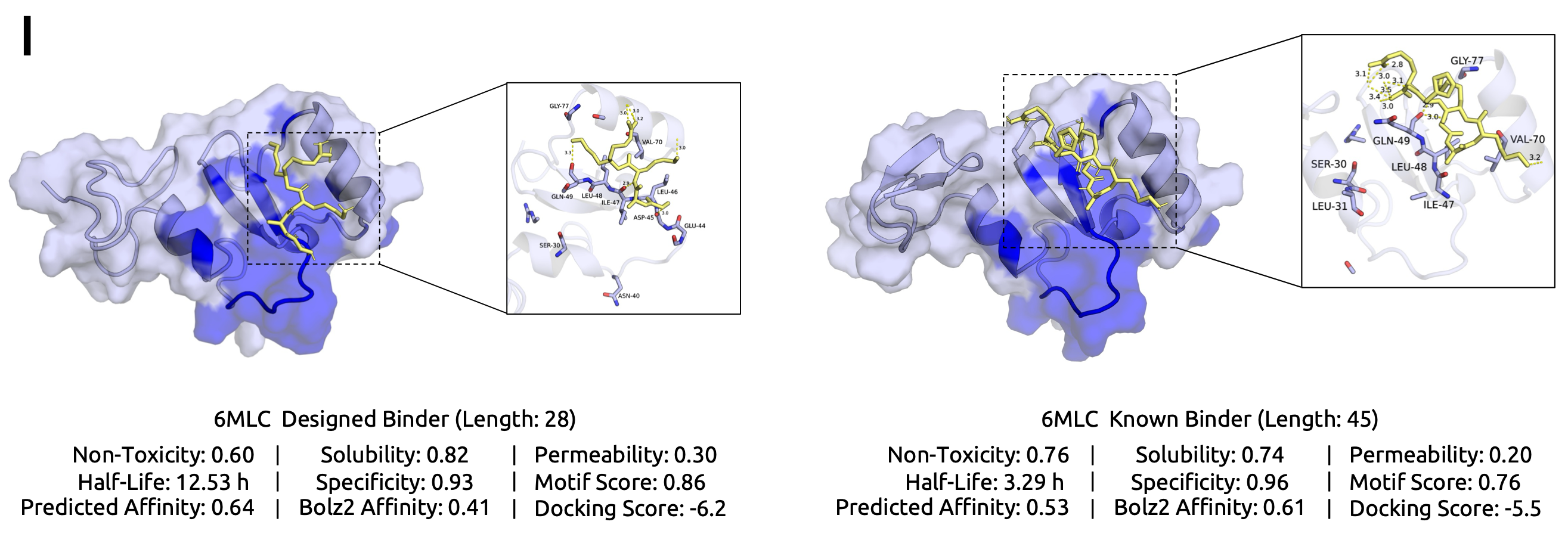}
    \caption{\textbf{Complex structures of known peptide binders and pCoMole-designed peptidomimetic binders with three target PDB proteins (G) 4GNE (H) 5AZ8 (I) 6MLC.} Binders are shown in yellow, target proteins in light blue, and the dark-blue surface highlights the target motifs used by pCoMole to enforce motif-specific binding during design. Insets zoom into the binding interface to illustrate key contacts. Predicted property scores for each binder are reported.}
    \label{fig:dock3}
    \vskip -0.1in
\end{figure}

\begin{figure}[t]
    \centering
      \includegraphics[width=\textwidth]{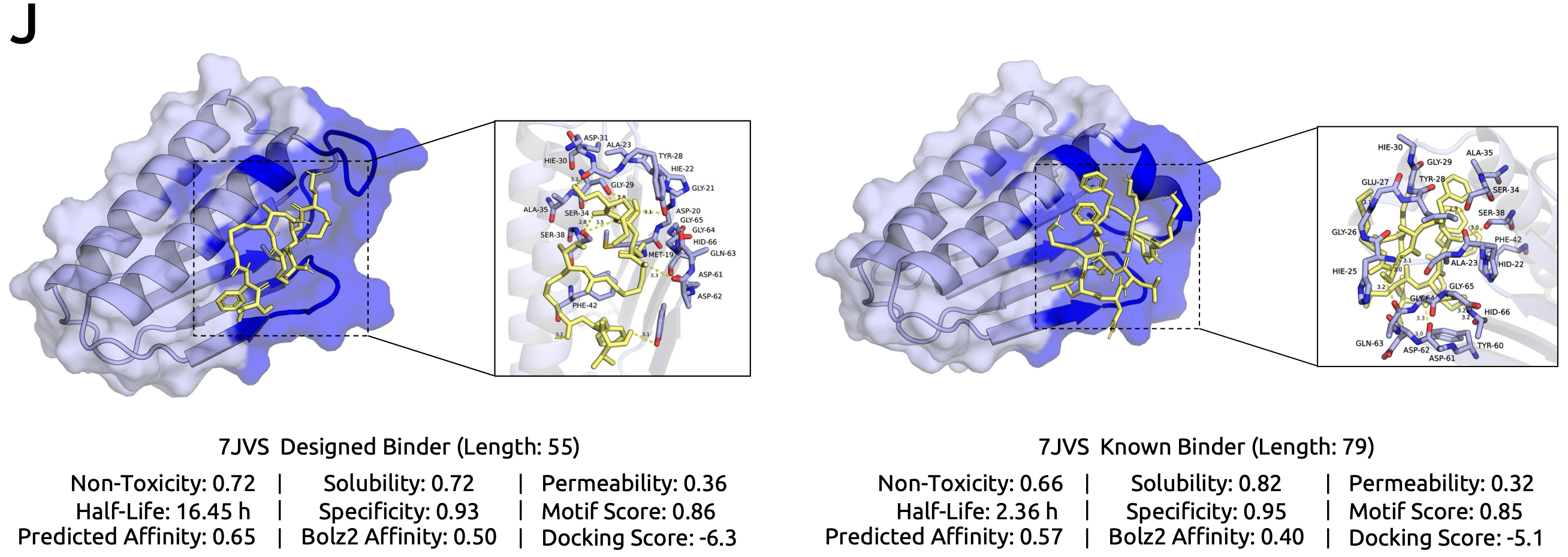}
      \vspace{2em}\\
      \includegraphics[width=\textwidth]{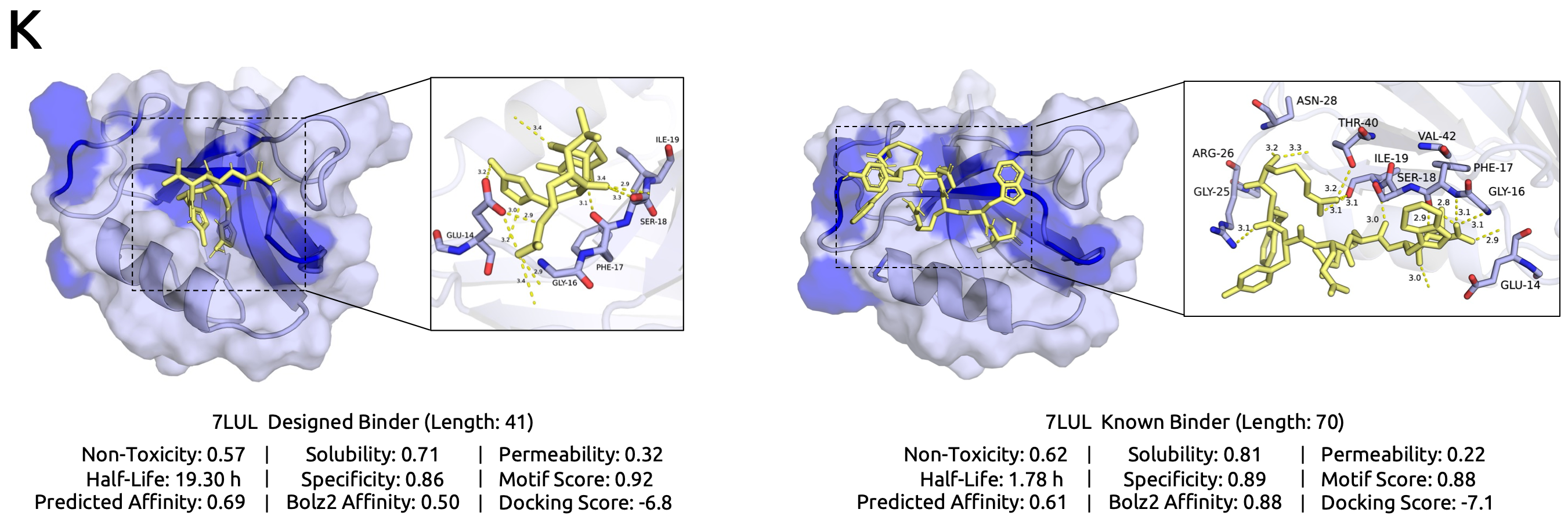}
       \vspace{2em}\\
      \includegraphics[width=\textwidth]{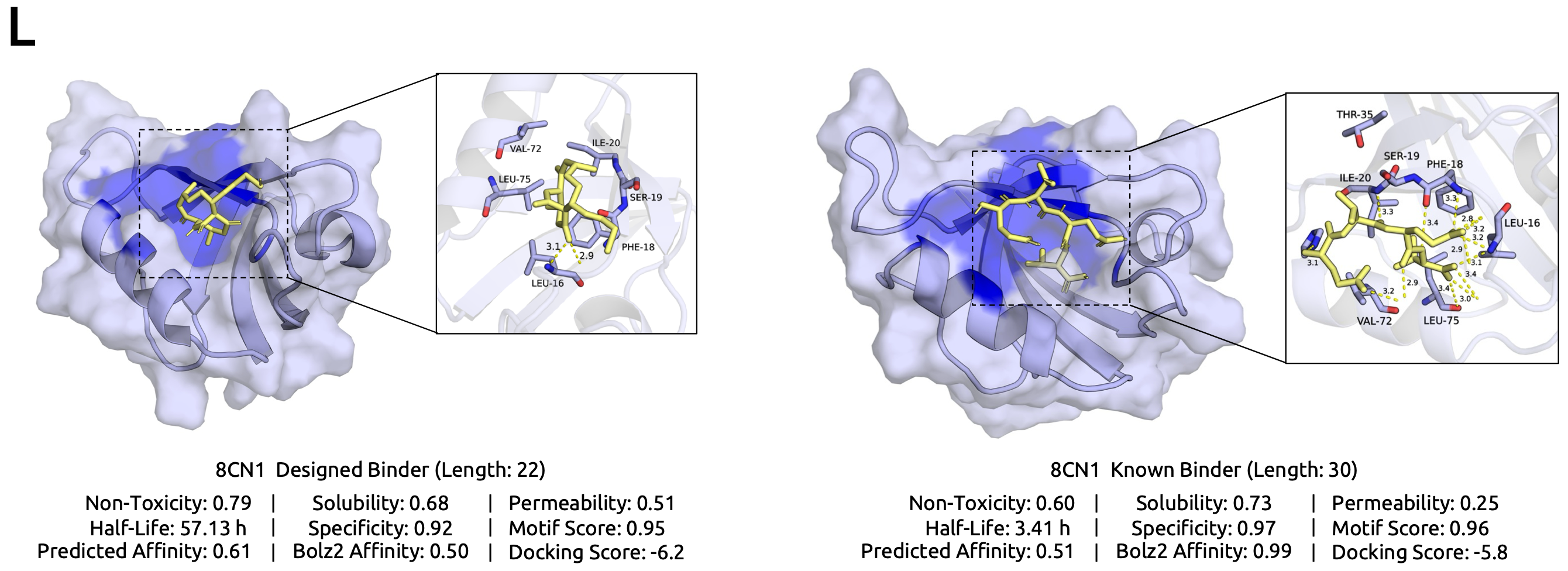}
    \caption{\textbf{Complex structures of known peptide binders and pCoMole-designed peptidomimetic binders with three target PDB proteins (J) 7JVS (K) 7LUL (L) 8CN1.} Binders are shown in yellow, target proteins in light blue, and the dark-blue surface highlights the target motifs used by pCoMole to enforce motif-specific binding during design. Insets zoom into the binding interface to illustrate key contacts. Predicted property scores for each binder are reported.}
    \label{fig:dock4}
    \vskip -0.1in
\end{figure}

\begin{figure}[t]
  \vskip 0.2in
  \begin{center}
    \centerline{\includegraphics[width=\columnwidth]{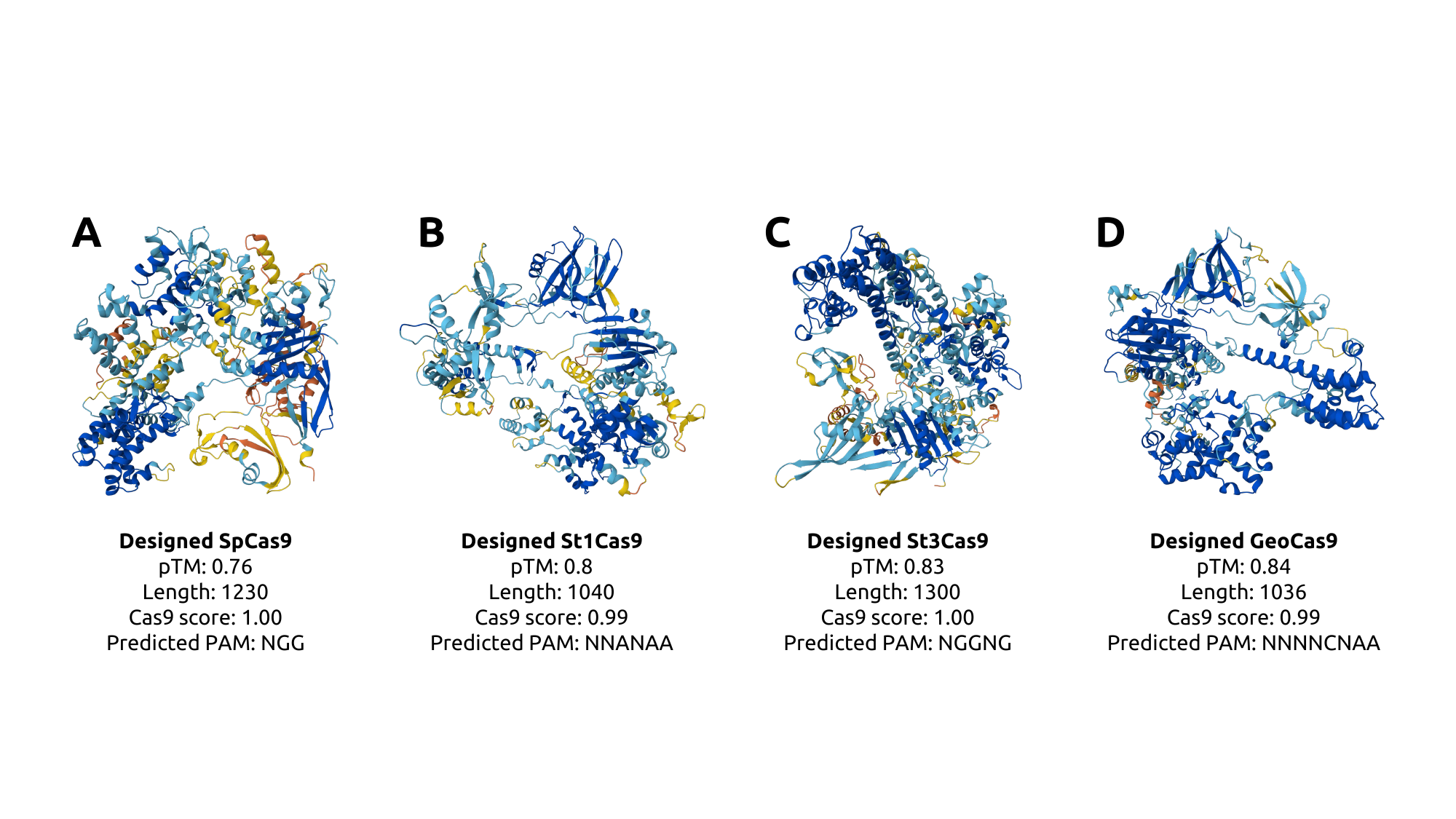}}
    \caption{\textbf{pCoMole effectively shrinks a diverse set of Cas9 orthologs.} Representative pCoMole-designed Cas9 sequences generated by shrinking four well-characterized Cas9 orthologs: \textbf{(A)} SpCas9, \textbf{(B)} St1Cas9 \textbf{(C)} St3Cas9 \textbf{(D)} GeoCas9. We present their secondary structures together with AlphaFold3-predicted pTM scores, sequence length, Cas9-likelihood score, and predicted PAM.}
    \label{fig:designed_cas9}
  \end{center}
\end{figure}

\clearpage
\newpage
\section*{NeurIPS Paper Checklist}

\begin{enumerate}

\item {\bf Claims}
    \item[] Question: Do the main claims made in the abstract and introduction accurately reflect the paper's contributions and scope?
    \item[] Answer: \answerYes{} % Replace by \answerYes{}, \answerNo{}, or \answerNA{}.
    \item[] Justification: The abstract and the contribution list in Sec~\ref{sec:introduction} (Introduction) clearly state the four contributions: introducing pCoMole as an offline method for Pareto-aligned biomolecular sequence editing under hard terminal constraints, deriving a Doob-$h$ guided editing process with an augmented Tchebycheff utility (Sec~\ref{sec:method}, Sec~\ref{sec:doob_h_guidance}), making it practical via short Monte Carlo rollouts with theoretical guarantees (Sec~\ref{sec:short_rollouts}), and demonstrating improved feasibility and trade-off control across GFP shrinkage (Sec~\ref{sec:gfp}), Cas9 shrinkage (Sec~\ref{sec:cas9}), and peptidomimetic binder design.

\item {\bf Limitations}
    \item[] Question: Does the paper discuss the limitations of the work performed by the authors?
    \item[] Answer: \answerYes{} % Replace by \answerYes{}, \answerNo{}, or \answerNA{}.
    \item[] Justification: Limitations are discussed in Sec~\ref{sec:ablation} and the Discussion section. Specifically, pCoMole incurs higher computational cost than shrinkage-only baselines because each edit step scores multiple candidate edits and estimates their continuation values via Monte Carlo rollouts, with runtime scaling in the number of steps, candidates per step, and rollouts per candidate (Table~\ref{tab:runtime_tiers}). The hyperparameter ablations (Tables~\ref{tab:step_ablation},~\ref{tab:rollout_ablation}, and~\ref{tab:candidate_ablation}) further document diminishing returns at larger budgets. The Discussion notes that pCoMole results are validated \textit{in silico} and that wet-lab validation is an important next step.

\item {\bf Theory assumptions and proofs}
    \item[] Question: For each theoretical result, does the paper provide the full set of assumptions and a complete (and correct) proof?
    \item[] Answer: \answerYes{} % Replace by \answerYes{}, \answerNo{}, or \answerNA{}.
    \item[] Justification: All theoretical results are stated formally and accompanied by complete proofs in App~\ref{sec:proofs}. This includes the Doob-$h$ transform construction (Proposition~\ref{prop:doob_transform}, proved as Proposition~\ref{prop:doob_transform_proof}), the augmented Tchebycheff Pareto coverage result (Proposition~\ref{prop:atc_pareto}), the rollout kernel approximation bound (Proposition~\ref{prop:approx_kernel}), the additive-to-multiplicative error conversion (Proposition~\ref{prop:add_to_mult}), the multistep stability result (Proposition~\ref{prop:multistep_stability}), and the incumbent feasibility guarantee (Proposition~\ref{prop:incumbent_feasible}). Log-domain rollout estimation details supporting these results are provided in App~\ref{app:log_domain}.

    \item {\bf Experimental result reproducibility}
    \item[] Question: Does the paper fully disclose all the information needed to reproduce the main experimental results of the paper to the extent that it affects the main claims and/or conclusions of the paper (regardless of whether the code and data are provided or not)?
    \item[] Answer: \answerYes{} % Replace by \answerYes{}, \answerNo{}, or \answerNA{}.
    \item[] Justification: The paper provides full details required for reproduction. Edit Flow training procedures, datasets, model architectures, and optimizer settings are described in App~\ref{sec:diversity} (Edit Flow details) and Score model details follow. Sampling details for all three editing tasks (GFP, Cas9, peptidomimetics) are documented in App~\ref{sec:sampling}, including all hyperparameters such as \texttt{num\_steps}, \texttt{num\_candidates}, \texttt{num\_rollouts}, and objective weights. Code is also released anonymously at \url{https://anonymous.4open.science/r/pCoMole-18EC/} as noted in the Discussion.

\item {\bf Open access to data and code}
    \item[] Question: Does the paper provide open access to the data and code, with sufficient instructions to faithfully reproduce the main experimental results, as described in supplemental material?
    \item[] Answer: \answerYes{} % Replace by \answerYes{}, \answerNo{}, or \answerNA{}.
    \item[] Justification: Anonymous code is provided at \url{https://anonymous.4open.science/r/pCoMole-18EC/} as stated in the Reproducibility paragraph of the Discussion. The experiments use publicly available datasets and resources, including UniRef sequences, FPBase fluorescent proteins, the CRISPR-Cas Atlas, ChemDiv peptidomimetics, and PDB targets, all cited in App~\ref{sec:diversity}. Pre-trained baselines (RayGun, SCISOR) and predictors (FPredX, Protein2PAM, Boltz-2, AutoDock VINA, AlphaFold3) are likewise publicly available and cited.

\item {\bf Experimental setting/details}
    \item[] Question: Does the paper specify all the training and test details (e.g., data splits, hyperparameters, how they were chosen, type of optimizer) necessary to understand the results?
    \item[] Answer: \answerYes{} % Replace by \answerYes{}, \answerNo{}, or \answerNA{}.
    \item[] Justification: Full training and test details are reported across the appendices on Edit Flow training, score model details, and pCoMole sampling. This includes optimizer choice (AdamW with weight decay 0.03), learning rates ($1\times 10^{-5}$ for UniRef and Cas9 Edit Flows, $1\times 10^{-6}$ for the GFP fine-tune), warmup and cosine decay schedules, model architecture (768-dimensional, 8 Transformer layers, 12 attention heads), 80-10-10 train/test/validation splits, and complete sampling hyperparameters per task in App~\ref{sec:sampling}. Train/test losses are reported in Table~\ref{tab:editflow}.

\item {\bf Experiment statistical significance}
    \item[] Question: Does the paper report error bars suitably and correctly defined or other appropriate information about the statistical significance of the experiments?
    \item[] Answer: \answerYes{} % Replace by \answerYes{}, \answerNo{}, or \answerNA{}.
    \item[] Justification: Quantitative results in Tables~\ref{tab:gfp_raygun},~\ref{tab:runtime_tiers},~\ref{tab:gfp_ablation},~\ref{tab:cas9_main_table},~\ref{tab:mimetics_drug},~\ref{tab:cas9_benchmark},~\ref{tab:cas9_ablation},~\ref{tab:mimetics_pdb},~\ref{tab:constraint_ablation},~\ref{tab:step_ablation},~\ref{tab:rollout_ablation},~\ref{tab:candidate_ablation}, and~\ref{tab:coverage} report aggregate metrics over multiple generated sequences per condition (e.g., 50 shrunk variants per Cas9 input in Sec~\ref{sec:cas9}, 100 sequences per method in the RayGun/SCISOR Cas9 benchmark, 100 peptidomimetic candidates per PDB target). Distributions of mean pLDDT scores over 1{,}000 sequences are shown in Figure~\ref{fig:plddts} for the UniRef$_S$ Edit Flow.

\item {\bf Experiments compute resources}
    \item[] Question: For each experiment, does the paper provide sufficient information on the computer resources (type of compute workers, memory, time of execution) needed to reproduce the experiments?
    \item[] Answer: \answerYes{} % Replace by \answerYes{}, \answerNo{}, or \answerNA{}.
    \item[] Justification: Compute resources are reported in App~\ref{sec:diversity} (Edit Flow details) and the Benchmark Models appendix. Edit Flow training uses one NVIDIA RTX A6000 GPU with 48~GB of VRAM with epoch counts specified per model (around 30 for UniRef$_S$, 10 for UniRef$_L$, and 100 for Cas9). Baseline RayGun runs use an NVIDIA H100 NVL GPU with 94~GB of VRAM, and SCISOR runs use a single NVIDIA RTX A6000 with 48~GB of VRAM. pCoMole runtime, including comparisons to baselines, is reported in Tables~\ref{tab:runtime_tiers}.
    
\item {\bf Code of ethics}
    \item[] Question: Does the research conducted in the paper conform, in every respect, with the NeurIPS Code of Ethics \url{https://neurips.cc/public/EthicsGuidelines}?
    \item[] Answer: \answerYes{} % Replace by \answerYes{}, \answerNo{}, or \answerNA{}.
    \item[] Justification: The research conforms with the NeurIPS Code of Ethics. The work does not involve human subjects or sensitive personal data, all datasets and pre-trained models used are publicly available and properly cited, and the submission preserves anonymity through the anonymous code repository.

\item {\bf Broader impacts}
    \item[] Question: Does the paper discuss both potential positive societal impacts and negative societal impacts of the work performed?
    \item[] Answer: \answerYes{} % Replace by \answerYes{}, \answerNo{}, or \answerNA{}.
    \item[] Justification: The paper includes a dedicated Impact Statement that discusses both the positive societal impacts of pCoMole, including its potential to expand the space of deliverable and clinically viable genome-editing tools by reducing the size of systems such as CRISPR-Cas9 without altering target specificity, and the safety considerations that this raises, including off-target activity, immunogenicity, and context-dependent behavior \emph{in vivo}. The Impact Statement explicitly notes that pCoMole must be coupled with careful experimental validation, safety screening, and responsible deployment for genome-editing and therapeutic applications.
    
\item {\bf Safeguards}
    \item[] Question: Does the paper describe safeguards that have been put in place for responsible release of data or models that have a high risk for misuse (e.g., pre-trained language models, image generators, or scraped datasets)?
    \item[] Answer: \answerNA{} % Replace by \answerYes{}, \answerNo{}, or \answerNA{}.
    \item[] Justification: The paper does not release new pre-trained foundation models, datasets, or generators that pose a high risk for misuse. Released artifacts consist of editing methodology and Edit Flow models trained on publicly available biological sequence data, and the Impact Statement notes that downstream therapeutic and genome-editing use requires careful experimental validation and safety screening before deployment.

\item {\bf Licenses for existing assets}
    \item[] Question: Are the creators or original owners of assets (e.g., code, data, models), used in the paper, properly credited and are the license and terms of use explicitly mentioned and properly respected?
    \item[] Answer: \answerYes{} % Replace by \answerYes{}, \answerNo{}, or \answerNA{}.
    \item[] Justification: All existing assets used in the paper are properly credited via citations to the original works, including the UniRef and UniRef50 sequence databases, FPBase, the CRISPR-Cas Atlas~\citep{ruffolo2024design}, ChemDiv peptidomimetics, ESM-2-650M~\citep{lin2023evolutionary}, FPredX~\citep{tam2022fpredx}, Protein2PAM~\citep{nayfach2025engineering}, Boltz-2~\citep{passaro2025boltz}, AutoDock VINA~\citep{trott2010autodock}, AlphaFold3~\citep{abramson2024accurate}, and the RayGun~\citep{devkota2024miniaturizing} and SCISOR~\citep{baron2025diffusion} baselines, all of which are cited in the Edit Flow details, score model, and benchmark model appendices.

\item {\bf New assets}
    \item[] Question: Are new assets introduced in the paper well documented and is the documentation provided alongside the assets?
    \item[] Answer: \answerYes{} % Replace by \answerYes{}, \answerNo{}, or \answerNA{}.
    \item[] Justification: The new assets introduced are the pCoMole framework and the trained Edit Flow models. These are documented through the methodology in Sec~\ref{sec:method}, the architectural and training details in App~\ref{sec:diversity}, the score model details in the corresponding appendix, the sampling configuration in App~\ref{sec:sampling}, and the anonymous code repository at \url{https://anonymous.4open.science/r/pCoMole-18EC/}.

\item {\bf Crowdsourcing and research with human subjects}
    \item[] Question: For crowdsourcing experiments and research with human subjects, does the paper include the full text of instructions given to participants and screenshots, if applicable, as well as details about compensation (if any)? 
    \item[] Answer: \answerNA{} % Replace by \answerYes{}, \answerNo{}, or \answerNA{}.
    \item[] Justification: The paper does not involve crowdsourcing or research with human subjects. All experiments use publicly available datasets and \textit{in silico} evaluations.

\item {\bf Institutional review board (IRB) approvals or equivalent for research with human subjects}
    \item[] Question: Does the paper describe potential risks incurred by study participants, whether such risks were disclosed to the subjects, and whether Institutional Review Board (IRB) approvals (or an equivalent approval/review based on the requirements of your country or institution) were obtained?
    \item[] Answer: \answerNA{} % Replace by \answerYes{}, \answerNo{}, or \answerNA{}.
    \item[] Justification: The paper does not involve human subjects research and therefore no IRB approval is required.

\item {\bf Declaration of LLM usage}
    \item[] Question: Does the paper describe the usage of LLMs if it is an important, original, or non-standard component of the core methods in this research? Note that if the LLM is used only for writing, editing, or formatting purposes and does \emph{not} impact the core methodology, scientific rigor, or originality of the research, declaration is not required.
    %this research? 
    \item[] Answer: \answerNA{} % Replace by \answerYes{}, \answerNo{}, or \answerNA{}.
    \item[] Justification: LLMs are not used as an important, original, or non-standard component of the core methodology. The framework concerns Pareto-constrained editing of biomolecular sequences via Doob-$h$ guided discrete flows with a Transformer-based Edit Flow backbone, and does not employ language models in the methodology beyond the pre-trained ESM-2-650M sequence encoder used for protein sequence embeddings as cited in App~\ref{sec:diversity}.

\end{enumerate}

\end{document}